\documentclass[twoside]{article}

\usepackage[accepted]{aistats2027}
\usepackage[numbers]{natbib}

\usepackage{caption, amssymb, mathtools, extpfeil, float, yfonts, csquotes, lipsum, fancyhdr, algorithm, algorithmic} 

\usepackage{xcolor}

\usepackage{booktabs}
\usepackage{rotating}

\usepackage{multirow}

\usepackage{dsfont}
\usepackage{subcaption}
\usepackage{hyperref}       

\usepackage{amsmath, mathtools, extpfeil, amsthm}

\theoremstyle{plain}
\newtheorem{theorem}{Theorem}
\newtheorem{proposition}{Proposition}
\newtheorem{lemma}{Lemma}
\newtheorem{corollary}{Corollary}
\theoremstyle{definition}
\newtheorem{definition}{Definition}
\theoremstyle{remark}

\newcommand{\themethod}{\texttt{OCPQ}}

\newcommand{\np}[1]{#1}

\begin{document}

%

%
\runningauthor{Skalse, Pona, Simeone, Paoletti}

\twocolumn[

\aistatstitle{Online Conformal Prediction Beyond Feedback}

\aistatsauthor{ Joar Skalse\textsuperscript{\rm 1} \And Edoardo Pona\textsuperscript{\rm 1} \And  Osvaldo Simeone\textsuperscript{\rm 2} \And Nicola Paoletti\textsuperscript{\rm 1} }

\aistatsaddress{ \textsuperscript{\rm 1}King's College London \\\textsuperscript{\rm 2}Northeastern University London } ]

\begin{abstract}
Online conformal prediction (OCP) provides theoretically principled uncertainty quantification for arbitrary black-box classifiers and non-i.i.d.\ data streams by constructing prediction sets that are guaranteed to contain the true label at a user-specified frequency. 
OCP usually updates prediction sets using feedback from previously deployed predictions. We instead study an OCP setting \emph{beyond feedback}: on each round, the learner can \emph{either} output a prediction set \emph{or} query the correct label, but not both. Thus, no deployed prediction is ever evaluated directly. 
We reduce this problem to a partial monitoring game in which prediction actions return no observation and a separate query action reveals the label.
To solve this game, we develop \emph{OCP with queries} (\texttt{OCPQ}), a label-efficient forecasting algorithm designed for this beyond-feedback setting.    
Using an adequate reward construction, for any black-box classifier and any (non-i.i.d.) oblivious data stream of length $T$, \texttt{OCPQ} has $O(T^{2/3})$ expected regret and expected coverage of at least $\beta-O(T^{-1/3})$ for a user-defined $\beta$, while querying only an expected $T^{-1/3}$ fraction of rounds. This provides coverage comparable to bandit-based OCP methods while requiring no feedback from deployed predictions. Experiments on real-world datasets demonstrate the effectiveness of our approach.
\end{abstract}

\section{Introduction}
Uncertainty quantification is important for decision-making, especially in safety-critical domains. \emph{Conformal prediction} is a theoretically principled framework for uncertainty quantification of black-box models \cite{angelopoulos2022gentleintroductionconformalprediction,vovk2022}. A conformal prediction method complements the outputs of an arbitrary pre-trained predictor with prediction sets that achieve a user-specified coverage level, while ideally remaining as small as possible. This makes it possible to obtain formal uncertainty guarantees with minimal assumptions: in the classical setting, we normally only require exchangeability of calibration and test data. 

\emph{Online conformal prediction} (OCP) extends this idea to sequential, potentially non-exchangeable data, typically adopting worst-case, per-sequence guarantees that remain valid under distribution shift and adversarial sequences. OCP has consequently been studied in settings ranging from language models \citep{lee2026onlineconformalabstentionfactuality,rubashevskii2026adaptiveconformalpredictionimproving} and robotics \citep{lindemann2023safeplanningdynamicenvironments,zhao2025conformalizedinteractiveimitationlearning} to telecommunications networks \citep{10262367} and distributed systems \citep{zhu2025conformaldistributedremoteinference}.

Existing OCP methods typically assume that the correct labels can be obtained after a set prediction is issued  --- in other words, they assume that the algorithm gets \emph{feedback}. Feedback may potentially be partial or delayed \citep{wang2025mirroronlineconformalprediction, wang2026onlineconformalpredictioncorrupted}, but it always refers to decisions made by the algorithm. This feedback is leveraged to compensate for past mistakes, for instance, by enlarging future prediction sets after miscoverage events. 

We study a different regime in which prediction and supervision are mutually exclusive. On each round, the learner may either deploy a prediction set or query the correct label, but not both. 
We refer to this setting as \emph{online conformal prediction beyond feedback}. 

This protocol is naturally motivated by rejection learning  \cite{cortes2016learning}, where a predictive model can either make an autonomous decision or defer the input to a more capable decision-maker, such as a human expert or a stronger model. Labels may become available for deferred inputs, while autonomous decisions may never receive ground-truth feedback. Thus, the learner cannot observe the outcomes of its own deployed predictions, only those of query points selected for labelling. A concrete example arises in safety monitoring for large language models, where inputs are either processed by lightweight monitors such as latent probes or escalated to a human or stronger model \citep{mckenzie2026detectinghighstakesinteractionsactivation,pona2026calibratethendelegatesafetymonitoringrisk}. 

In this paper, we introduce \emph{Online Conformal Prediction with Queries} (\themethod), a new OCP method for this setting. We formulate the problem as a finite partial-monitoring game whose prediction actions reveal no information and whose query action reveals the outcome. 
We then solve the resulting game using an adaptation of the label-efficient exponentially weighted forecaster algorithm \cite{label_efficient_prediction}. 
In our setting, experts correspond to conformal prediction thresholds, and a query replaces, rather than follows, a prediction. {\themethod} leverages the fact that random queries provide unbiased estimates of the cumulative rewards of all thresholds, including those that were not deployed.

For any black-box classifier, any data stream of length $T$, and any user-defined $\beta \in [0,1]$, \themethod{} guarantees that the expected coverage rate $p_{\mathrm{cover}}$ satisfies
\[
\mathbb{E}[p_{\mathrm{cover}}] \geq \beta-O(T^{-1/3}),
\]
while querying an expected $T^{-1/3}$ fraction of the rounds. 
\themethod{} thus provides finite-horizon coverage guarantees 
without ever observing whether any deployed prediction set covered its label.

\subsection{Problem Formulation}

We focus on an online prediction setting operating along a discrete time index $t=1,2,\dots,T$, in which the input $x_t$ belongs to an input space $\mathcal{X}$ and the correct label $y_t$ belongs to an output space $\mathcal{Y}$. We assume that $\mathcal{Y}$ is finite, but $\mathcal{X}$ need not be finite.\footnote{The assumption that $\mathcal{Y}$ is finite is mostly made out of convenience, and can be lifted; see Appendix~\ref{app:generalisations}.}
We adopt an adversarial formulation in which the sequence 
\begin{equation}
z^T= \langle(x_t, y_t)\rangle_{t=1}^T \in (\mathcal X \times \mathcal Y)^T,
\end{equation} 
with $z_t=(x_t,y_t)$ for $t=1, \dots, T$, is selected in advance, possibly in an adversarial way, i.e., we have an \textit{oblivious adversary}. 


On each round $t$, the learner observes the input $x_t$ and returns an action $a_t \in \mathcal A = \mathcal{P}(\mathcal{Y)} \cup \{\texttt{query}\}$ as
\begin{equation}\label{eq:action}
a_{t}=\begin{cases}
S_{t}\subseteq\mathcal{Y} & \text{(produce a prediction set)}\\
\texttt{query} & \text{(request to observe \ensuremath{y_{t})}}.
\end{cases}
\end{equation}
After selecting an action, the learner receives an observation $\omega_t \in \Omega = \mathcal{Y} \cup \{\bot\}$, where $\omega_t = y_t$ if $a_t = \texttt{query}$, and otherwise $\omega_t$ is a null observation $\bot$.
Note that this means that the learner never obtains feedback about the true label $y_t$, or even about the coverage event $\mathds{1}\{y_t \in \mathcal S_t\}$, on rounds in which it produces a prediction set $\mathcal S_t$. In this way, the output $y_t$ can only be observed on time steps at which the learner declined to make a prediction. 


The learner selects action $a_t$ based on a (stochastic)  \emph{policy} $\pi : \mathcal{H} \times \mathcal X \to \Delta(\mathcal{A})$, where $\Delta(\mathcal{A})$ is the simplex over the action space $\mathcal{A}$ and $\mathcal{H} \subset (\mathcal{X} \times \mathcal{A} \times \Omega)^\star$ is the set of all histories 
$
h_n = \langle(x_t,a_t,\omega_t)\rangle_{t=1}^n,
$
where $\omega_t$ is the observation at time $t$, \textcolor{black}{and $h_0 = \langle\rangle$}.

The learner has access to a classifier $C : \mathcal{X} \to \Delta(\mathcal{Y})$, which it can use to construct the prediction sets $S_t$.
We use $p_C(y\mid x)$ to denote the probability that $C$ assigns to label $y$ given input $x$. Following existing OCP methods, we assume that the learner selects a prediction threshold $m_t \in [0,1]$, and then let the prediction set $S_t$ for input $x_t$ be obtained as the superlevel set $\Gamma_C(x_t,m_t)$ defined by
\begin{multline}\label{equation:prediction_set}
    \Gamma_C(x_t,m_t) = \\
    \{y \in \mathcal{Y} \mid p_C(y\mid x_t) \geq \max_{y^\star \in \mathcal{Y}} p_C(y^\star\mid x_t) - m_t  \},
\end{multline}
where the choice of $m_t$ determines how conservative the learner is:
if $m_t=1$ then $\Gamma_C(x_t,m_t)=\mathcal{Y}$ (i.e., it contains all labels), and if $m_t=0$ then $\Gamma_C(x_t,m_t)$ includes only the most likely label (as judged by $C$).\footnote{In (O)CP literature, the inclusion condition is often expressed in terms of a score function: $s(x_t,y)\leq m_t$. We can retrieve this form by setting $s(x_t,y)=\max_{y^\star \in \mathcal{Y}} p_C(y^\star\mid x_t) -  p_C(y\mid x_t)$.}

Given a data stream $z^T$ and an action sequence $a^T=\langle a_1 \dots a_T\rangle$, 
we use $Q$ to denote the total number of queries, i.e., $Q = \sum_{t=1}^T \mathds{1}\{a_t = \texttt{query}\}$.
%
%

We are interested in 
controlling the \emph{coverage rate}
\begin{equation}
p_\text{cover} =\frac{1}{T} \cdot \sum_{t=1}^T \mathds{1}\{a_t \neq \texttt{query}\} \cdot \mathds{1}\{y_t \in S_t\},\end{equation}
i.e., the proportion of the rounds on which a prediction set $S_t$ is produced and the true label $y_t$ is contained in $S_t$. 
Note that a query counts as a miscoverage event, so the learner is incentivised to keep the number of queries low in order to control $p_\text{cover}$.
\np{Also note that $p_\text{cover}$ under-approximates both the usual \emph{marginal coverage} where all rounds are considered ($\frac{1}{T} \cdot \sum_{t=1}^T \mathds{1}\{y_t \in S_t\}$) and \emph{selective coverage} where only deployed predictions are ($\frac{1}{T-Q} \cdot \sum_{t=1}^T \mathds{1}\{a_t \neq \texttt{query}\} \cdot \mathds{1}\{y_t \in S_t\}$). Thus, controlling $p_\text{cover}$ automatically controls both notions of coverage. } 


While coverage can be trivially controlled by choosing the full set $\mathcal S_t =\mathcal Y$ on every round, prediction sets need to be informative, i.e., as small as possible. Thus, we also consider the 
average size of $S_t$, also known as the \emph{inefficiency} $I$:
\begin{equation}
I = \frac{\sum_{t=1}^T \mathds{1}\{a_t \neq \texttt{query}\} \cdot |S_t|}{T-Q}.
\end{equation} 
Note that if the actions are sampled from a stochastic policy, then $Q$, $p_\text{cover}$, and $I$ are all random variables. 

In summary, we wish to simultaneously control the expected coverage rate and the expected query rate, regardless of which data stream the adversary picks, while keeping prediction sets as small as possible.

\subsection{Related Work}
\np{Table~\ref{tab:feedback_comparison} in the Appendix compares our method to some of the most relevant OCP methods in terms of feedback models and guarantees. Below, we survey a wider selection of related works. } 
\paragraph{OCP With Feedback}
Most OCP methods observe the label (as in \citep{gibbs2021adaptiveconformalinferencedistribution,bhatnagar2023improvedonlineconformalprediction,pmlr-v235-angelopoulos24a,angelopoulos2024onlineconformalpredictiondecaying}) or a coverage signal \citep{wang2024efficientonlinesetvaluedclassification, gollapudi2026efficientonlineconformalselection} after issuing a prediction set. 
Some OCP methods take \emph{semi-bandit feedback}, meaning that they observe the label only if the prediction set has coverage \citep{pmlr-v267-ge25a, yang2026onlineconformalpredictionadversarial, lee2026onlineconformalabstentionfactuality}.
\citet{wang2025mirroronlineconformalprediction} allows feedback to arrive intermittently with exogenous probabilities $p_t$, while \citet{wang2026onlineconformalpredictioncorrupted} study binary coverage feedback that may be adversarially corrupted. These methods reduce the amount or reliability of feedback, but still update from observations associated with deployed predictions.

\paragraph{Bandit Approaches to OCP}
Several works have shown that OCP can be reduced to a multi-armed bandit problem~\cite{wang2024efficientonlinesetvaluedclassification,pmlr-v267-ge25a,yang2026onlineconformalpredictionadversarial}, leading to both coverage and regret  guarantees~\cite{ramalingam2025relationshipnoregretlearningonline}. In these works, however, the deployed action still produces an observation, either a correctness signal or a label on covered rounds. In contrast, \themethod{} receives no information from deployed predictions. It observes labels only through separate query actions, used on an expected $T^{-1/3}$ fraction of rounds, while obtaining an $O(T^{-1/3})$ coverage deficit.

\paragraph{Label-Efficient Adversarial Learning}
\citet{zhao2025conformalizedinteractiveimitationlearning} combine intermittent observations with active queries in OCP for interactive imitation learning. 
Our setting, on the other hand, 
does not assume the existence of any intermittent observations.
The learning rule underlying \themethod{} is an adaptation of the label-efficient forecaster of \citet{label_efficient_prediction}, which queries outcomes at random and updates all expert weights using inverse probability estimates. 
Our contribution is thus not a new generic adversarial learning rule, but rather the use of label-efficient learning to obtain an OCP method beyond feedback, together with the threshold reward, the query-only estimator, the regret-to-coverage result, \textcolor{black}{and the adaptive version of \texttt{OCPQ}, together with the associated theoretical guarantees}.
General partial monitoring game algorithms, e.g., \citep{RUSTICHINI1999224,lattimore2018cleaningneighborhoodclassificationadversarial,lattimore2019explorationoptimisationpartialmonitoring}, can also be applied in this setting, an option which we discuss in Appendix~\ref{app:bandit_formulation}.

\subsection{Contributions}
\begin{itemize}
    \item We formulate OCP beyond feedback, where the learner must choose between predicting and querying, and therefore never observes the outcome of a deployed prediction.
    \item We reduce OCP beyond feedback to a partial monitoring game, and solve the resulting game using a label-efficient adversarial learning scheme. 
    \item We prove finite horizon expected regret $O(T^{2/3})$ and expected coverage at least $\beta-O(T^{-1/3})$ using an expected query rate of $T^{-1/3}$, \textcolor{black}{where $\beta$ is a hyperparameter we can select freely}. We also give high-probability regret and coverage bounds.
    \item \np{We develop an adaptive variant of \textcolor{black}{our scheme} that \textcolor{black}{dynamically} adjusts the coverage target using query-time coverage estimates, with corresponding regret and coverage guarantees.}
    \item We evaluate our method under distribution shift and on adversarially rewritten safety prompts, showing that the coverage and efficiency can be controlled 
    \textcolor{black}{effectively even under conditions that are highly adversarial to the learner.}
\end{itemize}

\section{Online Conformal Prediction With Queries}

\subsection{Partial Monitoring Game Formulation}

We cast the beyond-feedback OCP problem as a finite \emph{partial monitoring game}, a generalisation of the adversarial multi-armed bandit framework in which the feedback generated by an action need not reveal its reward. The full game formulation is given in Appendix~\ref{app:bandit_formulation}. 
The \enquote{arms} are a finite set of prediction thresholds $M \subset [0,1]$,
where action $m \in M$ \textcolor{black}{together with classifier $C$} returns the prediction set $\Gamma_C(x_t,m)$ (Equation~\ref{equation:prediction_set}), together with the additional action $\texttt{query}$. For each observed input $x_t$ and hidden label $y_t$, these actions induce the reward in Definition~\ref{def:auxiliary_reward}: queries get no reward, and a threshold is rewarded according to whether its prediction set contains $y_t$ and how tight the threshold is. The best fixed prediction action therefore corresponds to the best fixed conformal threshold over the full data stream.

The feedback structure captures the key difficulty of OCP beyond feedback. Prediction actions reveal neither the label nor their coverage or reward. By contrast, the action $\texttt{query}$ returns no prediction set and has zero reward, but reveals $y_t$. This allows the learner to reconstruct the reward $R(m,t)$ for every threshold $m\in M$. As shown in Appendix~\ref{app:bandit_formulation}, this problem belongs to the hard class of partial monitoring games, for which the characteristic regret rate is $O(T^{2/3})$.

\subsection{The OCPQ Algorithm}
Having reduced our OCP settings to the partial monitoring game above, we solve the game using Algorithm~\ref{algorithm:OCP_with_queries}. 
The method takes the following inputs:
a \emph{classifier} $C : \mathcal{X} \to \Delta(\mathcal{Y})$, a finite set of \emph{prediction thresholds} $M \subset [0,1]$, such that $0 \in M$ and $1 \in M$, a \emph{query rate} $\epsilon \in (0,1)$, a \emph{baseline parameter} $\beta \in [0,1]$, a \emph{learning rate} $\eta \in (0,1]$, and a \emph{time horizon} $T \in \mathbb{N}$.
The parameter $\beta$ is used to control the trade-off between coverage and efficiency, and can be used to derive a lower bound on the coverage rate (c.f.\ Corollary~\ref{corollary:coverage_rate_bound}).
Theorem~\ref{thm:simple_learner_regret} in the next section gives guidance on how to set the values of $\epsilon$ and $\eta$ based on $T$.
$C$ can be any black-box classifier, including a neural network. 
\np{Note that we use a fixed-rate query strategy. This can be generalised to adaptive rates; see Section~\ref{sec:further-ext}.} 

\begin{algorithm}[ht]
\caption{\texttt{OCP With Queries}}\label{algorithm:OCP_with_queries}
\begin{algorithmic}[1]
\STATE \textbf{input:} classifier $C$, thresholds $M$, query rate $\epsilon$, baseline $\beta$, learning rate $\eta$, horizon $T$
\STATE \textbf{output:} actions $\langle a_1, \ldots, a_T \rangle$ as per Equation~\ref{eq:action}.

\STATE let $\hat{Z}_{m,1} = 0$ for all $m \in M$
\FOR{$t \in \{1 \dots T\}$}
\STATE observe $x_t$
\STATE let $a_t = \texttt{query}$ with probability $\epsilon$
\IF{$a_t = \texttt{query}$}
    \STATE observe $y_t$
    \STATE let $\hat{Z}_{m, t+1} = \hat{Z}_{m,t} +  1 - \frac{1 - R(m,t)}{\epsilon}$ for all $m \in M$, with $R(m,t)$ as per Definition~\ref{def:auxiliary_reward}
\ELSE
    \STATE let $p_{m,t} = 
    \frac{\exp(\eta \cdot \hat{Z}_{m,t})}{\sum_{i \in M} \exp(\eta \cdot \hat{Z}_{i,t})} 
    $ 
    for $m \in M$
    \STATE let  $m_t = m$ with probability $p_{m,t}$
    \STATE let $a_t = \Gamma_C(x_t, m_t)$ (Equation~\ref{equation:prediction_set})
    \STATE let $\hat{Z}_{m, t+1} = \hat{Z}_{m,t} +  1$ for all $m \in M$
\ENDIF
\ENDFOR
\end{algorithmic}
\end{algorithm}

On round $t$, \texttt{OCPQ} makes a query with probability $\epsilon$ (line 6). A query returns $y_t$ but produces no prediction set. Otherwise, the method samples a threshold $m_t$ from the exponentially weighted distribution (lines 11 and 12) and outputs $S_t=\Gamma_C(x_t,m_t)$ (line 13). The prediction set contains labels whose classifier scores lie within $m_t$ of the largest score. No label or coverage indicator is observed after this prediction. This separation between prediction and observation is the beyond feedback feature of the method.


To find the best prediction threshold, \texttt{OCPQ} implicitly tries to maximise an \emph{auxiliary reward function}:
\begin{definition}\label{def:auxiliary_reward}
    Given a data stream $z^T$, classifier $C$, and set of prediction thresholds $M$, let the \emph{auxiliary reward function} 
    $R : M \cup \{\texttt{query}\} \times \{1 \dots T\} \to [0,1]$
    be the function such that 
    $
    R(\texttt{query}, t) = 0
    $
    for all $t$, and such that for $m \in M$,
    \begin{equation*}
        R(m,t) =\begin{cases}
        1 - m \cdot (1-\beta) \text{ if } y_t \in \Gamma_C(x_t, m),\\
        0 \text{ otherwise},
        \end{cases}
    \end{equation*}
    where $\Gamma_C(x_t, m)$ is defined as in Equation~\ref{equation:prediction_set}.
\end{definition}
In other words, the algorithm never gets any reward when it makes a query or when it outputs a prediction set $S_t$ such that $y_t \not\in S_t$.
Otherwise, if $y_t \in S_t$, the algorithm gets a higher reward for selecting a tighter prediction threshold \np{(i.e., a small $m$)}.   
This will encourage the algorithm to select small prediction thresholds (which in turn leads to small prediction sets), but without selecting thresholds that are \textcolor{black}{too small to achieve coverage}.
\np{When $y_t \in S_t$, the learner receives at least a reward of $\beta$, and exactly $\beta$ if it produces the trivial set (i.e., $m=1$). }
The value of $\beta$ lets us control the trade-off between achieving a higher coverage rate and better efficiency 
(c.f.\ Corollary~\ref{corollary:coverage_rate_bound} and Figure~\ref{fig:coverage_and_set_size_both_versions}).
We provide further discussion of the construction of the reward function in Appendix~\ref{appendix:reward_function_construction}, where we also compare the reward in Definition~\ref{def:auxiliary_reward} to other reward functions in the existing literature.

Note that \texttt{OCPQ} never actually directly observes the value of the reward function --- this is why we say that it only \emph{implicitly} maximises this reward. If it makes a query then it can calculate how much reward each prediction threshold \emph{would have got} on that round, but at the expense of getting $0$ actual reward on that round. 
To estimate the reward of each prediction threshold, \texttt{OCPQ} associates each $m \in M$ with a \emph{weight} $\hat{Z}_{m,t}$, initially set to $0$, and then updates these weights in a way that tracks the value of the corresponding prediction threshold (lines 3, 9, and 14). In particular, for each $m \in M$ and $t \leq T$, we have that $\hat{Z}_{m,t}$ is an unbiased estimator of the cumulative reward (before $t$) of threshold $m$,
as long as $\epsilon > 0$, where the expectation is over the randomness in \texttt{OCPQ} (c.f.\ Lemma~\ref{lemma:the_estimators_are_unbiased} in Appendix~\ref{appendix:proofs}). 
The main difference from the original label-efficient protocol of \citet{label_efficient_prediction} is that a query is treated as a distinct action; in \citet{label_efficient_prediction}, the learner may take an action and query the outcome on the same round.

\subsection{\textcolor{black}{Adaptive OCPQ}}

We can control the coverage-efficiency trade-off for \texttt{OCPQ} using $\beta$, and it is possible to derive an analytic lower bound on the expected coverage rate in terms of $\beta$ (c.f.\ Corollary~\ref{corollary:coverage_rate_bound}). However, the empirical coverage rate can be higher than this bound  (c.f.\ Figure~\ref{fig:coverage_and_set_size_both_versions}). This comes at the expense of lower efficiency, which means that it may be desirable to have a more fine-grained control of the coverage-efficiency trade-off. Fortunately, we can achieve this by simply adapting the value of $\beta$ over time. Specifically, we can maintain an estimate $\theta_t$ of the coverage rate that \texttt{OCPQ} has had up until round $t$, 
and
then use $\theta_t$ to dynamically adjust the value of $\beta$ over time, in order to match more closely the desired target coverage rate. We will refer to this protocol as \emph{adaptive OCPQ}.

Formally, let $\sigma$ be our target coverage rate, and let
\begin{equation}\label{eq:realised_cov}
\theta_t = \frac{1}{t} \sum_{j=1}^t   \left(\frac{\mathds{1}\{a_j = \texttt{query}\}}{\epsilon_j}\right) \cdot  \Pr(y_j \in S_j \mid h_{j-1})
\end{equation}
be our estimate of the coverage rate of \texttt{OCPQ} up until round $t$. Note that $\theta_t$ only depends on information that is available through the query rounds. To update the value of $\beta$, we divide the runtime of \texttt{OCPQ} into $k$ periods with lengths $n_1, n_2, \dots, n_k$, such that $n_1 = n$ for some $n > 0$ and such that $n_i = 2^{i-2}n$ for $i \geq 2$.\footnote{Note that $\sum_{i=1}^k n_i = T$. Also note that for $i > 1$, the length of the $i$'th period is equal to the number of rounds that passed before the start of the $i$'th period.} At the start of period $i$, we reset \texttt{OCPQ} and run it over horizon $n_i$ with $\epsilon_i = n_i^{-1/3}$, $\eta_i = n_i^{-2/3} \sqrt{\ln |M|}$, and $\beta_i = \text{clip}_{[0,1]}(2\sigma - \theta^*_{i})$ \np{if $i>1$ and $\beta_1=\sigma$}, where $\theta^*_{i}$ is the value of $\theta_t$ at the start of period $i$, and where $\text{clip}_{[a,b]}(x)$ projects $x$ into $[a,b]$. Note that if the learner has a coverage rate of $\theta$ over the first $n$ steps, and then achieves a coverage rate of (at least) $2\sigma - \theta$ over the next $n$ steps, then its total coverage rate after $2n$ steps will be (at least) $\sigma$; this is why we let $\beta_i = \text{clip}_{[0,1]}(2\sigma - \theta^*_{i})$. As we will see, this protocol improves the fit between the empirical and the target coverage rate (c.f.\ Figure~\ref{fig:coverage_and_set_size_both_versions}) while still inducing an analytic bound on the coverage rate (c.f.\ Theorem~\ref{thm:adaptive_ocpq_coverage_bound}).\footnote{The reason for why we use periods of increasing length is to ensure that this bound (Theorem~\ref{thm:adaptive_ocpq_coverage_bound}) converges at the same asymptotic rate as the corresponding bound for \texttt{OCPQ} with fixed $\beta$ (Corollary~\ref{corollary:coverage_rate_bound}) -- see Appendix~\ref{app:doubling_trick}.}


\subsection{Further Extensions}\label{sec:further-ext} 
In Appendix~\ref{app:generalisations}, we describe ways to generalise the \texttt{OCPQ} algorithm. Specifically, we describe how to accept delayed observations, how to support infinite or unknown time horizons, how to use set-valued classifiers, and how to use non-discrete label spaces. 

\np{In Appendix~\ref{app:adaptive-query}, we discuss an extension of Algorithm~1 where we replace the fixed query rate $\epsilon$ with an adaptive query rate $q_t$, which can be chosen using information up to time $t$ such as the past interaction history and the current input $x_t$. While adaptive querying may provide advantages when the sequence has some exploitable structure, we prove that, in our adversarial setting, this strategy cannot improve regret or coverage over a fixed-rate strategy. This result is not surprising and has already been proven for related but different online prediction settings~\cite{castro2023adaptive}. }


\section{Theoretical Results}
In this section, we present our theoretical results about \texttt{OCPQ}. All proofs are provided in Appendix~\ref{appendix:proofs}.


We first prove a \emph{regret bound} for \texttt{OCPQ}. 
Let $R$ be the auxiliary reward function from Definition~\ref{def:auxiliary_reward},
and define the \emph{regret} of a sequence of actions $a^T = \langle a_1 \dots a_T \rangle$ as
$$
\text{reg}(a^T) = \max_{m \in M} \sum_{t=1}^T R(m,t) - \sum_{t=1}^T R(a_t, t),
$$
where each action $a_t$ is either \texttt{query} or an element of $M$. Note that the regret is defined relative to a given classifier $C$, set of thresholds $M$, and data stream $z^T$.
%
Similarly, for a given algorithm, we define its \emph{expected regret} 
as the expectation of the regret of 
its actions.
We can now state the following result:

\begin{theorem}\label{thm:simple_learner_regret}
    For any data stream $z^T$ and classifier $C$, and any $\eta \in (0,1]$ and $\epsilon \in (0,0.5]$, the expected  regret of Algorithm~\ref{algorithm:OCP_with_queries} 
    is at most
    $$
    \frac{\ln(|M|)}{\eta} + \epsilon \cdot T + \left(\frac{\eta}{\epsilon}\right) \cdot T.
    $$
    If we set $\epsilon = T^{-1/3}$ and $\eta = T^{-2/3} \cdot \sqrt{\ln(|M|)}$ then this simplifies to
    $$
    T^{2/3} \cdot (2 \sqrt{\ln(|M|)} + 1).
    $$
\end{theorem}

In other words, with an appropriate query and learning rate, the expected regret of \texttt{OCPQ} 
grows as $O(T^{2/3})$. Using the results in \citet{lattimore2018cleaningneighborhoodclassificationadversarial}, it is possible to show that this growth rate is asymptotically optimal (see Proposition~\ref{prop:OCPPF_is_a_hard_game} in Appendix~\ref{app:bandit_formulation}). It is also worth noting that this can be achieved with a very low query rate of only $T^{-1/3}$. Of course, for this to be of interest to us, we must also establish some kind of connection between the regret and the coverage rate. Our next result provides such a connection. 

\begin{theorem}\label{thm:regret_coverage_connection}
    Given a data stream $z^T$, classifier $C$, and sequence of actions $a^T = \langle a_1 \dots a_T \rangle \in (M \cup \{\texttt{query}\})^T$, 
    if $\text{reg}(a^T) \leq N$, then the coverage of $a^T$ is at least $\beta - N/T$.
\end{theorem}

Using these two results, we can now derive a lower bound on the expected coverage rate of \texttt{OCPQ}: 

\begin{corollary}\label{corollary:coverage_rate_bound}
    On any data stream $z^T$ and for any classifier $C$, the coverage rate of Algorithm~\ref{algorithm:OCP_with_queries} satisfies
        $$
        \mathbb{E}\left[p_{\text{cover}}\right] \geq \beta - \left(\frac{2 \sqrt{\ln(|M|)} + 1}{\sqrt[3]{T}}\right)
        $$    
    if $\eta = T^{-2/3} \cdot \sqrt{\ln|M|} \leq 1$ and $\epsilon = T^{-1/3} \leq 0.5$.
\end{corollary}

Note that Corollary~\ref{corollary:coverage_rate_bound} means that $\beta$ is nearly a lower bound on the expected coverage rate.\footnote{Note that $|M|$ and $T$ can be assumed to be known beforehand. Thus, if we wish to target a specific coverage rate, then we can actually calculate the exact value of $(2 \sqrt{\ln(|M|)} + 1)/\sqrt[3]{T}$ and set the value of $\beta$ accordingly.} Since a higher coverage rate typically will correspond to lower efficiency, and vice versa, this means that we can use $\beta$ to control the trade-off between the two. 

In addition to having a lower bound on the coverage rate, it would be nice to also have an upper bound on the inefficiency. Unfortunately, it would be impossible to derive such a bound without making any assumptions about the classifier or the data stream. To see this, note that $z^T$ could be a data stream on which the correct label is always the one which $C$ assigns the lowest probability --- in that case, the inefficiency of an optimal learner will converge to $|\mathcal{Y}|$, which is also a trivial general upper bound.
However, note that Theorem~\ref{thm:simple_learner_regret} implies that the expected regret of \texttt{OCPQ} \emph{per time step} goes to 0 as $T$ increases. We should thus expect \texttt{OCPQ} to converge to the prediction threshold that maximises the auxiliary reward, and this will encourage the learner to use a prediction threshold that is as tight as possible. We can also bound $|S_t|$ by making assumptions about $C$ and $x_t$. For example:

\begin{proposition}
    For any classifier $C$, any $x_t \in \mathcal{X}$, and any threshold $m_t \in [0,1]$, let $S_t$ be constructed as in Equation~\ref{equation:prediction_set}.
    Let $\max_{y \in \mathcal{Y}} p_C(y \mid x_t) = p_\text{max}$. 
    If $m_t < p_\text{max}$ then $|S_t| \leq 1 + \frac{1-p_\text{max}}{p_\text{max} - m_t}$, and if $m_t \geq p_\text{max}$ then $|S_t| = |\mathcal{Y}|$.
\end{proposition}

We next provide a high-probability bound on the regret of \texttt{OCPQ}, which bounds its variability:

\begin{theorem}\label{thm:high_probability_bound}
    For any $\delta \in (0,1)$, if $\epsilon = T^{-1/3} \leq 0.5$ and $\eta = T^{-2/3} \cdot \sqrt{\ln(|M|)} \leq 1$, then with probability at least $1-\delta$, the regret of Algorithm~\ref{algorithm:OCP_with_queries} is at most
    \begin{align*}
     &T^{2/3}\Big(2\sqrt{\ln|M|}+4\sqrt{\ln(3/\delta)}+1\Big)\\
     +&T^{1/3}\Big(\sqrt{\ln|M|}+4\ln(3/\delta)+2\sqrt{\ln(3/\delta)\ln|M|}\Big)\\
     +&2\ln(3/\delta)\sqrt{\ln|M|}.
    \end{align*}
\end{theorem}

The Appendix provides a slightly more general version of this theorem, that applies for arbitrary choices of $\epsilon$ and $\eta$. 
Note that Theorem~\ref{thm:high_probability_bound} says that the regret of \texttt{OCPQ} is at most $O\big(T^{2/3}(\sqrt{\ln|M|}+\sqrt{\ln(1/\delta)})\big)$ with probability at least $1-\delta$.
This demonstrates that there is no data stream on which \texttt{OCPQ} has pathologically high variability.\footnote{It is worth noting that an analogous result cannot be derived for \texttt{EXP3}, see \cite{bandit_algorithms_book}, exercise 11.6.} \citet{label_efficient_prediction} also prove a high-probability regret bound for their label-efficient forecaster; 
in Appendix~\ref{app:high_probability_bound_comparison}, we
show that our bound is typically tighter. 
Finally, we note that Theorem~\ref{thm:high_probability_bound} induces a high-probability bound on the coverage rate:

\begin{figure*}[t]
    \centering
    \begin{tabular}{ccccc}
         & \begin{scriptsize}\textbf{MNIST $\rightarrow$ USPS}\end{scriptsize} & \begin{scriptsize}\textbf{MNIST $\rightarrow$ MNIST-C}\end{scriptsize} & \begin{scriptsize}\textbf{CIFAR-10 $\rightarrow$ 10-C}\end{scriptsize} & \begin{scriptsize}\textbf{CIFAR-100 $\rightarrow$ 100-C}\end{scriptsize}\\
 \begin{sideways}\hspace{0.5cm}\textbf{Coverage} \end{sideways}&

        \includegraphics[width=.22\linewidth]{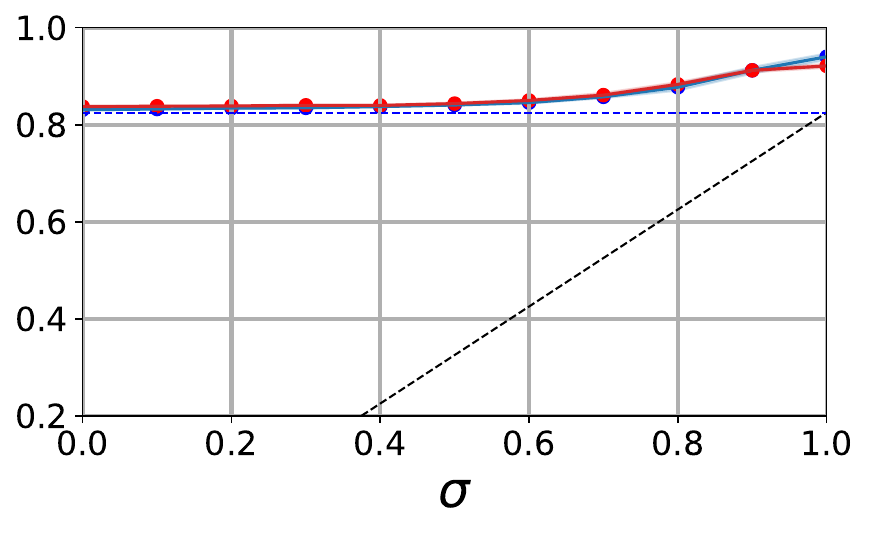}
   &
        \includegraphics[width=.22\linewidth]{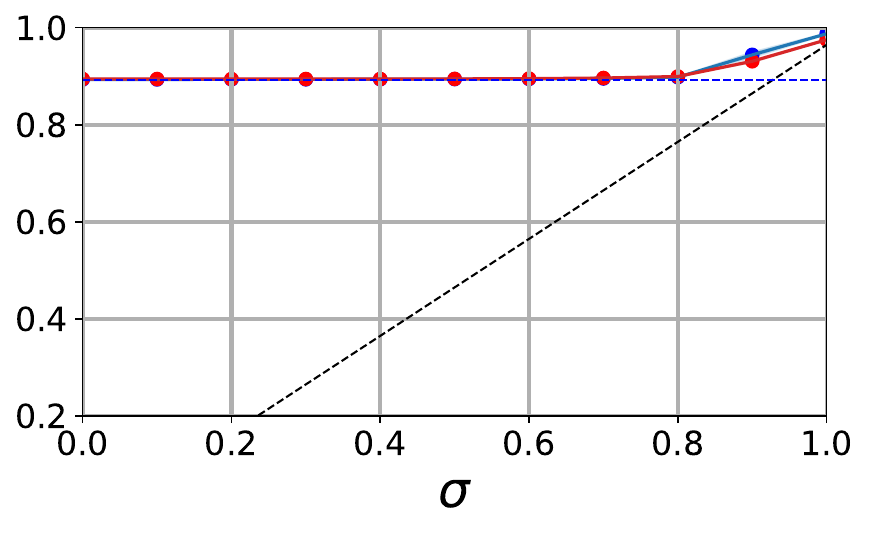}
    &
        \includegraphics[width=.22\linewidth]{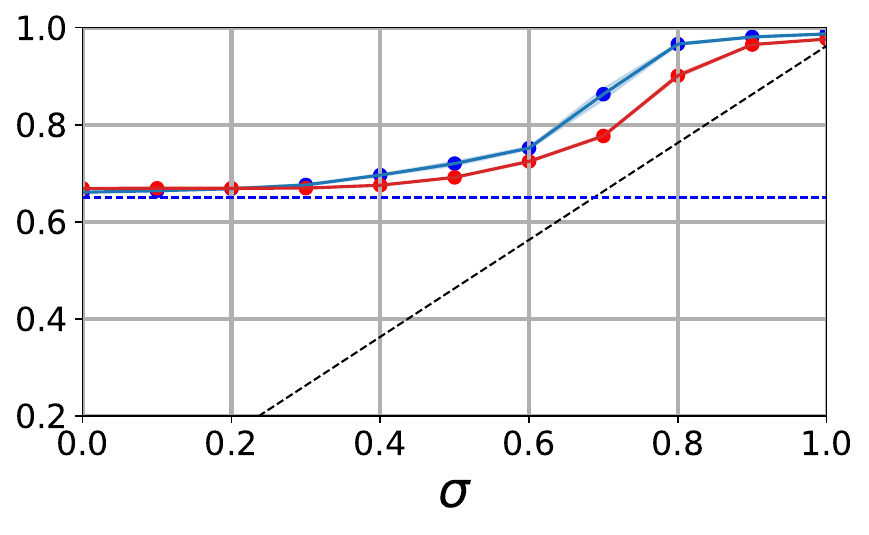}
    &
        \includegraphics[width=.22\linewidth]{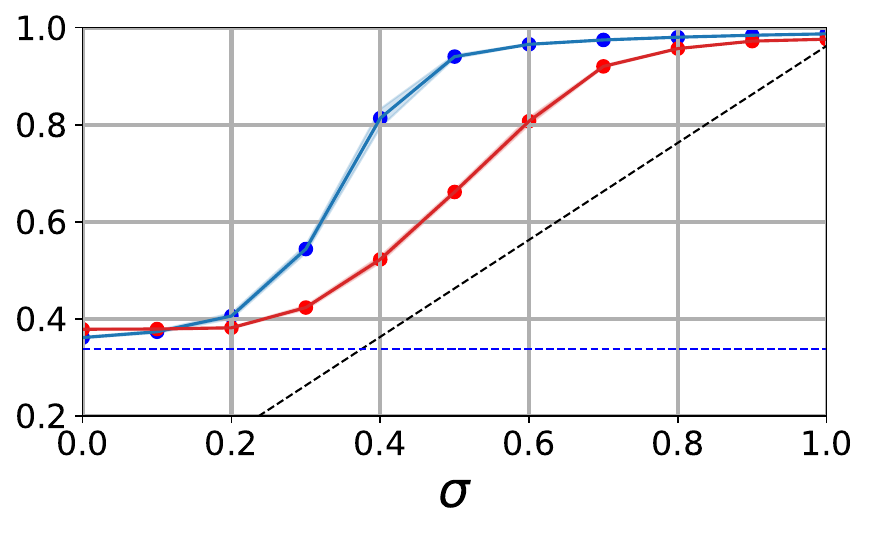}
\\ 
\begin{sideways}\hspace{0.8cm}\textbf{Inefficiency} \end{sideways}&
        \includegraphics[width=.22\linewidth]{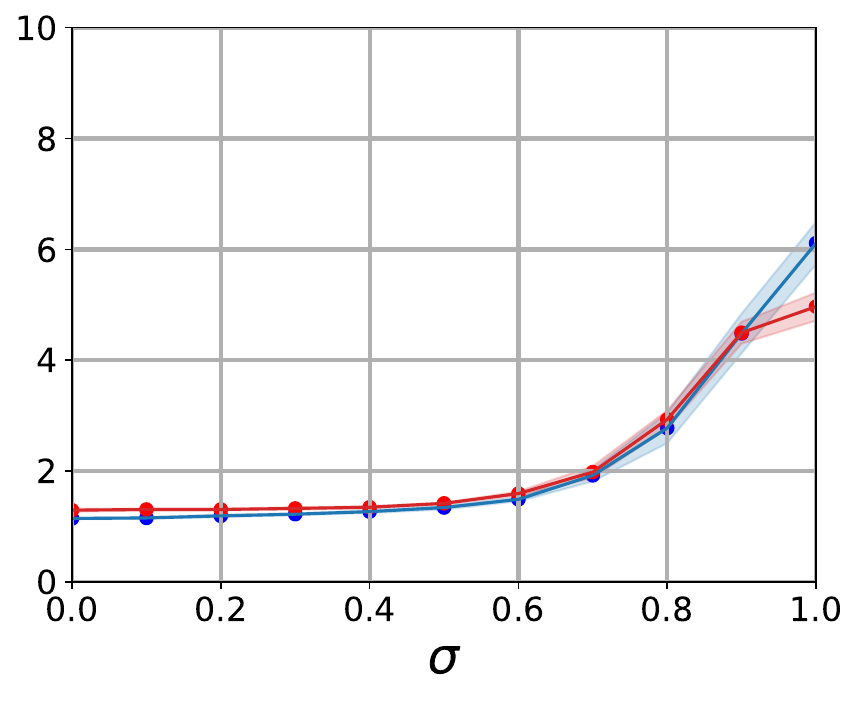}
    &
        \includegraphics[width=.22\linewidth]{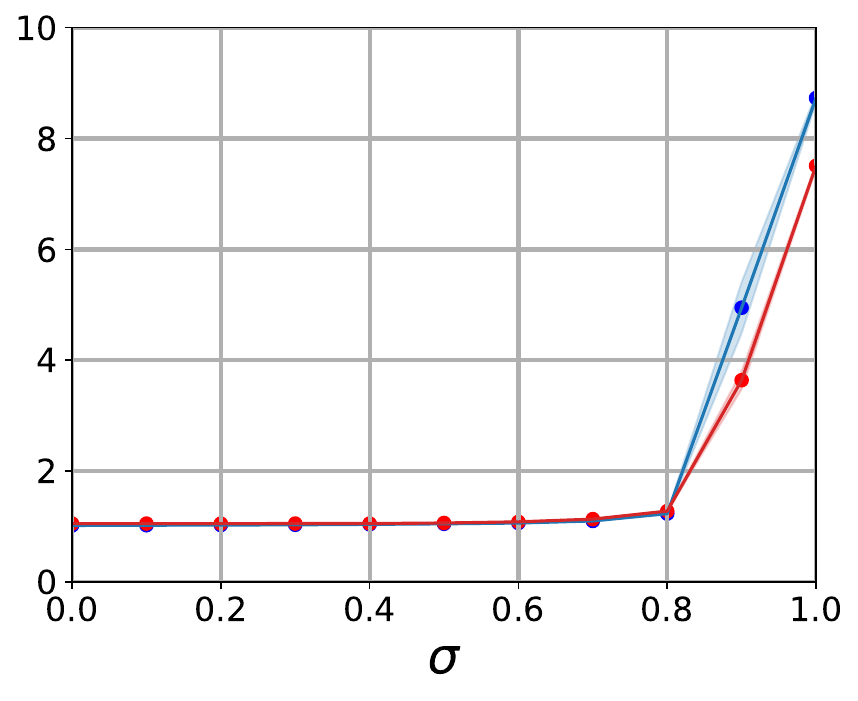}
    &
        \includegraphics[width=.22\linewidth]{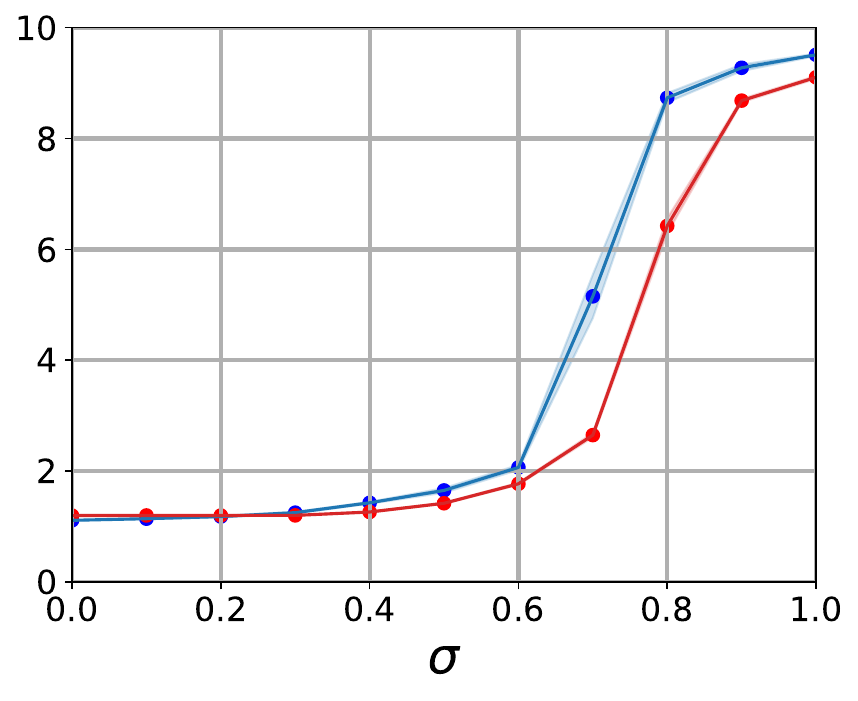}
    &
        \includegraphics[width=.22\linewidth]{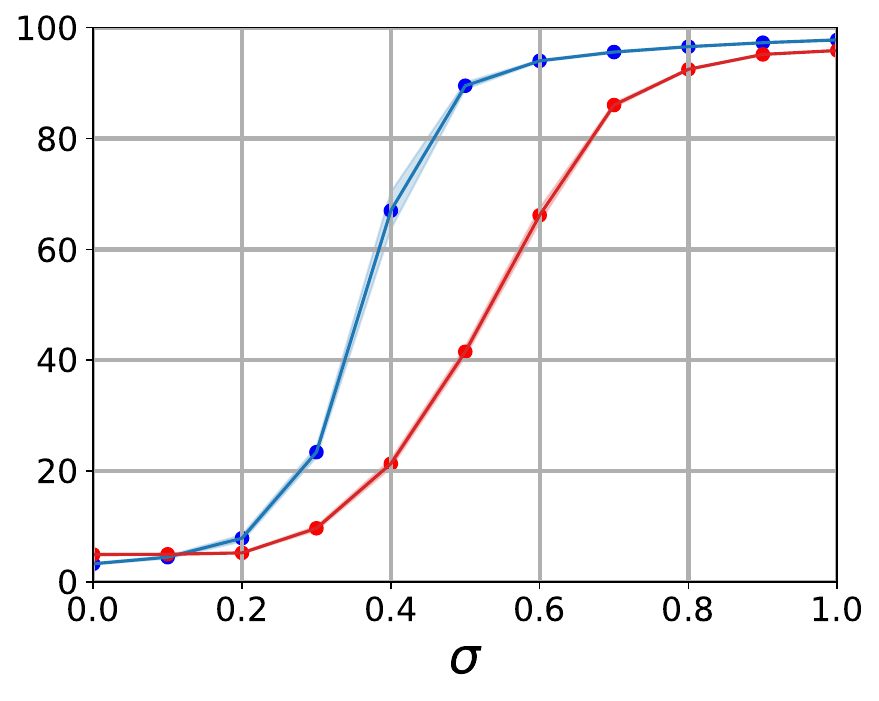}
        \end{tabular}

    \caption{Coverage and efficiency against target coverage $\sigma$.
    We use a neural network trained on MNIST \citep{mnist}, CIFAR-10, or CIFAR-100 \citep{cifar10}, and let the data stream be the USPS handwritten digit data set \citep{USPS}, MNIST-C \citep{mnist_c}, CIFAR-10-C, or CIFAR-100-C \citep{cifar10c}.
    We vary $\sigma$ from $0$ to $1$ in increments of $0.1$ and run \textbf{\textcolor{blue}{OCPQ}} and \textbf{\textcolor{red}{adaptive OCPQ}} 10 times for each value of $\sigma$ with $M = \{0.0, 0.2, 0.4, 0.6, 0.8, 1.0\}$. 
    OCPQ uses $\beta = \sigma$, $\epsilon = T^{-1/3}$, and $\eta = T^{-2/3} \cdot \sqrt{\ln|M|}$, where $T$ is the size of the data set. 
    Adaptive OCPQ uses the parameters from Theorem~\ref{thm:adaptive_ocpq_coverage_bound}, with the initial period having length $n_1 = 1{,}000$.
    The graphs show the average coverage rate and efficiency,
    with error bars showing the mean absolute deviation. 
    The \textbf{black} dashed line is the lower bound from Corollary~\ref{corollary:coverage_rate_bound}, and the \textbf{\textcolor{blue}{blue}} dashed line is the accuracy of the network on the data set when no conformal prediction is used.}
    \label{fig:coverage_and_set_size_both_versions}
\end{figure*}

\begin{corollary}
    If $\epsilon = T^{-1/3} \leq 0.5$ and $\eta = T^{-2/3} \cdot \sqrt{\ln(|M|)} \leq 1$, then for any $\delta \in (0, 1)$ we have that with probability at least $1-\delta$, the coverage rate $p_\text{cover}$ of Algorithm~\ref{algorithm:OCP_with_queries} is at least
    \begin{align*}
    \beta &- \frac{2\sqrt{\ln|M|}+4\sqrt{\ln(3/\delta)}+1}{T^{1/3}}\\
    &-\frac{\sqrt{\ln|M|}+4\ln(3/\delta)+2\sqrt{\ln(3/\delta)\ln|M|}}{T^{2/3}}\\
    &-\frac{2\ln(3/\delta)\sqrt{\ln|M|}}{T}.
    \end{align*}
\end{corollary}


\textcolor{black}{We next show that the adaptive version of \texttt{OCPQ} also induces a lower bound on the expected coverage rate:}

\begin{theorem}\label{thm:adaptive_ocpq_coverage_bound}
    Suppose that we run the adaptive version of \texttt{OCPQ} with target coverage $\sigma \in [0,1]$ over $k \geq 2$ time periods whose lengths are $n_1 \dots n_k$, with $n_i = n_1 \cdot 2^{i-2}$ for $i \geq 2$, that $\beta_1 = \sigma$ and $\beta_i = \text{clip}_{[0,1]}\left(2\sigma - \theta_{i}^\star\right)$ for $i > 1$, that $\eta_i = n_i^{-2/3} \cdot \sqrt{\ln|M|}$ and $\epsilon_i = n_i^{-1/3}$, and that the length of each period is long enough that $\eta_i \leq 1$ and $\epsilon_i \leq 0.5$. We then have that
    \begin{align*}
        \mathbb{E}[p_{cover}] \geq &\sigma - \left(\frac{1}{2}\right)^{2/3}\cdot\left(\frac{2\sqrt{\ln(|M|)} + 1}{\sqrt[3]{T}}\right)\\
        &- \frac{1}{2} \cdot \mathbb{E}\left[\max\left(2\sigma -\theta_{k}^\star - 1, 0\right)\right].
    \end{align*}
    where $\theta^\star_{k}$ is the value of $\theta_t$ (c.f. Equation~\ref{eq:realised_cov}) at the start of the $k$'th period. 
\end{theorem}

\textcolor{black}{The Appendix proves a more general form of this theorem, which holds for time periods of arbitrary length.
The term $\mathbb{E}[\max(2\sigma - 1 - \theta^\star,0)]$ is an error term that results from the fact that $\beta_i$ must be clipped to lie in $[0,1]$. In Appendix~\ref{app:proof_appendix_for_adaptive_OCPQ}, we also provide a theorem that bounds the magnitude of this term (c.f.\ Theorem~\ref{thm:clip_bound} and Corollary~\ref{corollary:clip_bound}). On an informal level, we can also note that we should expect this term to be small. 
In particular, $\sigma - 1 \leq 0$ for any $\sigma \leq 1$. Moreover, the coverage rate of \texttt{OCPQ} is typically higher than the target $\sigma$ (c.f.\ Figure~\ref{fig:coverage_and_set_size_both_versions}), and the estimate $\theta^\star$ will be close to \texttt{OCPQ}'s past coverage rate with high probability, in which case we have that $\sigma - \theta^\star < 0$. Together, this means that we should expect that $\max(2\sigma - 1 - \theta^\star,0)$ eventually will be $0$ with high probability.
Also note that this implies that the coverage bound for adaptive \texttt{OCPQ} becomes tighter than the bound for non-adaptive \texttt{OCPQ}, as $(1/2)^{2/3} \approx 0.63 <1$.
}

\section{Empirical Results}

\subsection{The Coverage-Efficiency Tradeoff}
Figure~\ref{fig:coverage_and_set_size_both_versions} shows experimental results in which we apply both the adaptive and non-adaptive versions of \texttt{OCPQ} to several standard benchmarks. 
In the experiments, we let $C$ be a neural network trained on one data set, but we let the data stream $z^T$ be a different data set, in order to simulate distribution shifts. We vary the target coverage $\sigma$ from $0$ to $1$ and show the resulting coverage rate and efficiency. As we can see, we can control the trade-off between coverage and efficiency via the choice of $\beta$. 
The coverage is always above the accuracy of the base classifier and the analytic bound (as we expect).
Moreover, the adaptive version of \texttt{OCPQ} attains a coverage rate closer to the target coverage rate, especially for harder data sets. This also means that adaptive \texttt{OCPQ} achieves better efficiency.
%
Many additional results are provided in Appendix~\ref{appendix:additional_experiments}, including experiments on varying the granularity of $M$, different choices of adversaries and reward functions, and distribution shifts in the data stream.

\begin{figure}[h]
    \centering
    \includegraphics[width=\columnwidth]
        {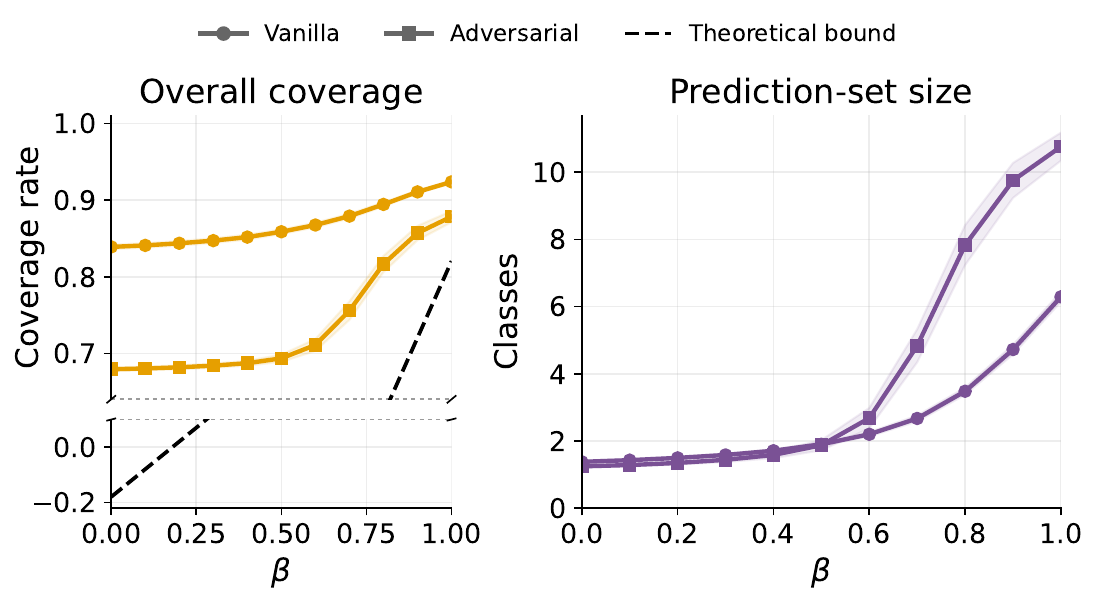}
    \caption{$\beta$ against coverage and efficiency for \texttt{OCPQ} on vanilla and adversarial WildGuardMix prompts. Solid lines and shaded regions show means and mean absolute deviations over ten runs. Dashed lines show the analytic coverage bound from Corollary~\ref{corollary:coverage_rate_bound}.}
    \label{fig:wildguard_beta}
\end{figure}

\begin{figure}[h]
    \centering
    \includegraphics[width=\columnwidth]
        {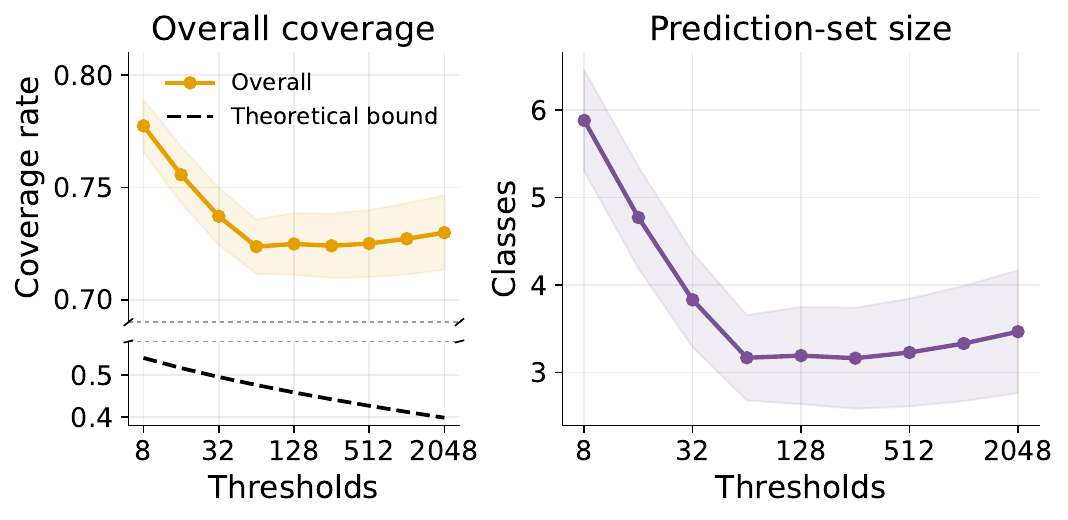}
    \caption{Sensitivity of \texttt{OCPQ} to the size of $M$ on adversarial WildGuardMix prompts at \(\beta=0.75\).
    Solid lines and shaded regions show means and mean absolute deviations
    over ten runs. The dashed line
    shows the analytic coverage bound from Corollary~\ref{corollary:coverage_rate_bound}.}
    \label{fig:wildguard_threshold_grid}
\end{figure}

\subsection{LLM Safety Monitoring}
    
We next consider safety monitoring for LLM prompts, where obtaining human labels during deployment may be costly and users may deliberately pose unsafe requests in an adversarial manner to evade a safeguard. 
We use WildGuardMix \citep{wildguard2024}, which contains both vanilla prompts and adversarial rewrites labelled according to a safety-risk taxonomy. We formulate the task as 14-way classification: one class represents safe prompts and the remaining 13 represent distinct risk categories. The adversarial prompts preserve the underlying safety task while introducing a shift in how potentially harmful requests are expressed.
Our base classifier is DistilBERT \citep{sanh2020distilbertdistilledversionbert} with a 14-way classification head, fine-tuned using only vanilla prompts.
The full experimental details are in Appendix~\ref{app:wildguardmix}.

We evaluate the monitor on 8,584 vanilla prompts and 6,404 adversarial prompts
over ten independent runs.
The distribution shift is substantial: the base classifier's accuracy falls from 0.7854 on vanilla prompts to 0.5914 on adversarial prompts. 
Figure~\ref{fig:wildguard_beta} reports coverage and efficiency as \(\beta\) varies.
As expected, adversarial prompts make coverage more expensive,
%
and the empirical coverage remains above the theoretical bound. 

We also examine sensitivity to $|M|$ at \(\beta=0.75\) on adversarial prompts (Figure~\ref{fig:wildguard_threshold_grid}).
Increasing \(|M|\) from \(8\) to \(64\) reduces inefficiency from \(5.75\) to \(2.99\) while the coverage decreases from \(0.776\) to \(0.721\). Beyond \(64\) thresholds, both quantities remain broadly stable. 

We compare \texttt{OCPQ} with Adaptive Conformal Inference (\texttt{ACI})
\cite{gibbs2021adaptive} using a fixed calibration set of 2,141 vanilla
prompts. \texttt{ACI} observes the label after every prediction and selects continuous
empirical-quantile thresholds. \texttt{OCPQ} instead receives feedback on approximately
\(5\%\) of rounds and selects among 32 uniformly spaced thresholds. Because
 \texttt{ACI}'s target coverage and \texttt{OCPQ}'s \(\beta\) have different meanings, we compare
their coverage--inefficiency frontiers rather than individual
values.
As expected, \texttt{ACI} is more efficient at comparable coverage, 
though
this gap is surprisingly small considering that \texttt{OCPQ} uses much less information.

\begin{figure}[t]
    \centering
    \includegraphics[width=0.9\columnwidth]
        {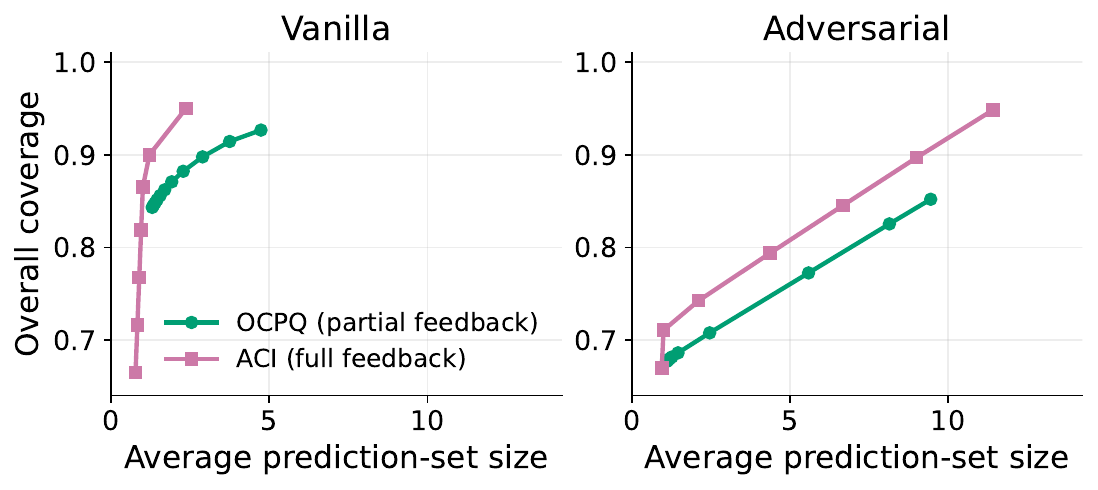}
    \caption{
        Comparison with \texttt{ACI} on vanilla and adversarial WildGuardMix prompts.  \texttt{ACI} sweeps target coverage from \(0.65\) to \(0.95\) and OCPQ sweeps \(\beta\in\{0,0.1,\ldots,1\}\). Points show means over 10 runs.
         \texttt{ACI} observes every label, whereas \texttt{OCPQ} queries \(\approx4.9\%\) and \(\approx5.4\%\) of the vanilla and adversarial streams.
    }
    \label{fig:wildguard-aci}
\end{figure}

\section{Discussion}

\paragraph{Conclusions and Significance}
We have developed an OCP method for a setting where labels are available only through rounds on which the learner declines to predict. \texttt{OCPQ} reduces this problem to a partial monitoring game and solves the resulting game using a version of the label-efficient forecaster \citep{label_efficient_prediction} that has been adapted so that queries replace actions. Its expected regret and coverage guarantees hold for any fixed black box classifier and any oblivious adversarial data stream. The expected query fraction is only $T^{-1/3}$, while the coverage deficit is $O(T^{-1/3})$. This separates \texttt{OCPQ} from bandit and semi-bandit OCP methods that obtain information from deployed actions, and from intermittent methods that require exogenous feedback. 
The experiments show that this sparse supervision can still support useful coverage and efficiency trade-offs, even under adversarial distribution shifts.
\textcolor{black}{We also show how we can adapt the parameter $\beta$ of \texttt{OCPQ} over time, and thereby obtain a tighter fit between the empirical and the target coverage rate.}


\paragraph{Limitations and Future Work} 
\textcolor{black}{Our work can be extended in a number of ways}
\np{(beyond \textcolor{black}{those} already discussed in Section~\ref{sec:further-ext}).} 
The parameter $\beta$ allows the user to control the coverage-efficiency trade-off in \texttt{OCPQ}, with larger values yielding stronger coverage guarantees (Corollary~\ref{corollary:coverage_rate_bound}). However, in practice, the achieved coverage is often higher than the guaranteed level, with a corresponding loss in efficiency (Figure~\ref{fig:coverage_and_set_size_both_versions}). 
\textcolor{black}{Adapting $\beta$ over time improves the efficiency, but it does still not guarantee that the achieved coverage will converge to the target coverage.}
\np{\textcolor{black}{Thus, w}hether we can further improve efficiency is an open question worth exploring.}
\footnote{\np{Contemporaneous work by \citet{vaze2026simultaneous} gives efficiency bounds for full-feedback OCP in terms of the Lipschitz constant of the set efficiency measure; extending these results to our setting remains open.}}

\texttt{OCPQ} works for any data stream $z^T$, but it assumes that this data stream is fixed at the start of the interaction. \np{This so-called oblivious adversary model is strong enough to capture powerful adversaries with white-box access to the classifier, the learning algorithm, and the target coverage. The only limitation is that the adversary cannot choose points adaptively based on the learner's actions.}
\textcolor{black}{However,} generalising \texttt{OCPQ} to adaptive adversaries could be an interesting extension, though this also comes with well-known conceptual challenges; see e.g., \citep{arora2012onlinebanditlearningadaptive}. 


\subsection*{Acknowledgements}
The work was supported by the Coefficient Giving grant for \emph{Verifiably Robust Conformal Probes}. The work of O.\ Simeone was also supported by an Open Fellowship of the EPSRC (EP/W024101/1), by EPSRC (EP/X011852/1), and by the ERC (No. 101198347).

\subsection*{\textcolor{black}{AI use statement}}

We used generative AI tools to prove Lemmas~\ref{lemma:custom_freedman_lemma}-\ref{lemma:estimation_error} and Theorems~\ref{thm:adaptive_ocpq_coverage_bound} and \ref{thm:clip_bound}, and the results in Appendix~\ref{app:adaptive-query}, though the resulting proofs were subsequently verified, cleaned up, and simplified before being added to the paper. Moreover, Theorem~\ref{thm:high_probability_bound} was originally proven by us, but generative AI tools were used to improve the proof in order to produce a somewhat tighter bound. The remaining theoretical results were proven without the use of generative AI tools. The code that was used in the experiments was written without the use of generative AI tools (except to look up syntax). The text in the paper was mostly written without the use of generative AI tools, except that AI tools were occasionally used to compress text, spot typos, and suggest better ways of phrasing some claims. 
Generative AI tools were also used to find and summarise relevant related papers, but all cited papers have also been read by us to ensure that the citations are correct and free from hallucinations.
Generative AI tools were \underline{not} used to develop the presented algorithms, suggest the research methodology, or to interpret the results. The rest of the required disclosure tasks (hypothesis generation, translation, dataset cleaning or reformatting) are not applicable to this work.
We take responsibility for the final content of this work, including text, claims, or artifacts produced with the aid of generative AI.

\bibliography{bibfile.bib}


\clearpage
\appendix
\thispagestyle{empty}

\onecolumn
\begin{table*}[t]
\centering
\small
\begin{tabular}{p{2.8cm}p{1.7cm}p{5.5cm}p{5.2cm}}
\toprule
\textbf{Method} & \textbf{Data Model} & \textbf{Feedback} & \textbf{Main Guarantees} \\
\midrule
\citet{wang2024efficientonlinesetvaluedclassification}
& stochastic/iid
& bandit feedback at every round
& sublinear regret, finite-time coverage \\
\addlinespace
\citet{pmlr-v267-ge25a}
& stochastic/iid
& semi-bandit feedback at every round
& sublinear regret \\
\addlinespace
\citet{yang2026onlineconformalpredictionadversarial}
& adversarial
& semi-bandit feedback at every round
& sublinear regret, finite-time coverage\\
\addlinespace
\citet{wang2025mirroronlineconformalprediction}
& adversarial
& intermittent miscoverage feedback
& sublinear regret, finite-time coverage \\
\addlinespace
\citet{zhao2025conformalizedinteractiveimitationlearning}
& adversarial
& intermittent miscoverage feedback
& finite-time coverage \\
\addlinespace
\themethod{} (ours)
& adversarial
& label revealed with fixed probability, no feedback on deployed prediction
& sublinear regret, finite-time coverage\\
\bottomrule
\end{tabular}
\caption{Feedback models and guarantees of \themethod{} compared to related OCP methods. \emph{Bandit feedback}: the learner proposes a label $y$ and receives feedback $\mathbf{1}(y_t = y)$. 
\emph{Semi-bandit feedback}: the learner proposes a set $S$ and receives feedback $y_t$ if $\mathbf{1}(y_t \in S)$. \emph{Intermittent miscoverage feedback}: the learner receives only $\mathbf{1}(y_t \in S)$ as feedback, not on every round, but with a positive probability. Regret is defined w.r.t. an optimal fixed threshold for the prediction region. Regret bounds are finite-sample. A finite-time coverage guarantee proves a finite-sample bound on the coverage rate.}
\label{tab:feedback_comparison}
\end{table*}

\section{OCP as a Bandit Problem}\label{app:bandit_formulation}

The \texttt{OCPQ} algorithm approaches OCP as an instance of a \emph{bandit problem} --- more specifically, it treats OCP as a \emph{partial monitoring game}. In this Appendix, we describe this connection in more detail. Specifically, we describe a scheme that reduces OCP (without feedback) to an instance of a partial monitoring game, thereby allowing any bandit algorithm for general partial monitoring games to be used in this setting. We show that \texttt{OCPQ} is asymptotically optimal for the resulting game, and we discuss how it compares to existing partial monitoring game algorithms \citep[especially][]{label_efficient_prediction, lattimore2019explorationoptimisationpartialmonitoring}.

For a detailed overview of partial monitoring games, see, e.g., \citet{bandit_algorithms_book}.
In brief, a partial monitoring game is a game between two players, a \emph{learner} and an \emph{adversary}. The game is defined in terms of a set of \emph{actions} $\mathcal{A}$, a set of \emph{observations} $\Omega$, and a set of \emph{states} $\mathcal{S}$. We assume that $\mathcal{A}$, $\Omega$, and $\mathcal{S}$ all are finite.
The game also includes a \emph{reward matrix} $R : \mathcal{A} \times \mathcal{S} \to [0,1]$ and an \emph{observation matrix} $\Phi : \mathcal{A} \times \mathcal{S} \to \Omega$. These matrices are both known at the start of the game. 

The game is played over $T$ rounds. At the start of the game the adversary picks a sequence of states $\xi \in \mathcal{S}^T$, which the learner does not get to observe. On each round the learner picks an action $a_t \in \mathcal{A}$, possibly by sampling from some distribution over $\mathcal{A}$. The learner then observes $\Phi(a_t, \xi_t)$ and receives $R(a_t, \xi_t)$ reward. Unlike in a typical bandit problem, the learner does \underline{not} observe its reward, except when this can be inferred from $\Phi(a_t, \xi_t)$. This setting generalises the standard adversarial bandit setting, since we can let $\Phi$ reveal the reward to the learner, or choose to give the learner more limited information (potentially with a complex structure behind the feedback that is given to the learner).

The goal of the learner is to maximise its cumulative reward over the $T$ rounds of the game. The learner is typically evaluated based on its \emph{expected regret}, defined as
$$
\max_{a^\star \in \mathcal{A}} \sum_{t=1}^T R(a^\star, \xi_t) - \sum_{t=1}^T\mathbb{E}_{A_t \sim P_t}\left[R(A_t, \xi_t)\right],
$$
where $P_t$ is the distribution over $\mathcal{A}$ that the learner samples from on round $t$. In other words, we compare the reward that the agent obtains to the reward of the single action that would have been best in hindsight, if we had to pick the same action on each round of the game.
How well we can do in a partial monitoring game depends on the structure of $R$ and $\Phi$. \citet{lattimore2018cleaningneighborhoodclassificationadversarial} show that in any partial monitoring game, the optimal worst-case expected regret will scale as either $O(1)$, $O(T^{1/2})$, $O(T^{2/3})$, or $O(T)$. For an example of an algorithm for partial monitoring games, see, e.g., \texttt{ExO}, due to \citet{lattimore2019explorationoptimisationpartialmonitoring}.

We can represent OCP without feedback as a partial monitoring game in the following way:
\begin{enumerate}
    \item We assume that we have a classifier $C : X \to \Delta(Y)$, so that $C(x)$ is a distribution over $Y$. For convenience, let $C(y, x) = C(x)_y$.
    \item We discretise $[0,1]$ into a finite set of \emph{prediction thresholds} $M \subset [0,1]$.
    \item On each round, we observe an $x_t$. If we make a query, we also observe $y_t$. Otherwise, we pick a prediction threshold $m_t \in M$, and let 
    $$
    S_t = \{y \in Y : C(y, x_t) \geq \max_{y^\star} C(y^\star, x_t) - m_t\}.
    $$
\end{enumerate}
We must also specify the action space, observation space, and state space. We could let the states $\mathcal{S}$ represent $x$ and $y$ directly, but this becomes very unwieldy, so we will instead let the state at time $t$ correspond to the \emph{optimal prediction threshold} at time $t$. Specifically, let
\begin{enumerate}
    \item $\mathcal{A} = M \cup \{\texttt{query}\}$,
    \item $\mathcal{S} = M$,
    \item $\Omega = M \cup \{\bot\}$.
\end{enumerate}
Using this representation, if $s_t = m$, then this should be interpreted as $m$ being the smallest value in $M$ such that $C(y_t, x_t) \geq \max_{y^\star} C(y^\star, x_t) - m$, where $y_t$ is the correct label at time $t$. If we make a query then we find out what this optimal threshold would have been, and otherwise we get a null observation $\bot$. 
Note that while the data stream is \enquote{actually} a sequence over $(X \times Y)^T$, we can w.l.o.g.\ represent it as a sequence over $M^T$ for a given classifier $C$ in the resulting game. Similarly, when the learner makes a query, it would normally observe $y_t$, but this is equivalent to observing the corresponding optimal prediction threshold $s_t$ for the purposes of solving this game.
Representing the states, actions, and observations in this way will make things easier to work with.

The observation matrix $\Phi$ is given by $\Phi(\texttt{query}, m) = m$ for each $m$, and $\Phi(a, m) = \bot$ for each $a \neq \texttt{query}$. 
We can use different kinds of reward functions, but for now, let the reward correspond to the \emph{auxiliary reward function} (Definition~\ref{def:auxiliary_reward}). 
Specifically, let $R(\texttt{query}, m) = 0$ for all $m$, and let $R(m_a, m_s) = 0$ whenever $m_a < m_s$. Otherwise, let $R(m_a, m_s) = 1 - m_a \cdot (1-\beta)$ for a given baseline parameter $\beta$.
The matrices below show what $\Phi$ and $R$ look like when $M = \{0.0, 0.2, 0.4, 0.6, 0.8, 1.0\}$ and $\beta = 0.75$:
\begin{equation*}
R = 
\begin{array}{|c|c|c|c|c|c|c|} 
    \hline & 0.0 & 0.2 & 0.4 & 0.6 & 0.8 & 1.0\\ 
    \hline 0.0 & 1 & 0 & 0 & 0 & 0 & 0\\ 
    \hline 0.2 & 0.95 & 0.95 & 0 & 0 & 0 & 0\\
    \hline 0.4 & 0.90 & 0.90 & 0.90 & 0 & 0 & 0\\ 
    \hline 0.6 & 0.85 & 0.85 & 0.85 & 0.85 & 0 & 0\\ 
    \hline 0.8 & 0.80 & 0.80 & 0.80 & 0.80 & 0.80 & 0\\ 
    \hline 1.0 & 0.75 & 0.75 & 0.75 & 0.75 & 0.75 & 0.75\\ 
    \hline \mathrm{query} & 0 & 0 & 0 & 0 & 0 & 0\\ 
    \hline 
\end{array}
\end{equation*}
\begin{equation*}
\Phi =
\begin{array}{|c|c|c|c|c|c|c|} 
    \hline & 0.0 & 0.2 & 0.4 & 0.6 & 0.8 & 1.0\\ 
    \hline 0.0 & \bot & \bot & \bot & \bot & \bot & \bot\\ 
    \hline 0.2 & \bot & \bot & \bot & \bot & \bot & \bot\\
    \hline 0.4 & \bot & \bot & \bot & \bot & \bot & \bot\\
    \hline 0.6 & \bot & \bot & \bot & \bot & \bot & \bot\\
    \hline 0.8 & \bot & \bot & \bot & \bot & \bot & \bot\\
    \hline 1.0 & \bot & \bot & \bot & \bot & \bot & \bot\\
    \hline \mathrm{query} & 0.0 & 0.2 & 0.4 & 0.6 & 0.8 & 1.0\\
    \hline 
\end{array}
\vspace{0.2cm}
\end{equation*}
Let this partial monitoring game be called the \emph{Online Conformal Prediction Without Feedback} (OCPWF) game. Note that we can vary $M$ and $\beta$, and the resulting game will still be considered to be an instance of OCPWF. Using this construction, any algorithm for playing partial monitoring games can be used to do online conformal prediction with partial feedback, similar to our Algorithm~\ref{algorithm:OCP_with_queries}.

\emph{Exploration by Optimisation} (\texttt{ExO}) is a state-of-the-art algorithm for general partial monitoring games, due to \citet{lattimore2019explorationoptimisationpartialmonitoring}. In brief, this algorithm maintains a probability distribution $Q_t$ over the actions $\mathcal{A}$, which can be viewed as an \enquote{exploitation policy} that is only based on the estimated reward of each action, and not its value of information. To generate its action distribution, \texttt{ExO} then solves a certain optimisation problem to find a distribution $P_t$ over $\mathcal{A}$ and a function $G_t : \mathcal{A} \times \Omega \times \mathcal{A} \to \mathbb{R}$, where $P_t$ is used to sample the action at time $t$ and $G_t$ is used to estimate the reward of each action at time $t$ (with these estimates then being used together with $Q_t$ to compute $Q_{t+1}$). This optimisation problem is designed in such a way that $P_t$ will balance exploration and exploitation, and such that $G_t$ will create accurate reward estimates in expectation. The optimisation problem does not have an analytic solution, but it is convex, and contains $|\mathcal{A}| + |\mathcal{A}|^2\cdot |\Omega|$ variables. For details, see \citet{lattimore2019explorationoptimisationpartialmonitoring}.

In Figure~\ref{fig:algorithm_1_vs_ExO_on_synthetic_data}, we compare the empirical performance of \texttt{OCPQ} against \texttt{ExO} in the OCPWF game using a number of synthetic data streams. As we can see, \texttt{ExO} consistently suffers less regret than \texttt{OCPQ} (meaning that \texttt{ExO} is better). However, this difference is not huge, and \texttt{OCPQ} is substantially faster than \texttt{ExO} (since the action distribution in \texttt{OCPQ} can be computed directly, whereas \texttt{ExO} has to solve a moderately complex optimisation problem at each time step). 
Moreover, for \texttt{OCPQ}, it is easy to calculate an exact value for the upper bound on the regret, and thereby get an exact bound on the coverage rate (Theorem~\ref{thm:simple_learner_regret} and Corollary~\ref{corollary:coverage_rate_bound}). By contrast, the regret bound of \texttt{ExO} is expressed in terms of a constant $\mathrm{opt}_*(\eta)$ which is in general difficult to estimate and infeasible to calculate analytically \citep[see][their Theorem 5]{lattimore2019explorationoptimisationpartialmonitoring}. 
Furthermore, it is also easy to bound the number of queries that \texttt{OCPQ} makes, since they are simply sampled i.i.d.\ from a fixed, known distribution. By contrast, the querying behaviour of \texttt{ExO} is an entirely emergent behaviour stemming from the optimisation problem that lies at the core of the \texttt{ExO} algorithm, and it seems like it would be very difficult to bound the resulting query rate analytically.

Using the results in \citet{lattimore2018cleaningneighborhoodclassificationadversarial}, it is straightforward to show that OCPWF is a \enquote{hard} game, meaning that no algorithm can get a worst-case expected regret that grows slower than $O(T^{2/3})$ (unless $\beta = 1$, in which case we can trivially get $0$ regret by letting $m_t = 1$ every round). Since the worst-case expected regret of \texttt{OCPQ} grows as $O(T^{2/3})$ (Theorem~\ref{thm:simple_learner_regret}), this means that \texttt{OCPQ} is asymptotically optimal in OCPWF.

\begin{proposition}\label{prop:OCPPF_is_a_hard_game}
The optimal worst-case expected regret in OCPWF is $\Theta(T^{2/3})$, unless $\beta = 1$, in which case we get 0 regret by picking $a_t = 1$ every round.
\end{proposition}
\begin{proof}
The optimal worst-case expected regret of a partial monitoring game $G=(R, \Phi)$ satisfies
\begin{enumerate}
    \item 0 if $G$ has no pairs of neighbouring actions,
    \item $\Theta(T^{1/2})$ if $G$ is locally observable and has neighbouring actions,
    \item $\Theta(T^{2/3})$ if $G$ is globally observable but not locally observable, and
    \item $\Theta(T)$ otherwise,
\end{enumerate}
see Theorem 1 in \citet{lattimore2018cleaningneighborhoodclassificationadversarial}.
This proof will use terminology from \citet{lattimore2018cleaningneighborhoodclassificationadversarial}, but for the sake of brevity we will not define this terminology here, and instead refer to \citet{lattimore2018cleaningneighborhoodclassificationadversarial} for the relevant formal definitions.

Theorem~\ref{thm:simple_learner_regret} already implies that OCPWF must be globally observable, but for the sake of completeness, let us also spell out a direct argument. We must show that for any two Pareto-optimal actions $a_1, a_2 \in \mathcal{A}$, there exists a function $f : \mathcal{A} \times \Omega \to \mathbb{R}$ such that, for any state $s$, we have that
$$
\sum_{a \in \mathcal{A}} f(a, \Phi(a, s)) = R(a_1, s) - R(a_2, s).
$$
For OCPWF, we can find an $f$ with this property by simply letting
$$
f(\texttt{query}, m) = R(a_1, m) - R(a_2, m).
$$
for each $m \in M$, and letting $f(a, \bot) = 0$ for all $a \neq \texttt{query}$ (note that this construction works for all actions in $\mathcal{A}$, not just the Pareto-optimal actions).
This shows that OCPWF is globally observable.

A game is locally observable if we can additionally let $f(a,\omega) = 0$ unless $C_{a_1} \cap C_{a_2} \subseteq C_a$. In the case of OCPWF, we have that $R(\texttt{query}, s) = 0$ for all $s$. However, for all states $s$ we have that there exists an action $a$ such that $R(a, s) > 0$ (in particular, we can let $a = s$). This means that $C_\texttt{query} = \varnothing$, and so $C_{a_1} \cap C_{a_2} \not\subseteq C_\texttt{query}$ for all Pareto-optimal actions $a_1, a_2$. As a consequence, we need an $f$ where $f(\texttt{query}, m) = 0$ for all $m \in M$. However, since $\Phi(a, m) = \bot$ for all states $m$ and actions $a \neq \texttt{query}$, this would imply that there is a constant $c \in \mathbb{R}$ such that
$$
\sum_{a \in \mathcal{A}} f(a, \Phi(a, s)) = c
$$
regardless of the state $s$. However, there are states $m_1, m_2$ such that $R(a_1, m_1) - R(a_2, m_1) \neq R(a_1, m_2) - R(a_2, m_2)$ (in particular, we could let $m_1 = s_1$ and $m_2 = s_2$). This means that OCPWF is not locally observable.

Finally, it is only possible for a game to not have any neighbouring actions if there is a single action $a$ that is (at least weakly) optimal for all states $s$. If $\beta = 1$ then $a = 1$ satisfies this condition. If $\beta < 1$ then $R(1,0) = \beta$ and $R(0,0) = 1$, meaning that $1$ is not (weakly) optimal for all states. Similarly, $R(\texttt{query}, 1) = 0$ but $R(1,1) = \beta$, and so \texttt{query} is not weakly optimal for all states. Finally, for any $m \not\in \{1, \texttt{query}\}$ we have that $R(m,1) = 0$ but $R(1,1) = 1$, and so $m$ is likewise not weakly optimal for all states. Thus, unless $\beta = 1$, OCPWF must have some neighbouring actions. 
\end{proof}

As we have noted, \citet{label_efficient_prediction} have also introduced two algorithms for a setting with a feedback structure that is similar to that of the OCPWF game. In particular, \citet{label_efficient_prediction} consider a multi-armed bandit problem setting with adversarial rewards, and where the learner does not observe the reward of any action unless it makes a query. This setting can be viewed as a special case of a partial monitoring game.
\citet{label_efficient_prediction} do not study OCP, but they consider a feedback structure that is very similar to the one we are interested in, with the minor difference that they allow the learner to both make a query and select an arm on the same round, whereas we only allow the learner to either pick an action or make a query on a given round (but not both).
The update rule for the weights $\hat{Z}$ in \texttt{OCPQ} is an adaptation of the update rule in \citet{label_efficient_prediction} for this modified feedback structure.

\begin{figure*}[htp]
    \centering
    \begin{subfigure}[b]{0.24\textwidth}
        \includegraphics[width=\linewidth]{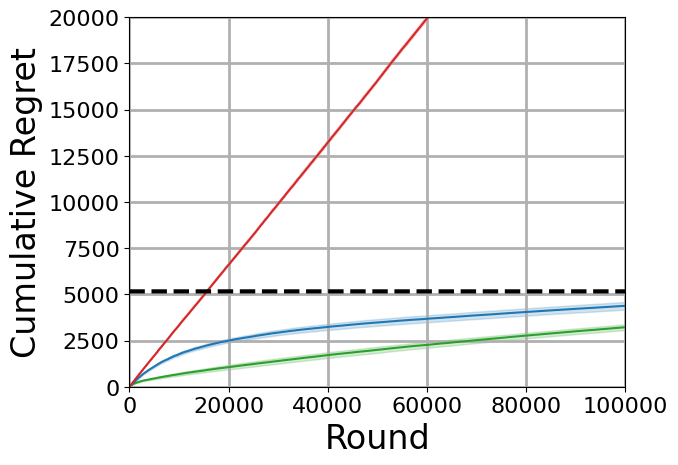}
        \caption{Random Adversary}
    \end{subfigure}
    \hfil
    \begin{subfigure}[b]{0.24\textwidth}
        \includegraphics[width=\linewidth]{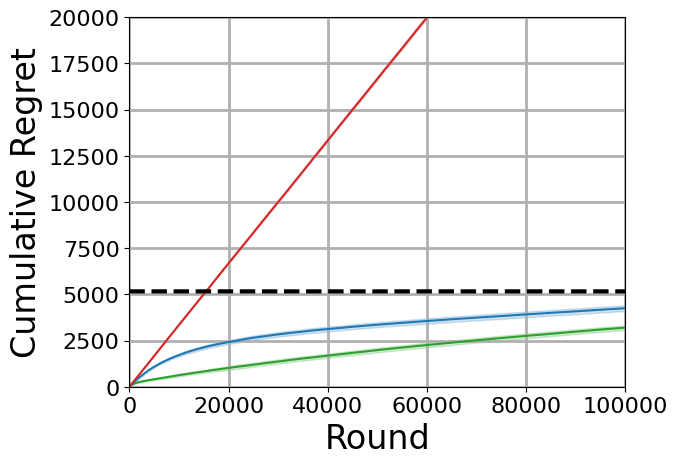}
        \caption{Alternating Adversary}
    \end{subfigure}
    \hfil
    \begin{subfigure}[b]{0.24\textwidth}
        \includegraphics[width=\linewidth]{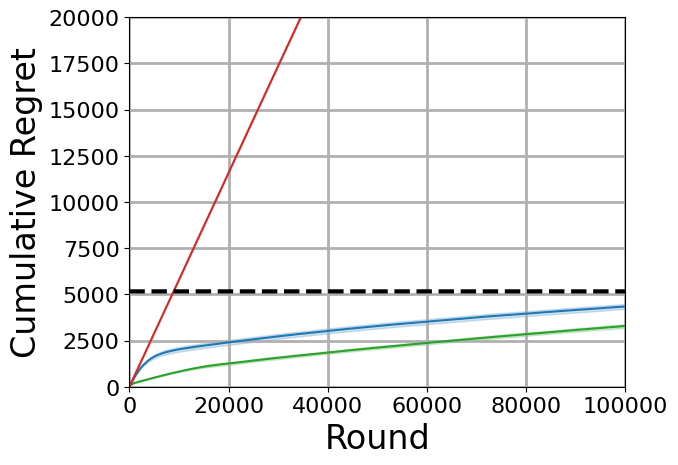}
        \caption{Deterministic Adversary}
    \end{subfigure}
    \hfil
    \begin{subfigure}[b]{0.24\textwidth}
        \includegraphics[width=\linewidth]{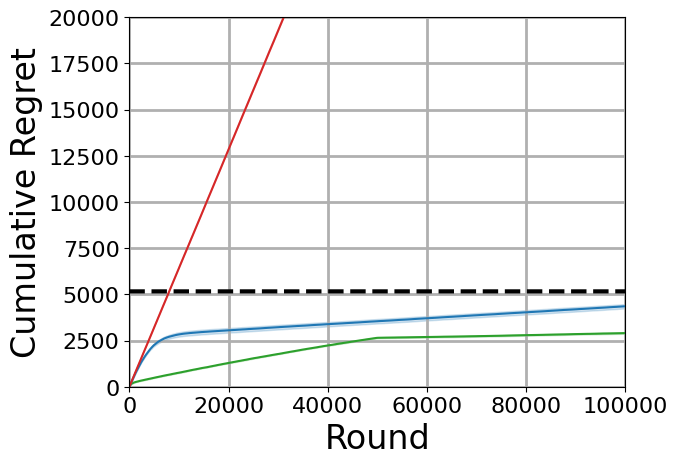}
        \caption{Hard Shift Adversary}
    \end{subfigure}
    \caption{Synthetic data for the OCPWF game with $\beta=0.75$ and $M = \{0.0, 0.2, 0.4, 0.6, 0.8, 1.0\}$, using \texttt{ExO}, \texttt{OCPQ}, and an agent that picks an action uniformly at random every round. The plots show regret rather than reward, so a lower value is better. The \enquote{random adversary} picks a random state $s_t$ each round, the \enquote{alternating adversary} cycles through the states in $\mathcal{S}$ sequentially, the \enquote{deterministic adversary} lets $s_t = 0.8$ every round, and the \enquote{hard shift adversary} lets $s_t = 1.0$ for $t \leq T/2$ and $s_t = 0.0$ for $t > T/2$. Each line shows the regret for each type of agent (green for \textcolor{teal}{\textbf{ExO}}, blue for \textcolor{blue}{\textbf{OCPQ}}, and red for the \textcolor{red}{\textbf{random agent}}) averaged over 10 runs, with error bars corresponding to the average distance to the average value (i.e., the mean absolute deviation). Note that the mean absolute deviation is so small that the error bars are barely visible in the plot. The dashed line shows the analytic bound on the expected worst-case regret of \texttt{OCPQ}, using the relevant value of $T$.}
    \label{fig:algorithm_1_vs_ExO_on_synthetic_data}
\end{figure*}

\section{Constructing the Reward Function}\label{appendix:reward_function_construction}

In order to represent OCP as a bandit problem where the arms correspond to different prediction thresholds, we must first associate each prediction threshold with a \emph{reward} (or, equivalently, a \emph{loss}). This reward function (or loss function) has to incentivise the learner to both achieve a high coverage and to achieve a high efficiency, and thereby control the trade-off between the two. We have used the \emph{auxiliary reward function} described in Definition~\ref{def:auxiliary_reward}, but this is not the only option. For example, \citet{yang2026onlineconformalpredictionadversarial} use a bandit-like algorithm for OCP, whose loss function is defined as
$$
\ell_t(\pi, \alpha) = d_t(\pi, \alpha) + a_t(\pi)
$$
where $\pi \in [0,1]$ is a prediction threshold, $1-\alpha$ is a target coverage rate, $a_t(\pi)$ is some function that penalises large prediction sets, and $d_t(\pi, \alpha)$ is
\begin{align*}
    m_t(\pi)
    + \mathds{1}\{m_t(\pi) = 0\}(\alpha + \alpha(1-\alpha))
    + \mathds{1}\{m_t(\pi) = 1\}(-\alpha(1-\alpha)),
\end{align*}
where $m_t(\pi)$ is an indicator variable that is equal to $1$ if prediction threshold $\pi$ lacks coverage on round $t$, and which is equal to $0$ otherwise. If we simplify slightly, we can see that $d_t(\pi, \alpha)$ is equal to $\alpha(2-\alpha)$ when $\pi$ has coverage, and $1 - \alpha + \alpha^2$ otherwise. They show that any learning algorithm which suffers a regret of at most $\mathrm{Reg}(T)$ with respect to this reward function satisfies
$$
\mathrm{MC}(T) - \alpha \leq \frac{1}{T} \cdot \mathrm{Reg}(T) + C_{MC}(T),
$$
where $\mathrm{MC}(T)$ is the \emph{miscoverage} of the algorithm (i.e., $\mathrm{MC}(T) = 1-p_{cover}$ in our terminology), and $C_{MC}(T)$ is a term of order $O(T^{-1})$. We can thus rewrite the bound of \citet{yang2026onlineconformalpredictionadversarial} as
$$
p_{cover} \geq 1-\alpha - \frac{\mathrm{Reg}(T)}{T} - O(T^{-1}).
$$
By contrast, we show that the reward function in Definition~\ref{def:auxiliary_reward} induces the very similar bound given by
$$
p_{cover} \geq \beta - \frac{\mathrm{Reg}(T)}{T},
$$
where $\beta$ corresponds to $1-\alpha$ (see Theorem~\ref{thm:regret_coverage_connection}). Our reward function has the benefit of being easier to construct, and the corresponding bound on the coverage rate is also easier to express. Note that since we can calculate the value of $\mathrm{Reg}(T)$ and $C_{MC}(T)$, there is no sense in which either bound gives us a \enquote{tighter} bound on $p_{cover}$ --- we can simply pick the value of $\beta$ or $\alpha$ that gives the right-hand side of the inequality the bound we desire for the coverage rate. Nonetheless, different reward functions can be used with Algorithm~\ref{algorithm:OCP_with_queries} (and similar algorithms), and it is possible that some reward functions lead to better behaviour than the one we have chosen.

To use Algorithm~\ref{algorithm:OCP_with_queries} with a different reward function, note that $\hat{Z}_{m,t+1} = \hat{Z}_{m,t} + 1 - \frac{1-R(m, t)}{\epsilon}$ if the algorithm makes a query at time $t$, and otherwise, $\hat{Z}_{m,t+1} = \hat{Z}_{m,t} + 1$. Any reward function can be used in place of $R$, and both Theorem~\ref{thm:simple_learner_regret} and \ref{thm:high_probability_bound} will still hold, as long as $R(\texttt{query}, t) = 0$ for all $t$. Theorem~\ref{thm:regret_coverage_connection} holds for any reward function such that $R(\texttt{query}, t) = 0$ for all $t$, such that $R(m,t) = 0$ if $y_t \not\in S_t$, and such that $R(m,t) \geq \beta$ if $y_t \in S_t$. In particular, the reward does not have to be linearly interpolated between $1$ and $\beta$ --- it could instead follow a convex (or concave) curve, and thereby give the algorithm a greater (or smaller) penalty for using a larger prediction threshold. In Figure~\ref{fig:accuracy_and_efficiency_with_alternative_reward}, we show the results of an experiment in which we compare the reward of Definition~\ref{def:auxiliary_reward} against two different convex reward functions. Based on this experiment, there does not appear to be a big difference between these three reward functions.

It is also worth noting that the quantiles of a distribution can be estimated using the \emph{pinball loss}, defined as
\begin{equation*}
    L_q(\hat{\tau}, \tau) =
    \begin{cases}
        q\cdot(\tau - \hat{\tau}) &\text{ if } \tau \geq \hat{\tau},\\
        (1-q)\cdot(\hat{\tau} - \tau) &\text{ otherwise}.
    \end{cases}
\end{equation*}
Here $q \in [0,1]$ is a target quantile, $\hat{\tau}$ is an estimate, and $\tau \in \mathbb{R}$ is a random variable sampled from some distribution $D$. The estimate $\hat{\tau}$ which minimises the expectation of the pinball loss is the $q$'th quantile of $D$. This loss function can be used in the context of conformal prediction; see, e.g., \citet{gibbs2021adaptiveconformalinferencedistribution, bhatnagar2023improvedonlineconformalprediction}, etc, or see \citet{ramalingam2025relationshipnoregretlearningonline} for a more extensive discussion. In particular, we can treat $\hat{\tau} \in [0,1]$ as a prediction threshold and use this to construct a prediction set
\begin{equation}\label{equation:pinball_loss_prediction_set}
    S_t = \{y \in \mathcal{Y} : s(x,y) \leq \hat{\tau}\}
\end{equation}
where $s : \mathcal{X} \times \mathcal{Y} \to [0,1]$ is an arbitrary function referred to as a \emph{non-conformity score}. For example, note that we can let $s(x,y) = \max_{y^\star \in \mathcal{Y}} p_C(y^\star \mid x) - p_C(y \mid x)$, in which case Equation~\ref{equation:pinball_loss_prediction_set} is the same as Equation~\ref{equation:prediction_set}. Moreover, given a data point $z_t = (x_t, y_t)$, we let $\tau = s(x_t, y_t)$, and we let $q$ be a target coverage rate. If each $(x_t, y_t)$ is sampled i.i.d.\ from some fixed distribution, then this loss function will be minimised by the smallest prediction threshold $\hat{\tau}$ that leads to a coverage rate of at least $q$. By letting $R(m,t) \propto -L_q(m, s_t)$, this loss function can also be used as a reward function. However, while regret with respect to the pinball loss is closely related to the coverage rate in the i.i.d.\ setting, this connection is more complicated in the adversarial setting --- for details, see \citet{ramalingam2025relationshipnoregretlearningonline}.

\section{Generalisations}\label{app:generalisations}
In this appendix we describe a number of small ways to generalise Algorithm~\ref{algorithm:OCP_with_queries}. 
These generalisations are relatively simple, but may still be helpful to point out explicitly.

\subsection{Delayed Observations}
Algorithm~\ref{algorithm:OCP_with_queries} assumes that the learner receives information immediately whenever it makes a query. However, it is easy to imagine situations where this information could be delayed, so that if the learner makes a query on round $t$, then it gets to observe $y_t$ on round $t+n$ for some $n > 0$. Fortunately, it is relatively straightforward to generalise Algorithm~\ref{algorithm:OCP_with_queries} to this setting. In particular, suppose that the observations can never be delayed by more than $n$ rounds. In this case, we can run $n$ different copies of Algorithm~\ref{algorithm:OCP_with_queries} sequentially, such that the $i$'th copy is run on round $t$ if $t \equiv i \mod{n}$ (so, for example, the 2nd copy is run on round $2$, $n+2$, $2n+2$, etc). Effectively, this splits the data stream into $n$ different interleaving data streams, each of which has length $T/n$, and each of which is solved by a separate instance of Algorithm~\ref{algorithm:OCP_with_queries}. Since each separate instance of Algorithm~\ref{algorithm:OCP_with_queries} now is guaranteed to make an observation between the rounds of its own sub-stream, our existing theoretical guarantees will apply to each sub-stream. If the expected regret of Algorithm~\ref{algorithm:OCP_with_queries} on a data stream of length $T$ is at most $T^{2/3} \cdot (2 \sqrt{\ln(|M|)} + 1)$, then the resulting (total) expected regret will be at most
\begin{align*}
    &n \cdot (T/n)^{2/3} \cdot \left(2 \sqrt{\ln(|M|)} + 1\right)\\
    = \hspace{0.2cm} &n^{1/3} \cdot T^{2/3} \cdot \left(2 \sqrt{\ln(|M|)} + 1\right).
\end{align*}
A scheme similar to this setup is used by \citet{cherenson2026staggeredintegralonlineconformal}. Note that this requires an upper bound $n$ on the delay of the observations --- if the delay is random and without any reasonable upper bound, then more sophisticated techniques may be required \citep[see, e.g,][]{NEURIPS2019_ae2a2db4}.

\subsection{Infinite or Unknown Time Horizons}\label{app:doubling_trick}
Algorithm~\ref{algorithm:OCP_with_queries} is defined relative to a finite time horizon $T$, at least if we wish to use values of $\eta$ and $\epsilon$ that lead to as low expected regret as possible.
However, in practice it is often desirable to have an \emph{anytime} algorithm. Fortunately, we can convert Algorithm~\ref{algorithm:OCP_with_queries} into an anytime algorithm using the well-known \emph{doubling trick} \citep[see, e.g.,][]{besson2018doublingtrickscantmultiarmed}. In brief, we first execute Algorithm~\ref{algorithm:OCP_with_queries} using $T=1$, and then reset it and execute it using $T=2$, and then $T=4$, then $T=8$, etc, until the actual time horizon is reached. Since the expected regret of Algorithm~\ref{algorithm:OCP_with_queries} for a known time horizon $T$ is at most $T^{2/3} \cdot (2 \sqrt{\ln(|M|)} + 1)$, the resulting expected regret of the modified algorithm for an arbitrary unknown time horizon $T$ will be at most
\begin{align*}
    &\sum_{i=0}^{\log_2{T}} \left(2^i\right)^{2/3} \cdot \left(2 \sqrt{\ln(|M|)} + 1\right)\\ 
    = \hspace{0.2cm} &\left(\frac{2^{2/3} \cdot T^{2/3} - 1}{2^{2/3}-1}\right)  \cdot  \left(2 \sqrt{\ln(|M|)} + 1\right)\\
    \leq \hspace{0.23cm} &2.71 \cdot T^{2/3} \cdot \left(2 \sqrt{\ln(|M|)} + 1\right).
\end{align*}

\subsection{Set-valued Classifiers}
We have assumed that the classifier $C$ outputs a distribution over the label space $Y$, and that each instance $x_t$ only has a single true label $y_t$. However, what if each instance $x_t$ is instead associated with a \emph{set} of labels $Y_t \subseteq Y$, and $C$ assigns an independent probability to each label (so $C$ has the type signature $X \to [0,1]^Y$)? It is straightforward to adapt Algorithm~\ref{algorithm:OCP_with_queries} to this setting -- the only change that is required is that we let $s_t$ be the smallest value in $M$ such that $C(y, x_t) \geq \max_{y^\star} C(y^\star, x_t) - s_t$ for \emph{all} $y \in Y_t$. For the theoretical guarantees, simply redefine the coverage rate to require that $Y_t \subseteq S_t$ instead of $y_t \in S_t$, and let the auxiliary reward function be $0$ whenever $Y_t \not\subseteq S_t$, and all theoretical guarantees will still hold.

\subsection{Infinite Label Spaces}
We have assumed that the label space $Y$ is finite. However, what if $Y$ instead consists of an interval? It is fairly easy to adapt Algorithm~\ref{algorithm:OCP_with_queries} to this setting, at least if $Y$ is bounded. In particular, suppose $Y$ is a bounded interval $I \subset \mathbb{R}$. We then let $M$ be a finite subset of $[0, |I|]$ such that $0 \in M$ and $|I| \in M$. If Algorithm~\ref{algorithm:OCP_with_queries} picks prediction threshold $m$ on round $t$, then we let $S_t = [C(x) - m, C(x) + m]$. In the specification of Algorithm~\ref{algorithm:OCP_with_queries}, we should let $s_t$ be the smallest value in $M$ such that  $y_t \in [C(x) - s_t, C(x) + s_t]$. Moreover, when the algorithm makes a query, we should let $\hat{Z}_{i, t+1} = \hat{Z}_{i,t} +  1 - \frac{(i/|I|) \cdot (1-\beta)}{\epsilon}$ instead of $\hat{Z}_{i,t} +  1 - \frac{m \cdot (1-\beta)}{\epsilon}$ for all $i \in M$ such that $i \geq s_t$. We should also modify the definition of the auxiliary reward function so that $R(m, t) = 1 - (m/|I|) \cdot (1 - \beta)$ when $y_t \in [C(x) - m, C(x) + m]$. With these changes, all the theoretical guarantees we have derived will still hold.

\section{Additional Experiments}\label{appendix:additional_experiments}

In this appendix, we provide the results of several additional experiments that we were not able to fit in the main text. First of all, in Figure~\ref{fig:M_vs_coverage_and_set_size}, we show how the choice of $M$ impacts the coverage rate and efficiency of \texttt{OCPQ} on a number of real-world data sets. As we can see, larger sizes of $M$ seems to generally lead to better performance, but only up to some asymptote. 
In Figure~\ref{fig:pareto_frontiers} we explore the Pareto trade-off between coverage and efficiency, and compare the performance of \texttt{OCPQ} to the \enquote{trivial} conformal predictor that outputs $\mathcal{Y}$ with probability $p$ and $\varnothing$ with probability $1-p$. As we can see, \texttt{OCPQ} improves the Pareto frontier compared to the trivial predictor (as we should expect).
In Figure~\ref{fig:accuracy_and_efficiency_over_time} we show how the coverage rate and inefficiency of \texttt{OCPQ} evolve over time for a number of real-world data sets (for some specific hyperparameter values). As in our other experiments, we can see that the (empirical) coverage rate of \texttt{OCPQ} typically is quite a bit higher than the analytic bound.
In Figure~\ref{fig:beta_vs_average_threshold} we measure how the average prediction threshold $m$ used by \texttt{OCPQ} depends on $\beta$, showing that it generally follows a sigmoidal shape with the same inflexion point that we observe in Figure~\ref{fig:coverage_and_set_size_both_versions}. 
In Figure~\ref{fig:accuracy_and_reward_per_m} we show the coverage rate of each prediction threshold in $M$ for a number of data sets, as well as the reward and expected reward of each prediction threshold for a specific example value of $\beta$. As we can see from these plots, many values of $\beta$ will lead to either of the two \enquote{trivial} prediction thresholds $m=0$ or $m=1$ being the prediction threshold that maximises the expected reward. This reaffirms the impression that picking a good value of $\beta$ may be a fairly delicate matter for the non-adaptive version of \texttt{OCPQ}.
Figure~\ref{fig:accuracy_and_coverage_with_hard_shift} shows how the coverage rate and inefficiency of \texttt{OCPQ} evolve over time when there is a hard distributional shift within the data stream. As we can see, \texttt{OCPQ} is affected by these distributional shifts, but quickly recovers (and the coverage rate stays above the analytic bound, as we should expect).
In Figure~\ref{fig:accuracy_and_efficiency_with_alternative_reward} we use \texttt{OCPQ} with three different reward functions --- namely, the reward function given in Definition~\ref{def:auxiliary_reward}, and two convex reward functions that are designed to give the algorithm a greater penalty (compared to the reward of Definition~\ref{def:auxiliary_reward}) for using a larger prediction threshold. As we can see, all three reward functions seem to induce similar behaviour in \texttt{OCPQ} (so, in particular, they are not effective for making the empirical cover rate of \texttt{OCPQ} closer to the analytic bound).

\begin{figure*}[htp]
    \centering
    \begin{subfigure}[b]{0.24\textwidth}
        \includegraphics[width=\linewidth]{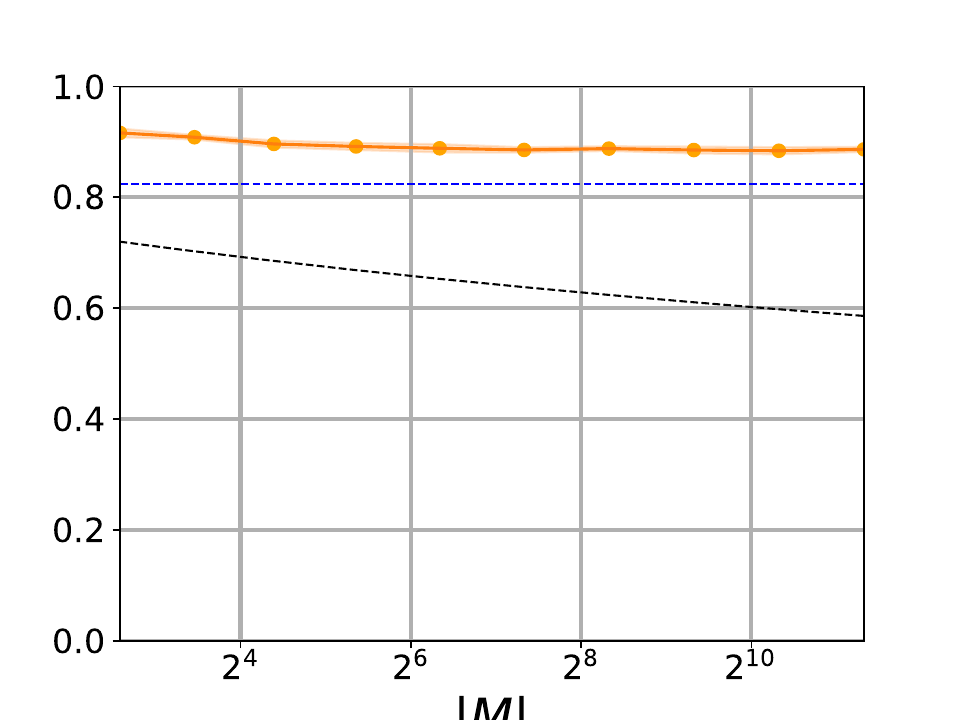}
        \caption{Coverage, USPS.}
    \end{subfigure}
    \hfil
    \begin{subfigure}[b]{0.24\textwidth}
        \includegraphics[width=\linewidth]{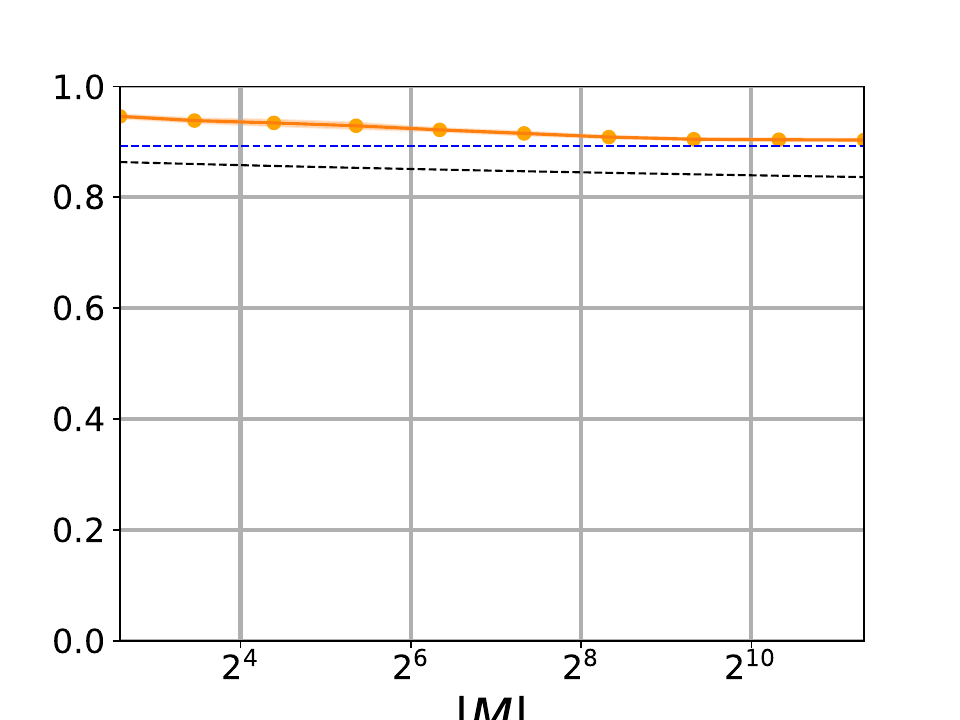}
        \caption{Coverage, MNIST-C.}
    \end{subfigure}
    \hfil
    \begin{subfigure}[b]{0.24\textwidth}
        \includegraphics[width=\linewidth]{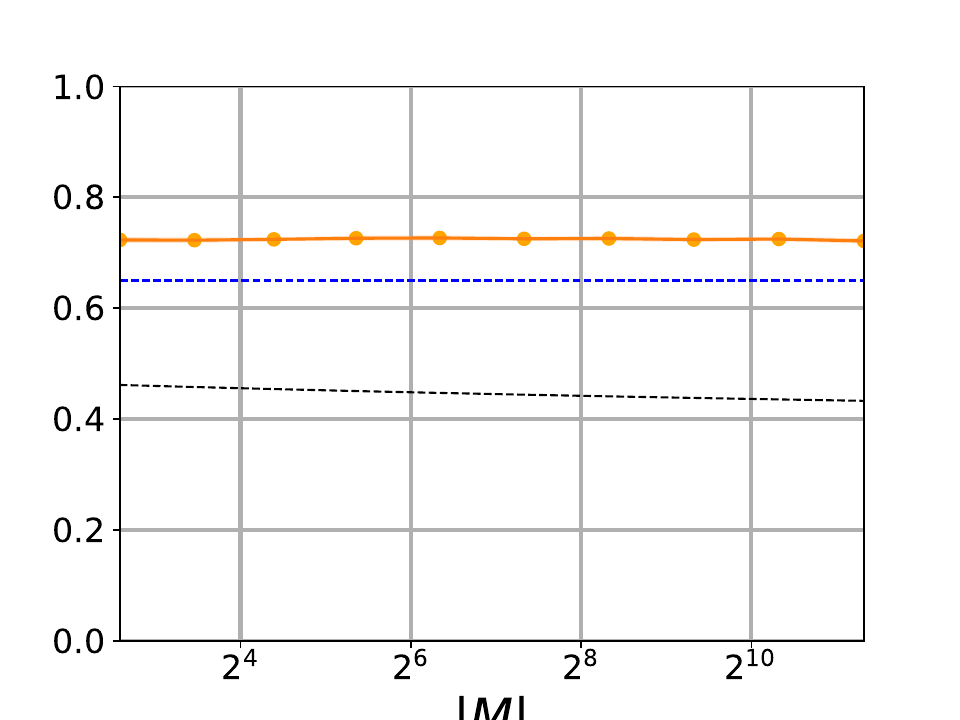}
        \caption{Coverage, CIFAR-10-C.}
    \end{subfigure}
    \hfil
    \begin{subfigure}[b]{0.24\textwidth}
        \includegraphics[width=\linewidth]{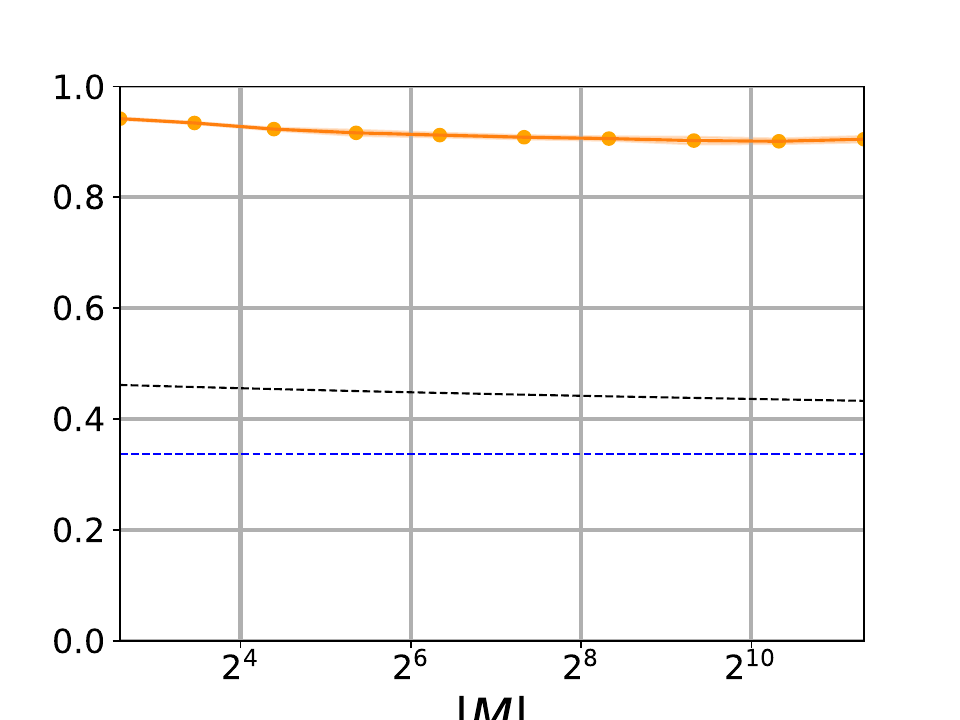}
        \caption{Coverage, CIFAR-100-C.}
    \end{subfigure}

    \begin{subfigure}[b]{0.24\textwidth}
        \includegraphics[width=\linewidth]{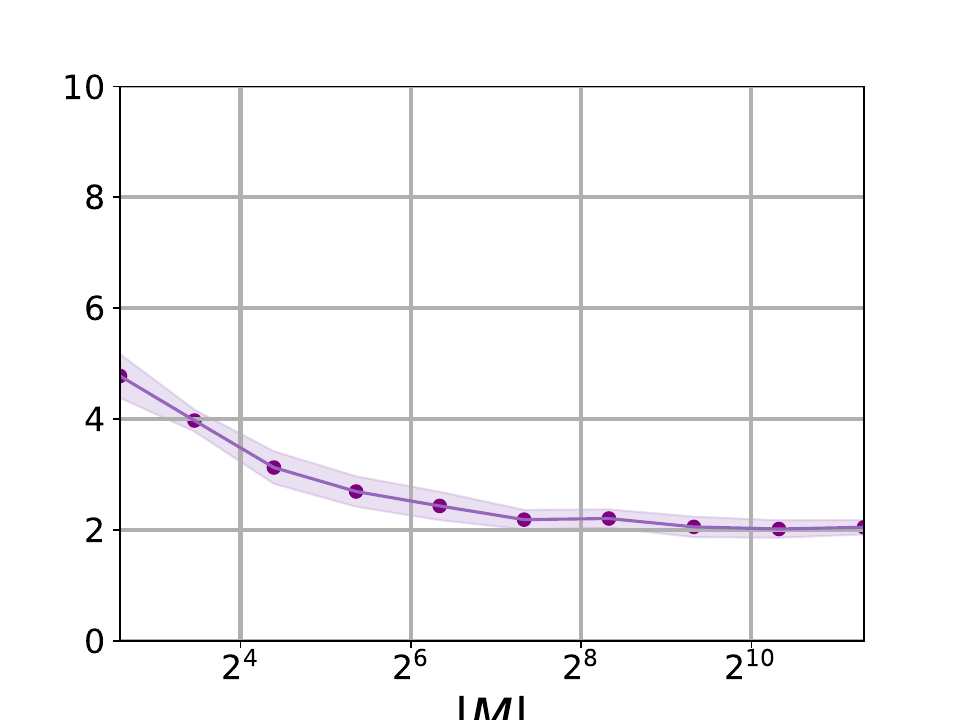}
        \caption{Efficiency, USPS.}
    \end{subfigure}
    \hfil
    \begin{subfigure}[b]{0.24\textwidth}
        \includegraphics[width=\linewidth]{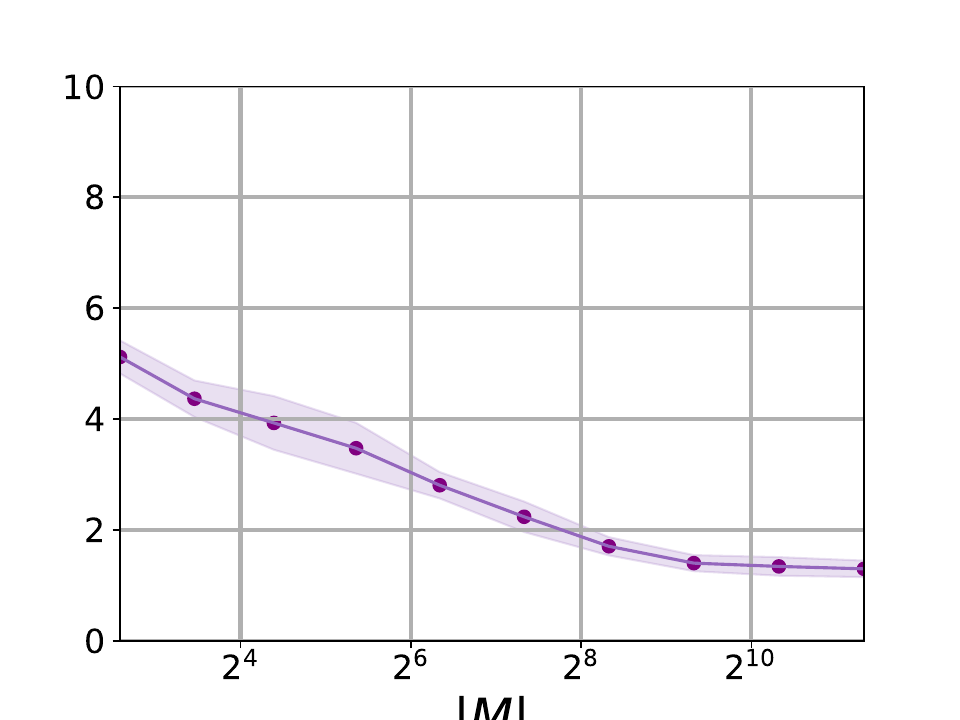}
        \caption{Efficiency, MNIST-C.}
    \end{subfigure}
    \hfil
    \begin{subfigure}[b]{0.24\textwidth}
        \includegraphics[width=\linewidth]{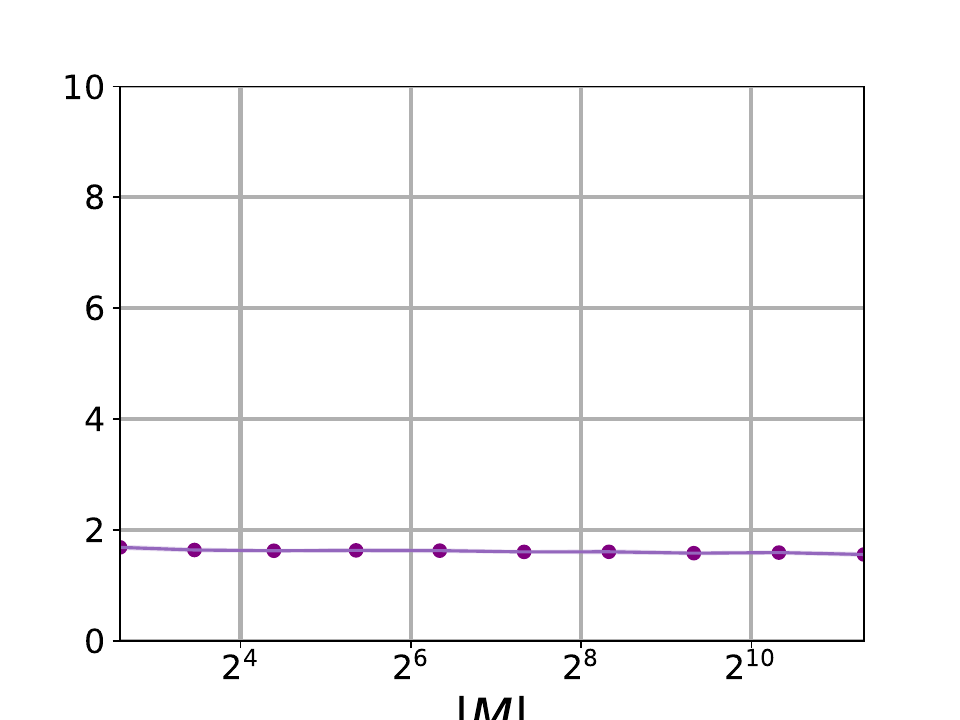}
        \caption{Efficiency, CIFAR-10-C.}
    \end{subfigure}
    \hfil
    \begin{subfigure}[b]{0.24\textwidth}
        \includegraphics[width=\linewidth]{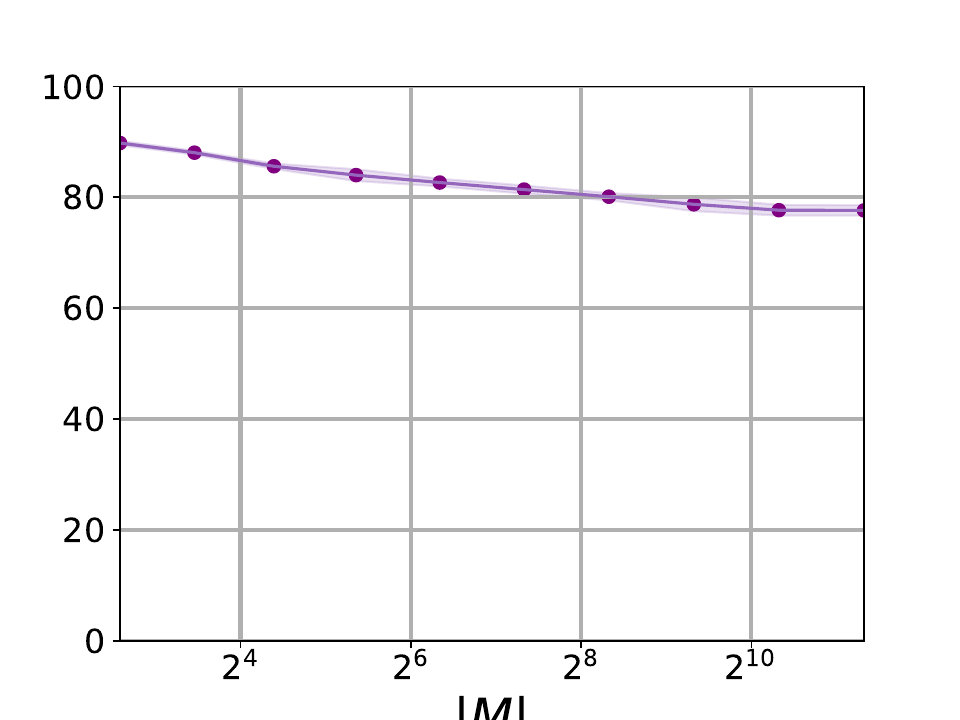}
        \caption{Efficiency, CIFAR-100-C.}
    \end{subfigure}
    \caption{These graphs show how the choice of $|M|$ affects the \textcolor{orange}{\textbf{coverage rate}} and \textcolor{violet}{\textbf{efficiency}} on a number of real-world data sets. We use a neural network that is trained on one data set but evaluated on a different data set (MNIST to USPS or MNIST-C, CIFAR-10 to CIFAR-10-C, or CIFAR-100 to CIFAR-100-C). We let $M$ consist of evenly spaced prediction thresholds, vary the total number of these thresholds, run \texttt{OCPQ} 10 times for each value of $|M|$, and record the resulting coverage rates and average prediction set sizes. 
    We use $\beta = 0.9$ for MNIST-C and USPS and $\beta = 0.5$ for CIFAR-10-C and CIFAR-100-C -- these choices were informed by the data shown in Figure~\ref{fig:coverage_and_set_size_both_versions}.
    The graphs show the average of these values across 10 runs, with error bars corresponding to the average distance to the average value (i.e., the mean absolute deviation). Note that the mean absolute deviation is so small that the error bars are barely visible in the plot. Also note that the x-axes are logarithmic. For the graphs displaying the coverage rates, the black dashed line corresponds to the \textbf{analytic lower bound} on the expected coverage rate, and the blue line corresponds to the baseline \textbf{\textcolor{blue}{accuracy}} of the network on the relevant data set.}
    \label{fig:M_vs_coverage_and_set_size}
\end{figure*}

\begin{figure*}[htp]
    \centering
    \begin{subfigure}[b]{0.24\textwidth}
        \includegraphics[width=\linewidth]{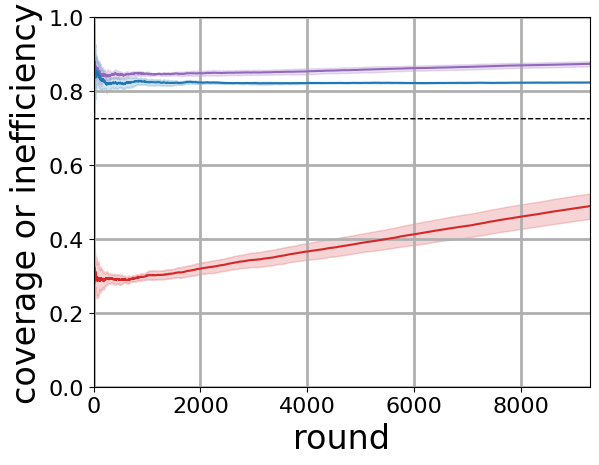}
        \caption{USPS}
    \end{subfigure}
    \hfil
    \begin{subfigure}[b]{0.24\textwidth}
        \includegraphics[width=\linewidth]{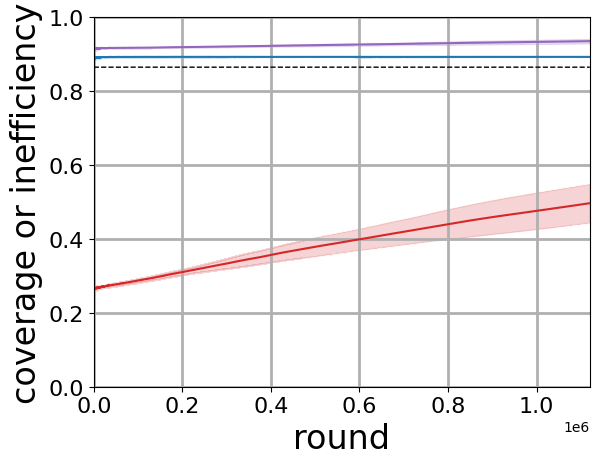}
        \caption{MNIST-C}
    \end{subfigure}
    \hfil
    \begin{subfigure}[b]{0.24\textwidth}
        \includegraphics[width=\linewidth]{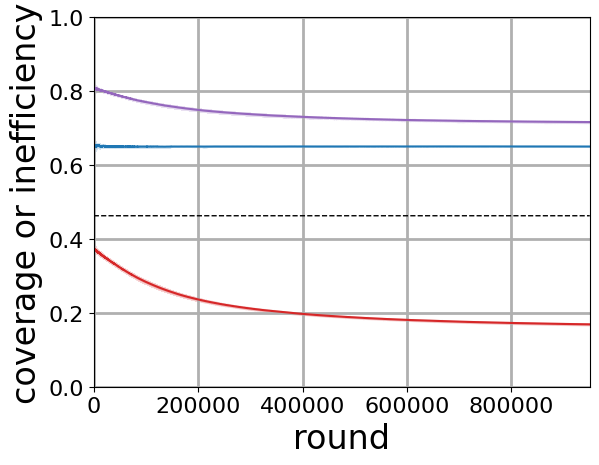}
        \caption{CIFAR-10-C}
    \end{subfigure}
    \hfil
    \begin{subfigure}[b]{0.24\textwidth}
        \includegraphics[width=\linewidth]{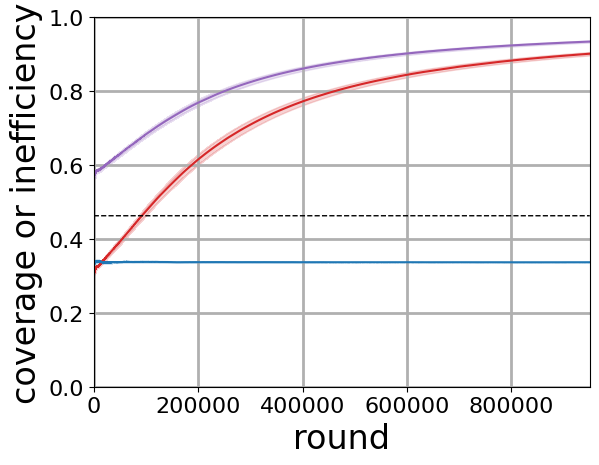}
        \caption{CIFAR-100-C}
    \end{subfigure}
    \caption{These graphs show how the long-run coverage rate and inefficiency evolve over time on a number of real-world data sets. We use a neural network that is trained on one data set but evaluated on a different data set (MNIST to USPS or MNIST-C, CIFAR-10 to CIFAR-10-C, or CIFAR-100 to CIFAR-100-C). We let $M = \{0.0, 0.2, 0.4, 0.6, 0.8, 1.0\}$, $\beta = 0.9$ for USPS and MNIST-C, and $\beta = 0.5$ for CIFAR-10-C and CIFAR-100-C. \texttt{OCPQ} is run 10 times for each data set, with the plot showing the average and the mean absolute deviation for each value over these 10 runs. The purple line shows the \textcolor{violet}{\textbf{coverage rate}}, the red line shows the \textcolor{red}{\textbf{inefficiency}}, the blue line shows the \textbf{\textcolor{blue}{accuracy}} of the underlying classifier, and the dashed black line shows the \textbf{analytic bound} on the (final) cover rate. The efficiency has been normalised to lie between 0 and 1.} 
    \label{fig:accuracy_and_efficiency_over_time}
\end{figure*}

\begin{figure*}[htp]
    \centering
    \begin{subfigure}[b]{0.24\textwidth}
        \includegraphics[width=\linewidth]{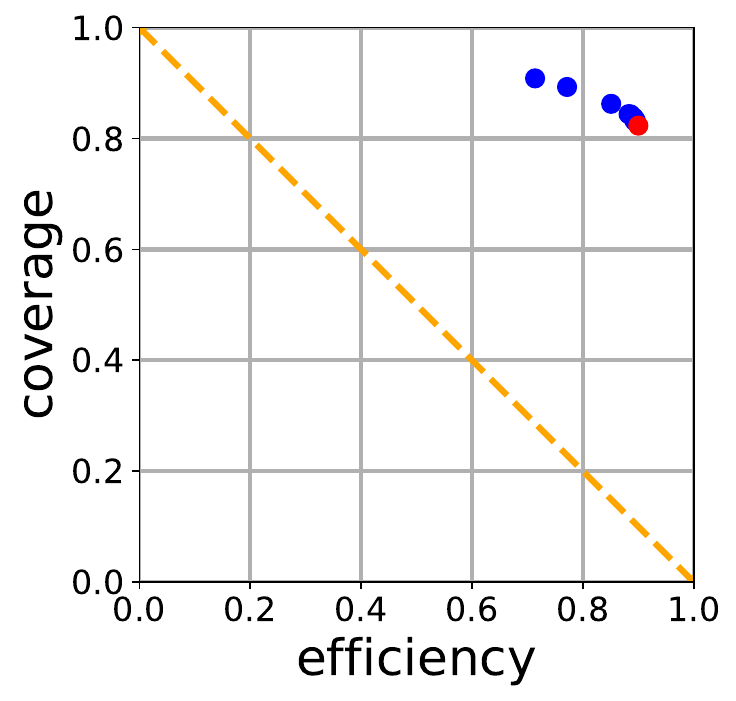}
        \caption{USPS}
    \end{subfigure}
    \hfil
    \begin{subfigure}[b]{0.24\textwidth}
        \includegraphics[width=\linewidth]{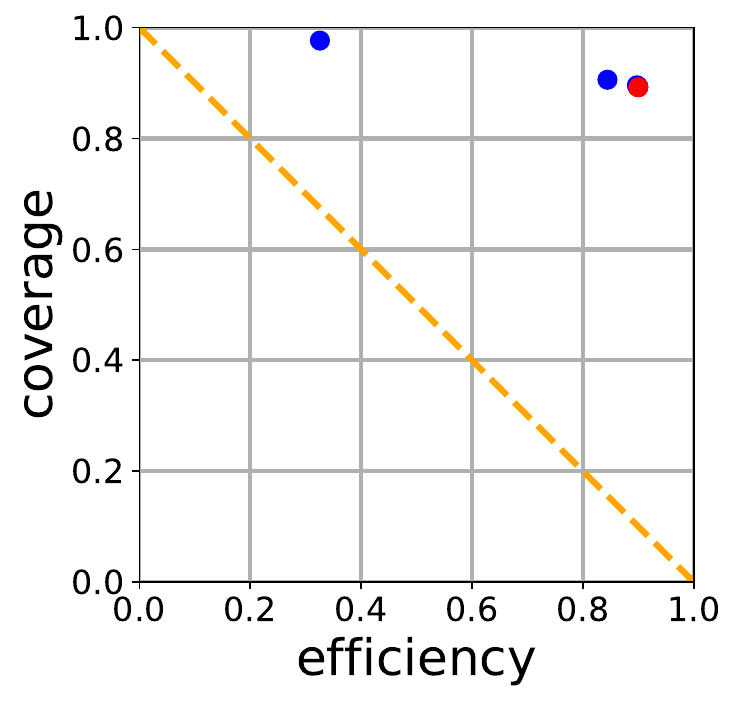}
        \caption{MNIST-C}
    \end{subfigure}
    \hfil
    \begin{subfigure}[b]{0.24\textwidth}
        \includegraphics[width=\linewidth]{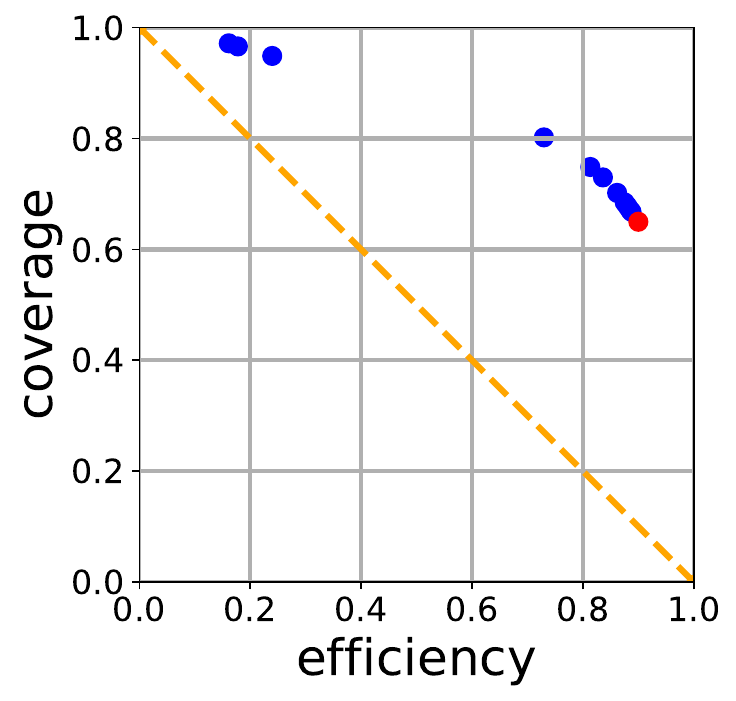}
        \caption{CIFAR-10-C}
    \end{subfigure}
    \hfil
    \begin{subfigure}[b]{0.24\textwidth}
        \includegraphics[width=\linewidth]{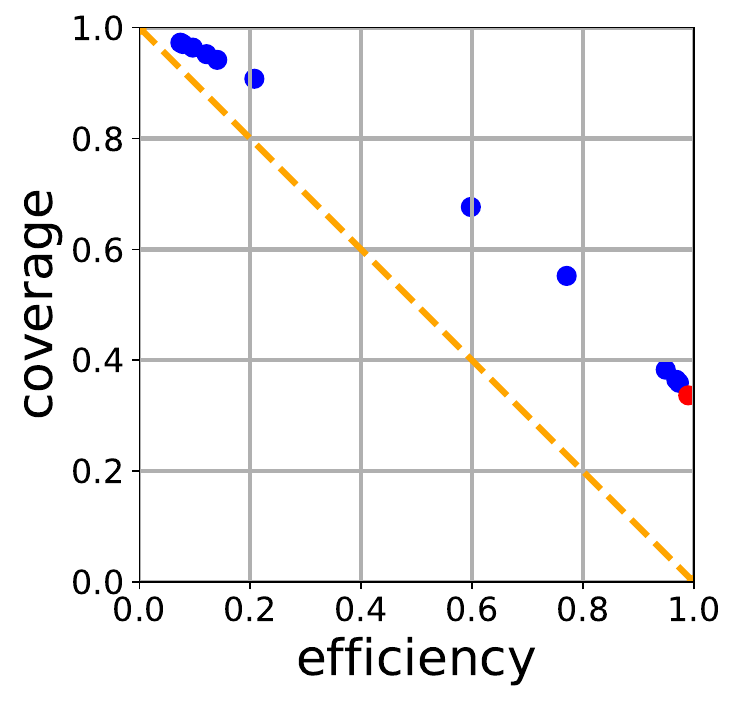}
        \caption{CIFAR-100-C}
    \end{subfigure}
    \caption{These graphs show the Pareto frontier over the coverage rate and the efficiency. We use a neural network that is trained on one data set but evaluated on a different data set (MNIST to USPS or MNIST-C, CIFAR-10 to CIFAR-10-C, or CIFAR-100 to CIFAR-100-C). Each blue dot corresponds to a run of \textbf{\textcolor{blue}{OCPQ}}, where $\beta$ has been varied from $0$ to $1$ in increments of $0.1$, and the red dot corresponds to the underlying \textbf{\textcolor{red}{classifier}} without conformal prediction. The orange line corresponds to the \textbf{\textcolor{orange}{trivial}} conformal predictor that simply outputs $Y$ with probability $p$ and $\varnothing$ with probability $1-p$. In these plots, we have normalised the efficiency to lie between $0$ ($S_t = \mathcal{Y}$) and $1$ ($S_t = \varnothing$), meaning that a higher value is better.} 
    \label{fig:pareto_frontiers}
\end{figure*}

\begin{figure*}[htp]
    \centering
    \begin{subfigure}[b]{0.24\textwidth}
        \includegraphics[width=\linewidth]{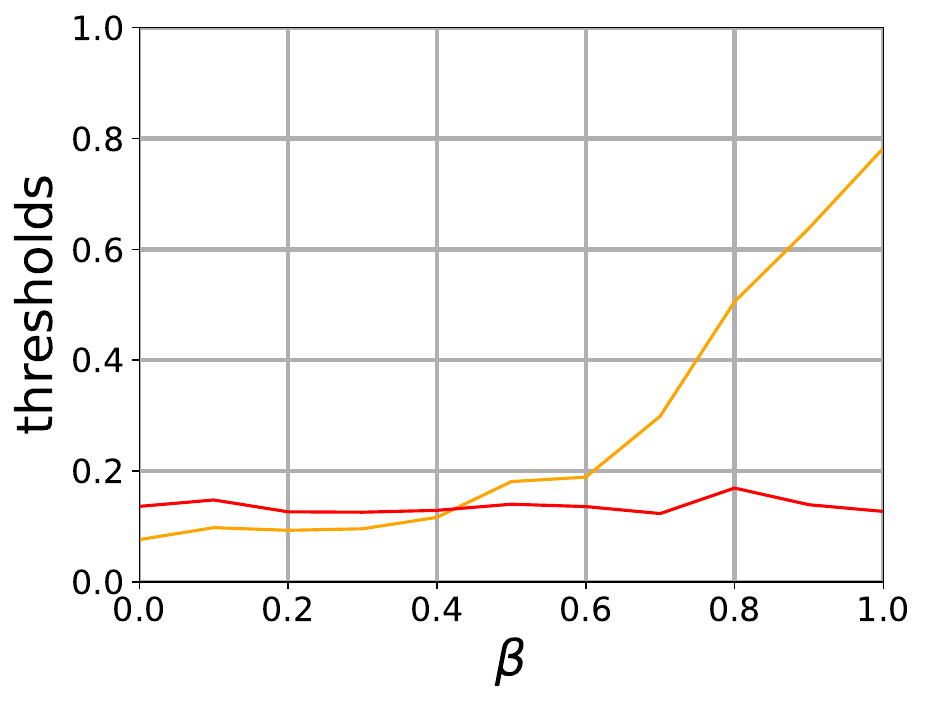}
        \caption{USPS}
    \end{subfigure}
    \hfil
    \begin{subfigure}[b]{0.24\textwidth}
        \includegraphics[width=\linewidth]{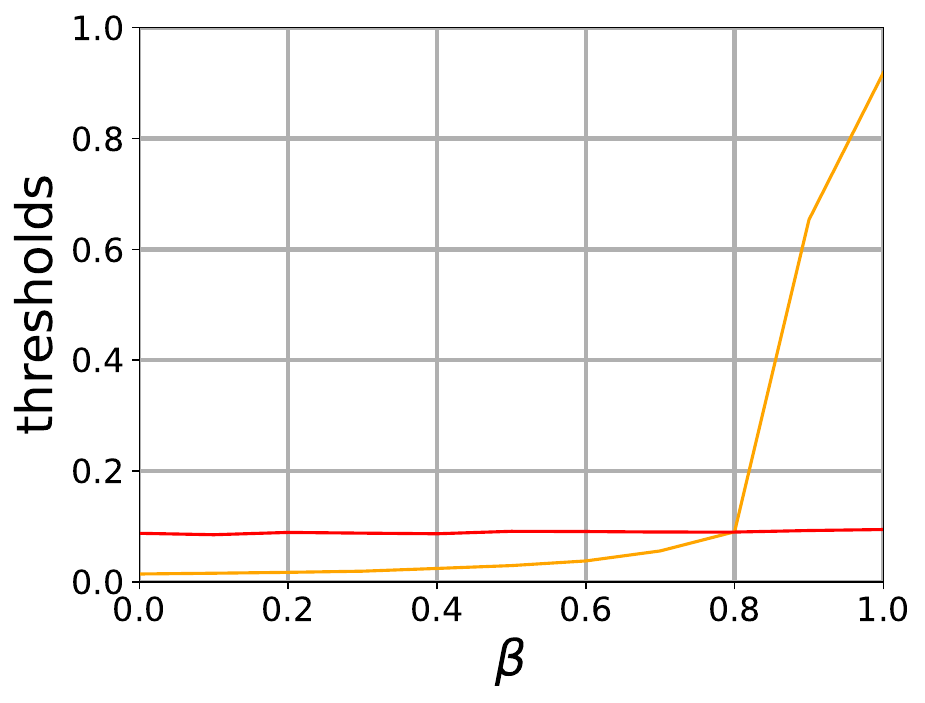}
        \caption{MNIST-C}
    \end{subfigure}
    \hfil
    \begin{subfigure}[b]{0.24\textwidth}
        \includegraphics[width=\linewidth]{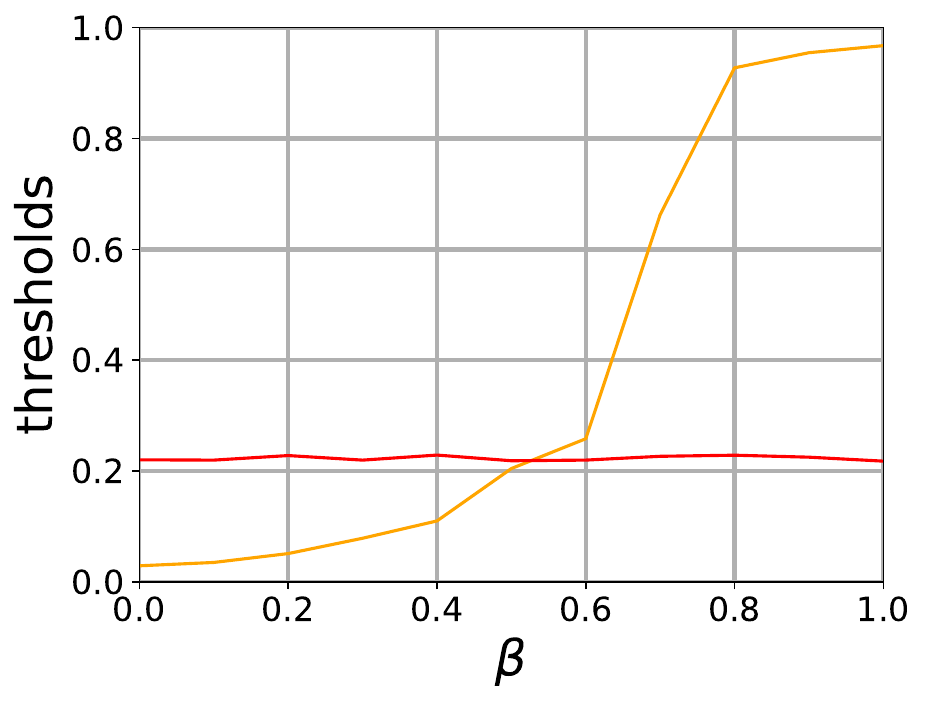}
        \caption{CIFAR-10-C}
    \end{subfigure}
    \hfil
    \begin{subfigure}[b]{0.24\textwidth}
        \includegraphics[width=\linewidth]{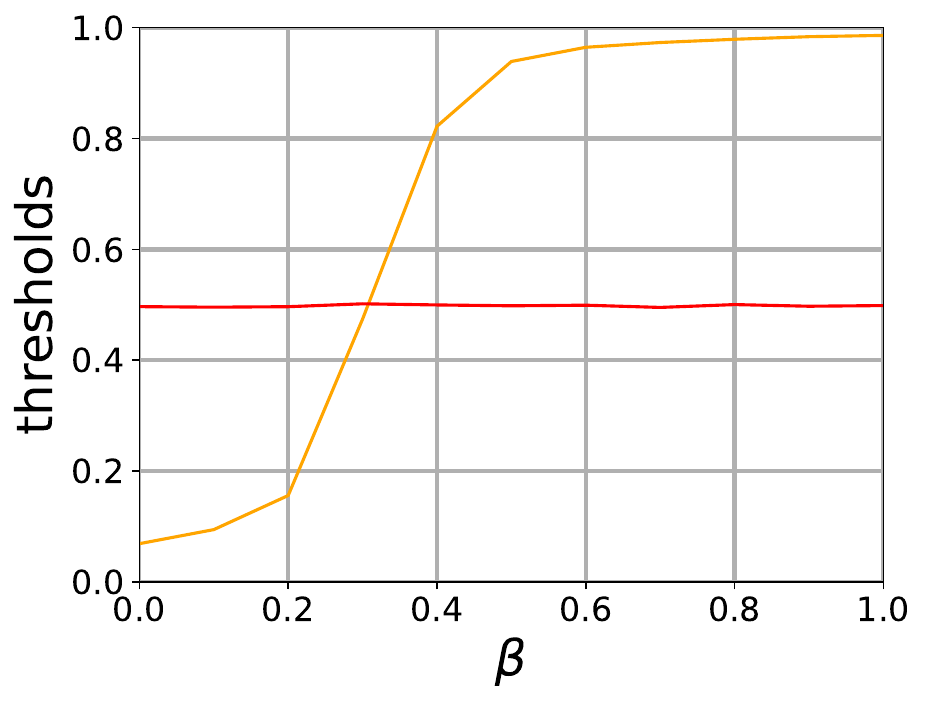}
        \caption{CIFAR-100-C}
    \end{subfigure}
    \caption{These graphs show how the hyperparameter $\beta$ affects the average prediction threshold on a number of real-world data sets. We use a neural network that is trained on one data set but evaluated on a different data set (MNIST to USPS or MNIST-C, CIFAR-10 to CIFAR-10-C, or CIFAR-100 to CIFAR-100-C). We vary $\beta$ from $0$ to $1$ in increments of $0.1$, run \texttt{OCPQ}, and record the resulting average prediction threshold $m$ (over the full data set). The orange line corresponds to the \textbf{\textcolor{orange}{average prediction threshold}} used by \texttt{OCPQ}, and the red line to the average of the \textcolor{red}{\textbf{tightest prediction threshold with coverage}}.} 
    \label{fig:beta_vs_average_threshold}
\end{figure*}

\begin{figure*}[htp]
    \centering
    \begin{subfigure}[b]{0.24\textwidth}
        \includegraphics[width=\linewidth]{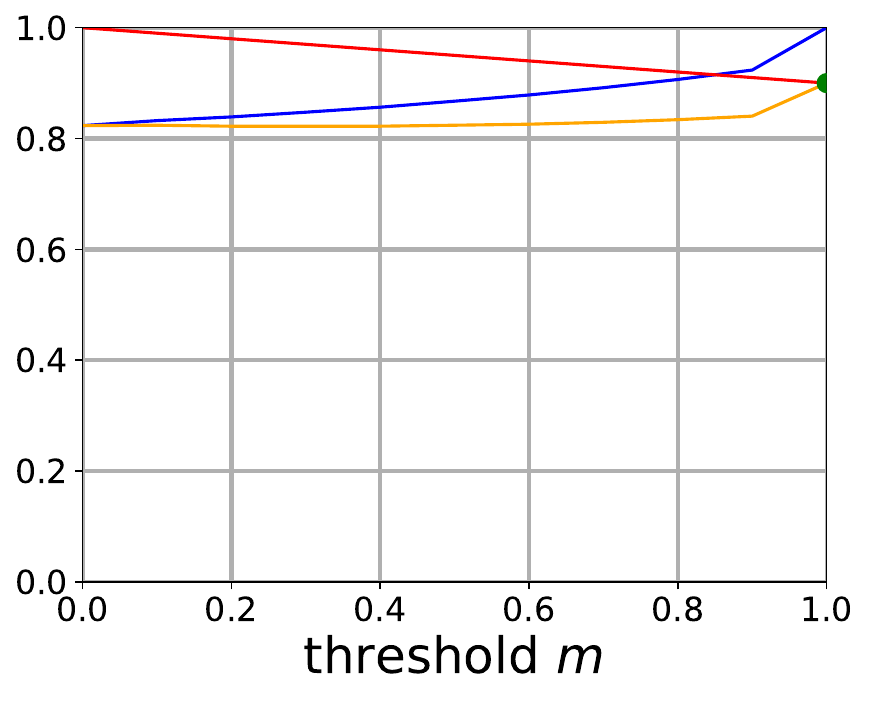}
        \caption{USPS}
    \end{subfigure}
    \hfil
    \begin{subfigure}[b]{0.24\textwidth}
        \includegraphics[width=\linewidth]{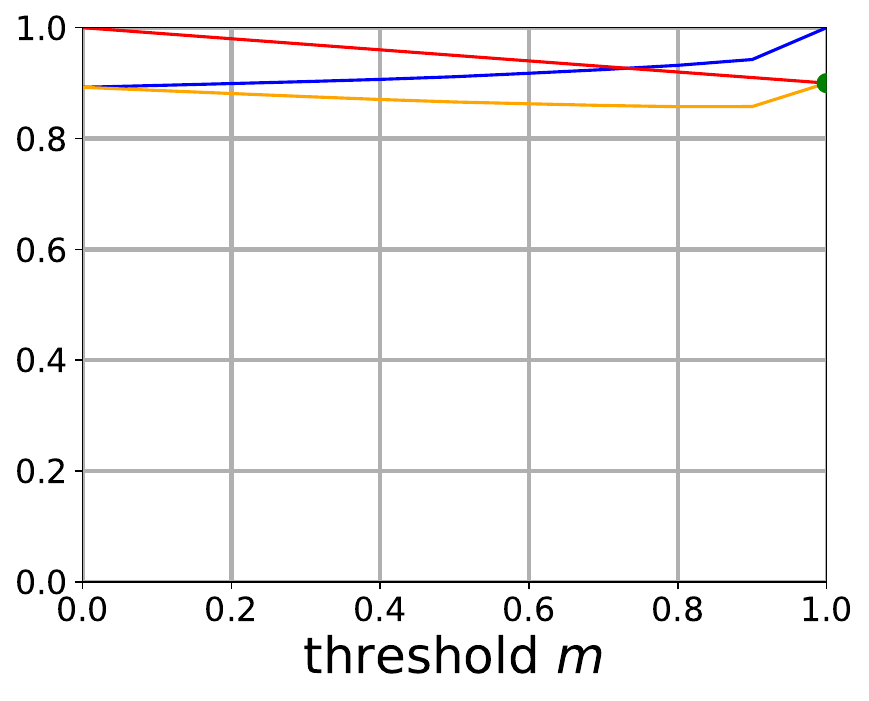}
        \caption{MNIST-C}
    \end{subfigure}
    \hfil
    \begin{subfigure}[b]{0.24\textwidth}
        \includegraphics[width=\linewidth]{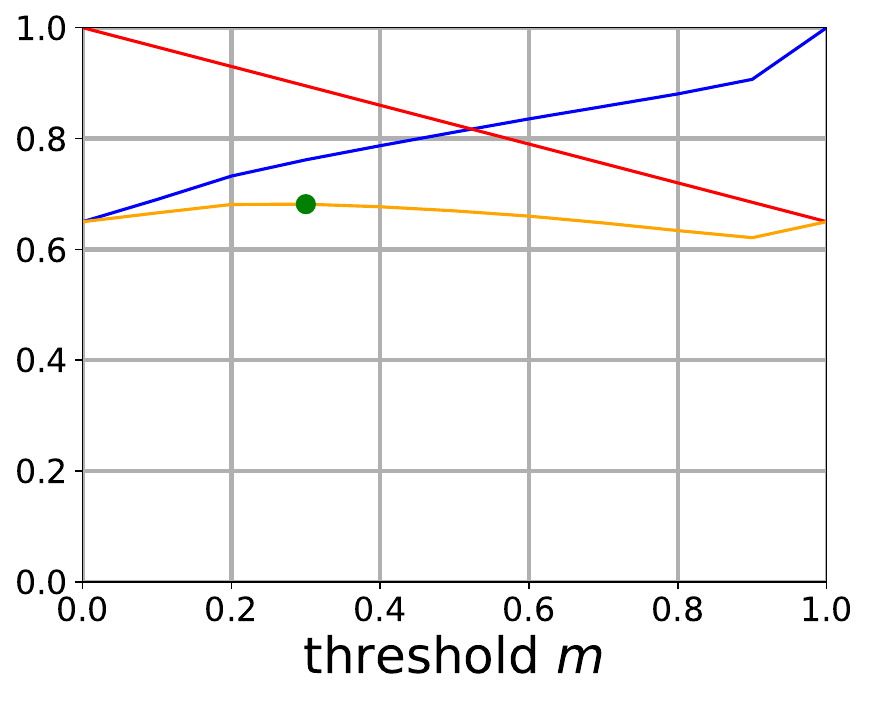}
        \caption{CIFAR-10-C}
    \end{subfigure}
    \hfil
    \begin{subfigure}[b]{0.24\textwidth}
        \includegraphics[width=\linewidth]{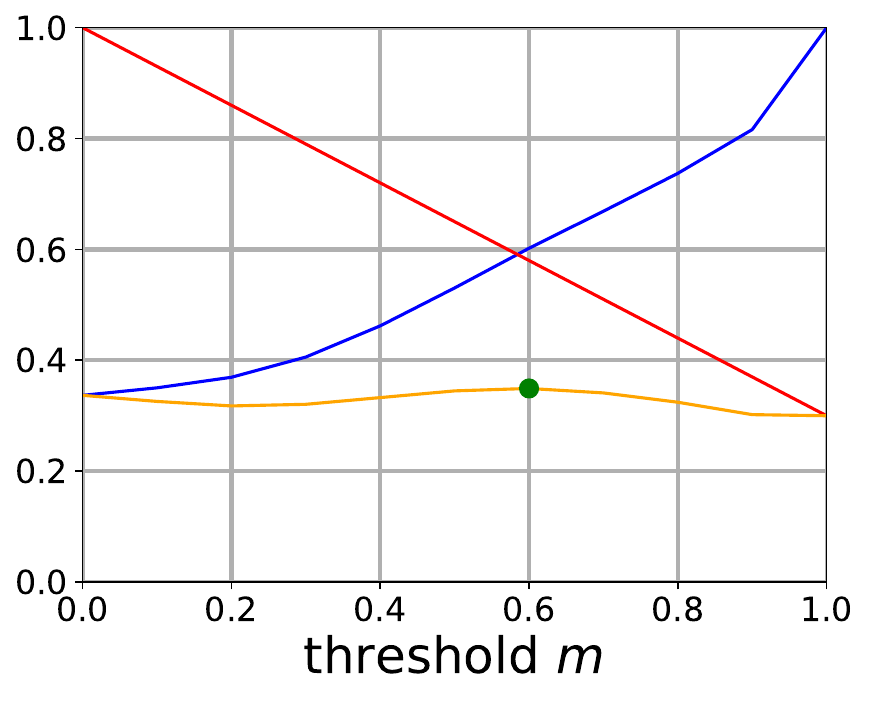}
        \caption{CIFAR-100-C}
    \end{subfigure}
    \caption{These graphs show the accuracy for each prediction threshold on a number of real-world data sets. We use a neural network that is trained on one data set but evaluated on a different data set (MNIST to USPS or MNIST-C, CIFAR-10 to CIFAR-10-C, or CIFAR-100 to CIFAR-100-C). We vary $m$ from 0 to 1 in increments of 0.1, and measure what \textbf{\textcolor{blue}{fraction}} of the time the true label $y_t$ is contained in the resulting prediction set $S_t$ if this value of $m$ is used on every round. For reference, we also plot an example of the auxiliary \textbf{\textcolor{red}{reward}} function using different values of $\beta$ ($0.9$ for MNIST-C and USPS, $0.65$ for CIFAR-10-C, and $0.3$ for CIFAR-100-C), as well as the \textbf{\textcolor{orange}{expected total reward}} of each prediction threshold for this reward function. The \textbf{\textcolor{teal}{green}} dot indicates the prediction threshold with the highest expected reward for the relevant choice of $\beta$, meaning that this is the threshold that \texttt{OCPQ} will converge towards (for the given data set, classifier, and choice of $\beta$).} 
    \label{fig:accuracy_and_reward_per_m}
\end{figure*}

\begin{figure*}[htp]
    \centering
    \begin{subfigure}[b]{0.49\textwidth}
        \includegraphics[width=\linewidth]{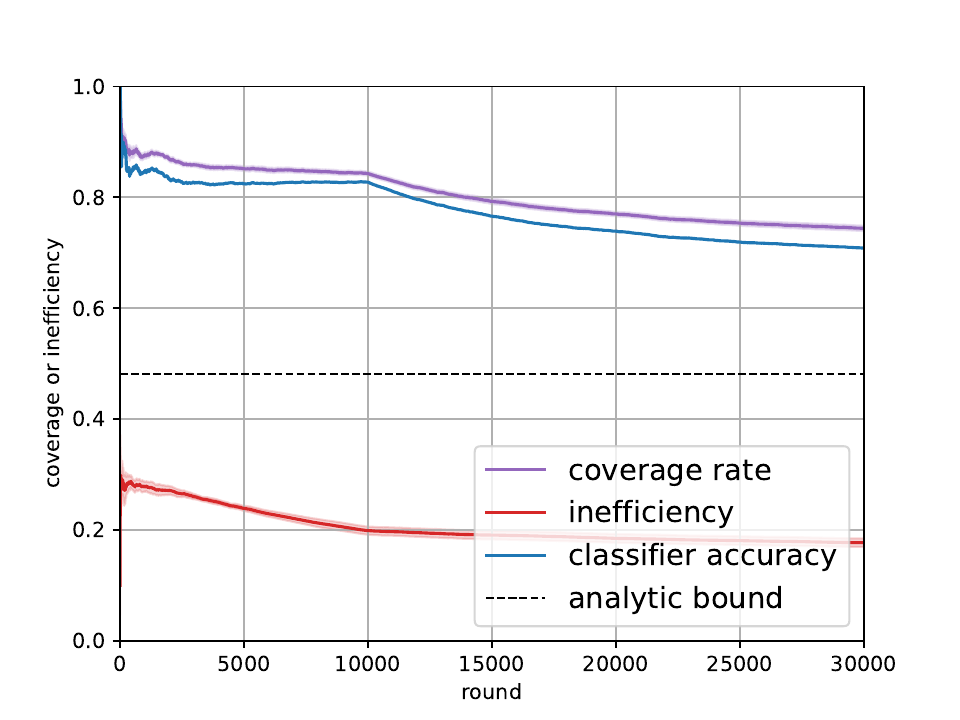}
        \caption{CIFAR-10 to CIFAR-10-C}
    \end{subfigure}
    \hfil
    \begin{subfigure}[b]{0.49\textwidth}
        \includegraphics[width=\linewidth]{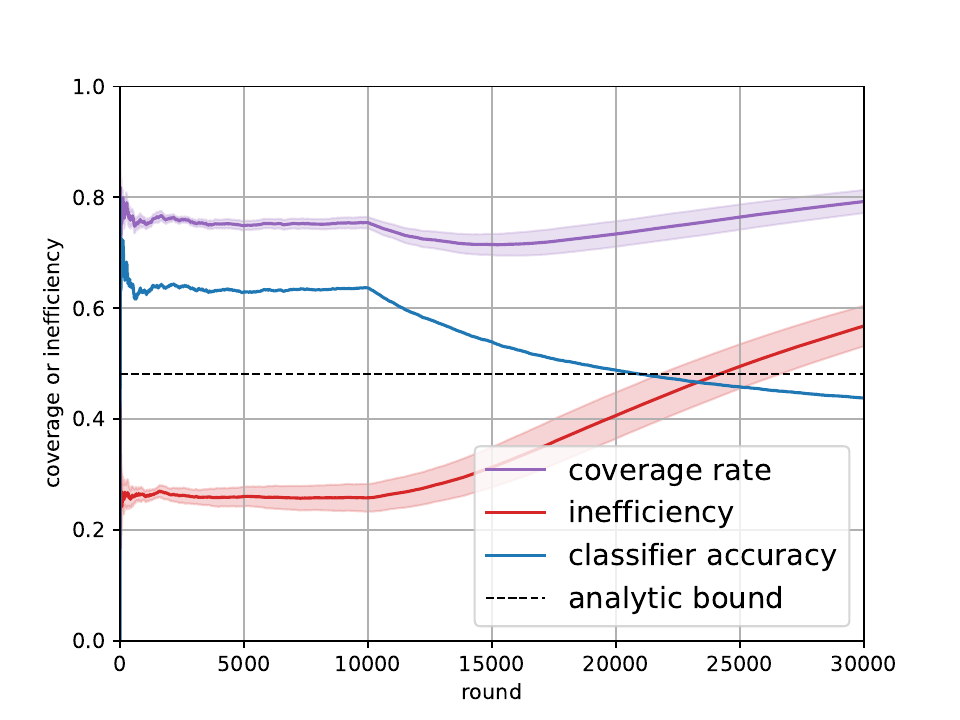}
        \caption{CIFAR-100 to CIFAR-100-C}
    \end{subfigure}
    \caption{These graphs show how the long-run coverage rate and inefficiency evolves over time when there is an abrupt distributional shift in the middle of the data stream. We use a neural network that is trained on either CIFAR-10 or CIFAR-100. The first $10,000$ instances in the data stream consists of the test sets from CIFAR-10 and CIFAR-100, whereas the last $20,000$ instances instead consist of random instances from CIFAR-10-C or CIFAR-100-C. We let $M = \{0.0, 0.2, 0.4, 0.6, 0.8, 1.0\}$, $\beta=0.6$, $\epsilon = T^{-1/3}$, and $\eta = T^{-2/3}\cdot \sqrt{\ln|M|}$, where $T = 30,000$. \texttt{OCPQ} is run 10 times for each data set, with the plot showing the average and the mean absolute deviation for each value over these 10 runs. The purple line shows the \textcolor{violet}{\textbf{coverage rate}}, the red line shows the \textcolor{red}{\textbf{inefficiency}}, the blue line shows the \textbf{\textcolor{blue}{accuracy}} of the underlying classifier, and the dashed black line shows the \textbf{analytic bound} on the (final) cover rate. The inefficiency has been normalised to lie between 0 and 1.} 
    \label{fig:accuracy_and_coverage_with_hard_shift}
\end{figure*}

\begin{figure*}[htp]
    \centering
    \begin{subfigure}[b]{0.3\textwidth}
        \includegraphics[width=\linewidth]{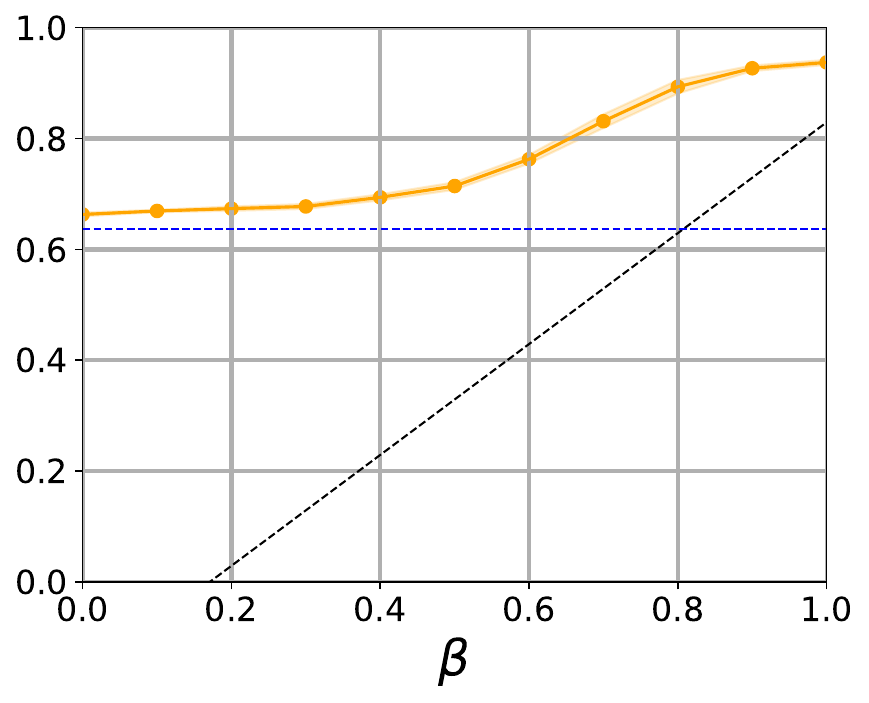}
        \caption{Coverage, linear reward.}
    \end{subfigure}
    \hfil
    \begin{subfigure}[b]{0.3\textwidth}
        \includegraphics[width=\linewidth]{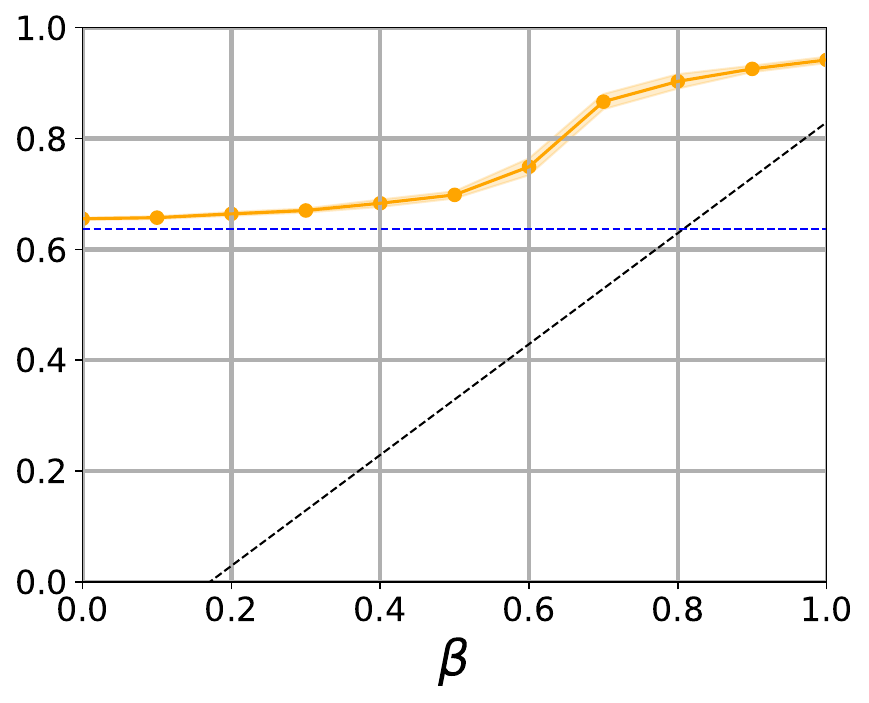}
        \caption{Coverage, quadratic reward.}
    \end{subfigure}
    \hfil
    \begin{subfigure}[b]{0.3\textwidth}
        \includegraphics[width=\linewidth]{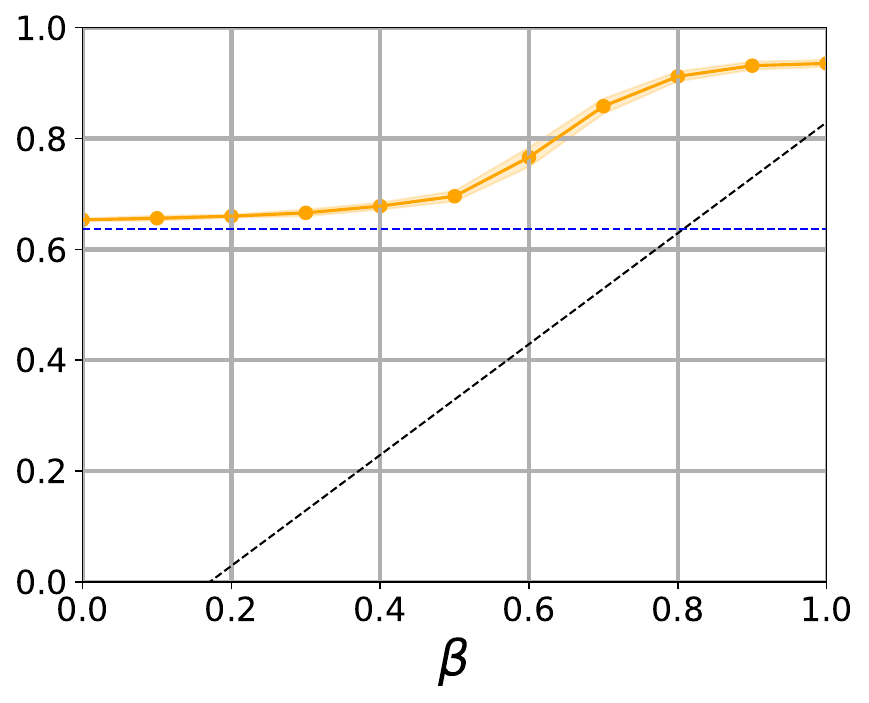}
        \caption{Coverage, quartic reward.}
    \end{subfigure}

    \begin{subfigure}[b]{0.3\textwidth}
        \includegraphics[width=\linewidth]{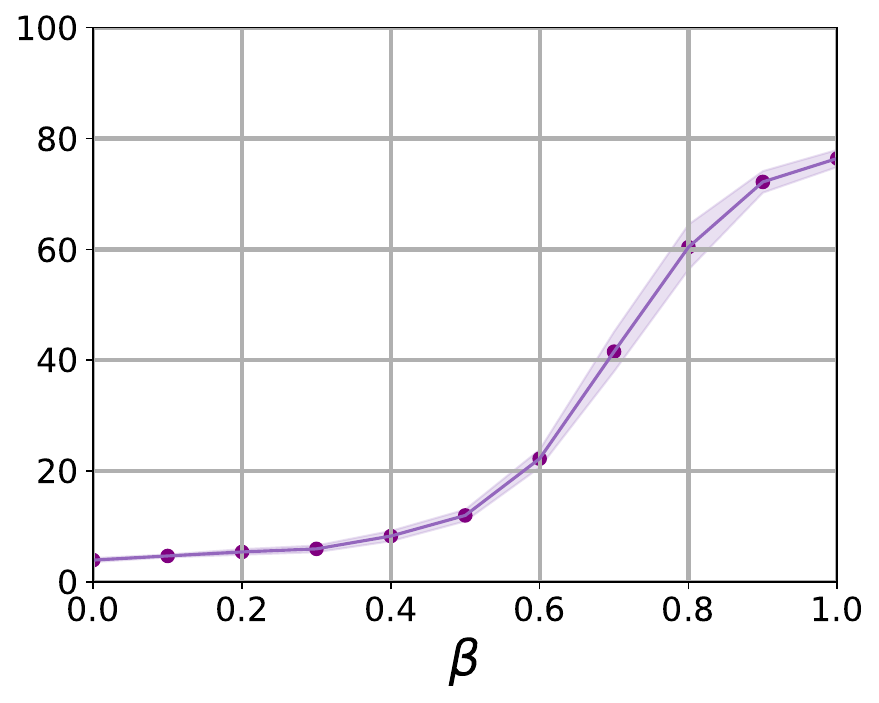}
        \caption{Efficiency, linear reward.}
    \end{subfigure}
    \hfil
    \begin{subfigure}[b]{0.3\textwidth}
        \includegraphics[width=\linewidth]{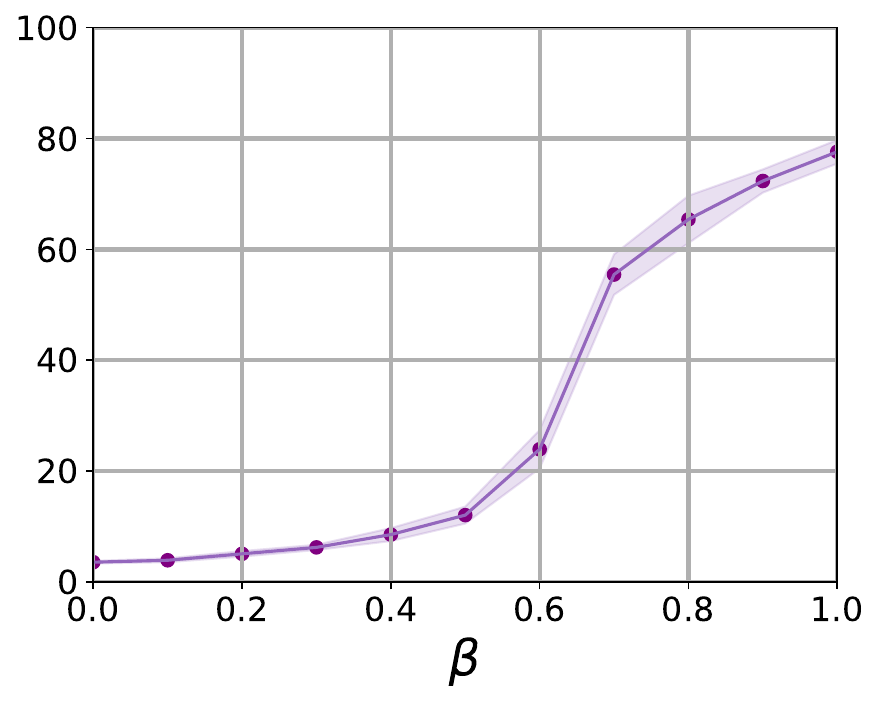}
        \caption{Efficiency, quadratic reward.}
    \end{subfigure}
    \hfil
    \begin{subfigure}[b]{0.3\textwidth}
        \includegraphics[width=\linewidth]{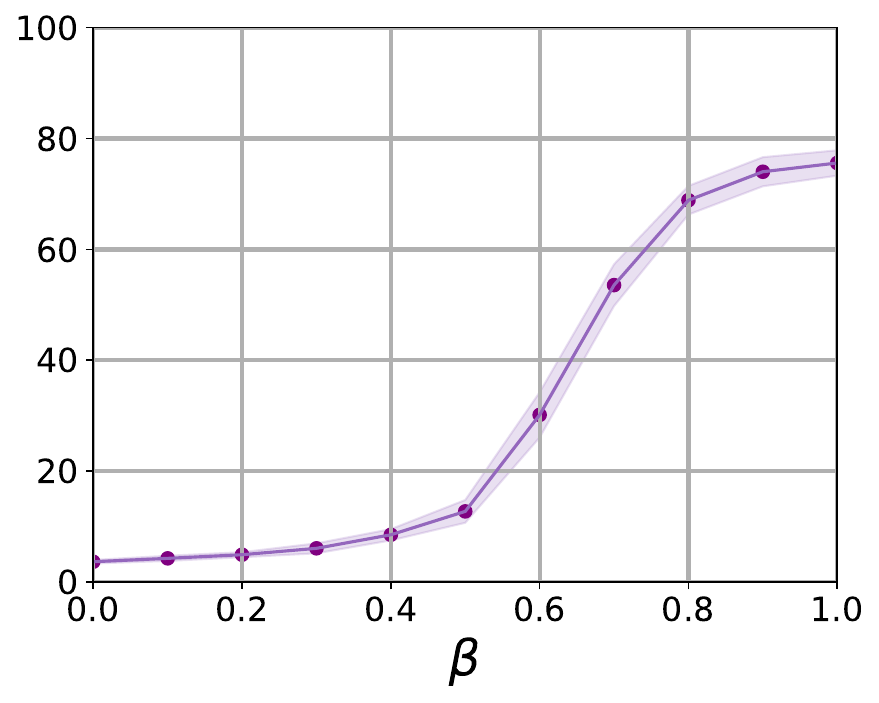}
        \caption{Efficiency, quartic reward.}
    \end{subfigure}

    \begin{subfigure}[b]{0.3\textwidth}
        \includegraphics[width=\linewidth]{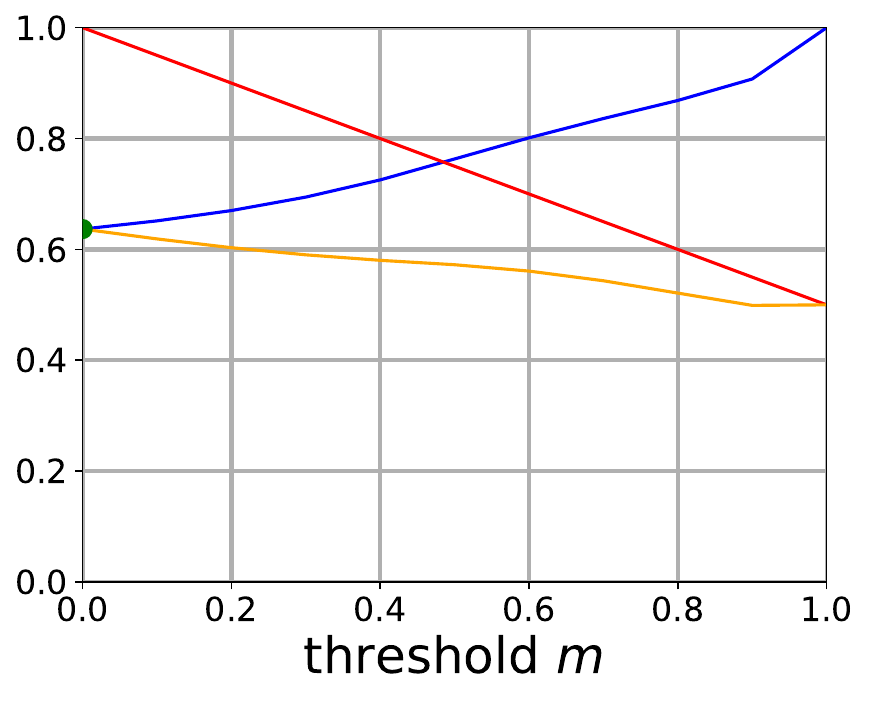}
        \caption{Reward per $m$, linear reward.}
    \end{subfigure}
    \hfil
    \begin{subfigure}[b]{0.3\textwidth}
        \includegraphics[width=\linewidth]{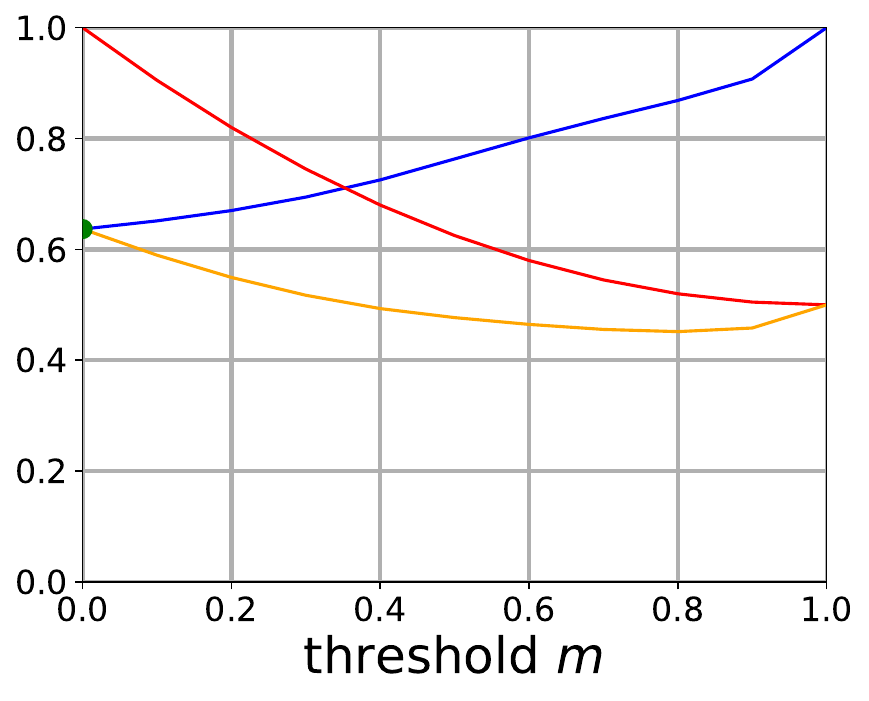}
        \caption{Reward per $m$, quadratic reward.}
    \end{subfigure}
    \hfil
    \begin{subfigure}[b]{0.3\textwidth}
        \includegraphics[width=\linewidth]{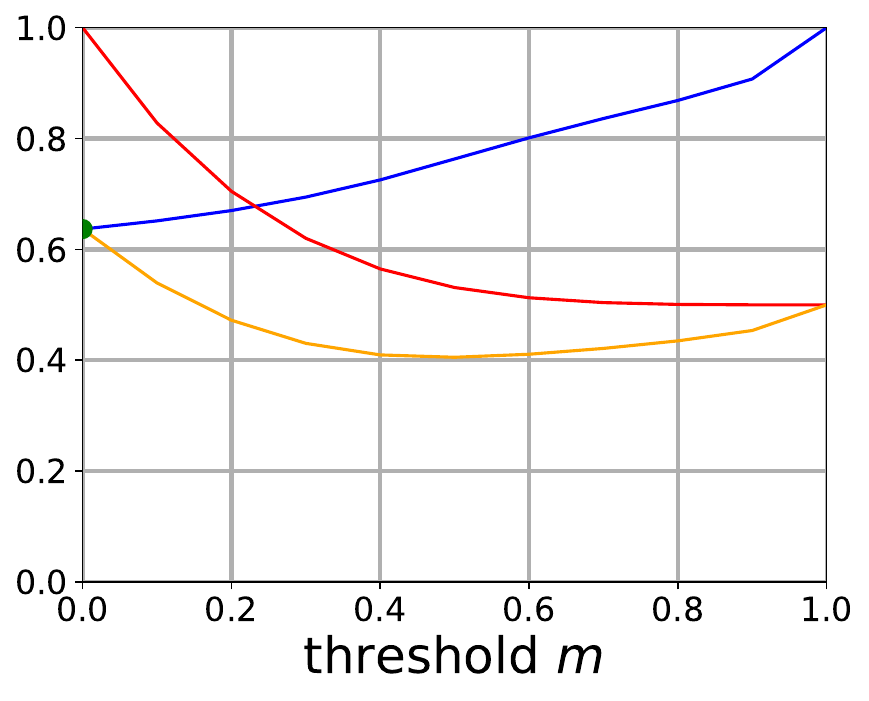}
        \caption{Reward per $m$, quartic reward.}
    \end{subfigure}
    \caption{These graphs show the coverage rate and average prediction set size for \texttt{OCPQ} using three different reward functions. The linear reward function is the reward function described in Definition~\ref{def:auxiliary_reward}. The quadratic reward function is equivalent to the linear reward function except that $R(m,t) = \beta + (1-\beta)\cdot(m-1)^2$ instead of $1-m\cdot(1-\beta)$ when $y_t \in S_t$. Similarly, the quartic reward function is equal to $\beta + (1-\beta)\cdot(m-1)^4$ when $y_t \in S_t$.
    We use a neural network that is trained on the CIFAR-100 training set, and the data stream is the CIFAR-100 test set. We let $M = \{0.0, 0.2, 0.4, 0.6, 0.8, 1.0\}$, $\epsilon = T^{-1/3}$, and $\eta = T^{-2/3} \cdot \sqrt{\ln|M|}$, where $T$ is the size of the data set. We vary $\beta$ from $0$ to $1$ in increments of $0.1$, run \texttt{OCPQ} 10 times for each value of $\beta$, and record the resulting coverage and efficiency. The graphs show the average of these values across 10 runs, with error bars corresponding to the average distance to the average value. For the graphs displaying the coverage rates, the black dashed line corresponds to the lower bound from Corollary~\ref{corollary:coverage_rate_bound}, and the blue line corresponds to the accuracy of the network on the CIFAR-100 test set when no conformal prediction is used. The bottom plots show the  \textbf{\textcolor{blue}{fraction}} of the time that the true label $y_t$ is contained in the resulting prediction set $S_t$ for each value of $m$. For reference, the bottom plots also show the \textbf{\textcolor{red}{reward}} for each value of $m$ conditional on $y_t \in S_t$, as well as the \textbf{\textcolor{orange}{expected reward}} for each value of $m$ over the entire data set, for each of the three reward functions (using $\beta = 0.5$). The \textbf{\textcolor{teal}{green}} dot shows the prediction threshold with the highest expected reward, which is $0$ for all three reward functions when $\beta = 0.5$.}
    \label{fig:accuracy_and_efficiency_with_alternative_reward}
\end{figure*}

\section{WildGuardMix Experimental Details}
\label{app:wildguardmix}

\paragraph{Dataset and task.}
We construct the LLM-safety monitoring task from WildGuardMix
\citep{wildguard2024}. We formulate prompt-risk detection as 14-way
classification, with one benign class and 13 unsafe risk subcategories
covering privacy, misinformation, harmful language, and malicious use.
Examples assigned to the dataset's \texttt{others} category are excluded
because they do not carry a consistent fine-grained risk label.

We operate at the prompt level. Prompts are lowercased, stripped, and
normalized by collapsing consecutive whitespace, after which exact
duplicates are removed. The resulting 37,441 WildGuardTrain prompts are
divided into training, model-selection, and online-monitoring partitions
using a $50\%/10\%/40\%$ split. Splitting uses seed 0 and is stratified jointly
by risk label and the vanilla/adversarial indicator.

\begin{table}[t]
    \centering
    \small
    \begin{tabular}{lrrr}
        \toprule
        Partition & Vanilla & Adversarial & Total \\
        \midrule
        Training       & 10,719 & 7,997 & 18,716 \\
        Model selection & 2,141 & 1,596 &  3,737 \\
        Online stream   & 8,584 & 6,404 & 14,988 \\
        \bottomrule
    \end{tabular}
    \caption{WildGuardTrain partitions after filtering and exact
    prompt-level deduplication. Classifier fitting uses only the vanilla
    portion of the training partition.}
    \label{tab:wildguard-splits}
\end{table}

\paragraph{Base classifier.}
We initialize a 14-way sequence-classification model from
\texttt{distilbert-base-uncased}
\citep{sanh2020distilbertdistilledversionbert} and fine-tune it using only
the 10,719 vanilla prompts in the training partition. Inputs are dynamically
padded within each batch. The complete training configuration is reported in
Table~\ref{tab:wildguard-classifier-config}.

\begin{table}[t]
    \centering
    \small
    \begin{tabular}{ll}
        \toprule
        Hyperparameter & Value \\
        \midrule
        Pretrained encoder & \texttt{distilbert-base-uncased} \\
        Number of classes & 14 \\
        Maximum sequence length & 384 tokens \\
        Training epochs & 3 \\
        Batch size & 128 \\
        Optimizer & AdamW \\
        Learning rate & $5\times10^{-5}$ \\
        Weight decay & $0.01$ \\
        Learning-rate schedule & Linear warm-up and decay \\
        Warm-up fraction & $0.06$ \\
        Gradient clipping & $1.0$ \\
        Random seed & 0 \\
        Checkpoint criterion & Calibration macro-F1 \\
        Selected checkpoint & Final epoch \\
        \bottomrule
    \end{tabular}
    \caption{Training configuration for the WildGuardMix base classifier.}
    \label{tab:wildguard-classifier-config}
\end{table}

The checkpoint with the highest macro-F1 on the complete model-selection
partition is retained; this is the checkpoint from the final epoch.

We separately evaluate the classifier on the provided WGTest split after
applying the same label filtering and exact-deduplication procedure. The
resulting test set contains 834 vanilla and 793 adversarial prompts. The
classifier obtains overall accuracy $0.6908$ and macro-F1 $0.3915$. Its
accuracy is $0.7854$ on the vanilla subset and $0.5914$ on the adversarial
subset.

\paragraph{Online monitoring protocol.}
The online experiments use the distinct monitoring partition in
Table~\ref{tab:wildguard-splits}: 8,584 vanilla prompts and 6,404 adversarial
prompts. Vanilla and adversarial prompts are evaluated as separate static
streams. Each run traverses every prompt in the relevant partition exactly
once, without replacement, in a seeded random order. We report means over ten
independently shuffled trajectories, with mean absolute deviations as the
uncertainty measure.

The classifier produces a softmax probability vector $p_t$. For threshold
$q$, its nested prediction set is
\[
    \Gamma_C(x_t, q)
    =
    \left\{k:\max_j p_{t,j}-p_{t,k}\leq q\right\}.
\]
For the principal experiment, we use
\[
    M=\{0,0.2,0.4,0.6,0.8,1\}
\]
and sweep $\beta\in\{0,0.1,\ldots,1\}$. At horizon $T$, OCPQ uses
\[
    \epsilon=T^{-1/3},
    \qquad
    \eta=T^{-2/3}\sqrt{\ln |M|}.
\]
The corresponding expected query rates are $0.04884$ for the vanilla stream
and $0.05385$ for the adversarial stream.

For the threshold-grid sensitivity experiment, we fix $\beta=0.75$ and use
uniformly spaced grids with
\[
    |M|\in\{8,16,32,64,128,256,512,1024,2048\}.
\]
We report overall coverage, under which query rounds count as uncovered.
Average prediction-set size is calculated only over rounds on which OCPQ
outputs a prediction set.

\section{Proofs}\label{appendix:proofs}

\setcounter{proposition}{0}
\setcounter{theorem}{0}
\setcounter{lemma}{0}
\setcounter{corollary}{0}

In this appendix, we provide the proofs of our theorems. The first subsection focuses on the theorems that apply to \texttt{OCPQ} with a fixed value of $\beta$, and the second subsection focuses on the case where $\beta$ is adapted over time.

\subsection{OCPQ With Fixed $\beta$}

Before we provide our proofs, we will first introduce a few new definitions and a few lemmas. First of all, in this section, we will generally consider the \enquote{actions} of the algorithm to be taken from $M \cup \{\texttt{query}\}$, and use $a_t$ to denote the action on round $t$. This is in slight contrast to our original problem formulation, where we define the action space $\mathcal{A}$ as $\mathcal{P}(Y) \cup \{\texttt{query}\}$. Since our notion of regret is defined relative to the reward function of Definition~\ref{def:auxiliary_reward}, and since this reward function is defined over $M \cup \{\texttt{query}\}$ rather than $\mathcal{P}(Y) \cup \{\texttt{query}\}$, it will be easier to treat the actions as elements of $M \cup \{\texttt{query}\}$. This is just a terminological detail which will make the proofs easier to read, and should regardless always be clear from the context. 

Moreover, for each $m \in M$ and $t \leq T$, we define a random variable $\hat{X}_{m,t}$ that is equal to $1$ if $a_t \neq \texttt{query}$, and otherwise, $\hat{X}_{m,t} = 1 - \frac{1-R(m, t)}{\epsilon}$. Note that this means that $\hat{Z}_{m, t+1} = \hat{Z}_{m,t} + \hat{X}_{m,t}$. Moreover, $\hat{X}_{m,t}$ acts as an estimator for $R(m,t)$. Its slightly convoluted definition is required to ensure that $\hat{X}_{m,t} \leq 1$ with probability 1, which is important in one step of the proof of the regret bound. Before we give the main proof, we will first show that $\hat{X}_{m,t}$ and $\hat{Z}_{m,t}$ are unbiased estimators of $R(m, t)$ and $\sum_{j=1}^{t-1} R(m, j)$:
\begin{lemma}\label{lemma:the_estimators_are_unbiased}
We have that $\mathbb{E}[\hat{X}_{m,t}] = R(m, t)$ and $\mathbb{E}[\hat{Z}_{m,t}] = \sum_{j=1}^{t-1} R(m, j)$, where the expectation is over the randomness in the learning algorithm.
\end{lemma}
\begin{proof}
The value of $\hat{X}_{m,t}$ depends on whether or not the algorithm makes a query at time $t$. We can thus calculate its expectation as follows:
\begin{align*}
\mathbb{E}[\hat{X}_{m,t}] = &\mathbb{P}(a_t \neq \texttt{query}) \cdot 1 
+ \mathbb{P}(a_t = \texttt{query})\cdot \left(1 - \frac{1-R(m,t)}{\epsilon}\right)\\
= &(1-\epsilon) \cdot 1 + \epsilon \cdot \left(1 - \frac{1-R(m,t)}{\epsilon}\right)\\
= &1 - \epsilon + \epsilon - (1 - R(m, t))\\
= &R(m, t)
\end{align*}
We thus have that $\mathbb{E}[\hat{X}_{m,t}] = R(m, t)$. From this, and the linearity of expectation, it also follows that $\mathbb{E}[\hat{Z}_{m,t}] = \sum_{j=1}^{t-1} R(m, j)$.
\end{proof}

We will also type out the following calculation, which will be used at one step in the proof (but which is otherwise not very interesting, so the main proof can be kept more readable by listing this calculation separately):

\begin{lemma}\label{lemma:simple_calculus_lemma}
Let $\epsilon \in (0, 0.5]$ and 
$f(r) = \epsilon\left(1-\frac{1-r}{\epsilon}\right)^2$. Then $f(r) \leq \frac{1}{\epsilon} + \epsilon - 2$ for $r \in [0,1]$.
\end{lemma}
\begin{proof}
Let us first expand the terms in $f$:
\begin{align*}
f(r) &= \epsilon\left(1-\frac{1-r}{\epsilon}\right)^2\\
&= \epsilon \left(1 - 2 \left(\frac{1-r}{\epsilon}\right) + \frac{(1-r)^2}{\epsilon^2}\right)\\
&= \epsilon - 2 + 2r + \frac{1 - 2r + r^2}{\epsilon}\\
&= \epsilon - 2 + 2r + \frac{1}{\epsilon} - \frac{2}{\epsilon}r + \frac{1}{\epsilon}r^2\\
&= \left(\frac{1}{\epsilon}\right)r^2 + \left(2-\frac{2}{\epsilon}\right) r + \left(\frac{1}{\epsilon} + \epsilon - 2\right).
\end{align*}
From this, we have that $f'(r) = \left(\frac{2}{\epsilon}\right)r + \left(2-\frac{2}{\epsilon}\right)$. Thus, if $f'(r) = 0$ then $r = 1 - \epsilon$. We can also straightforwardly determine that
\begin{enumerate}
    \item $f(0) = \frac{1}{\epsilon} + \epsilon - 2$
    \item $f(1) = \epsilon$ 
    \item $f(1-\epsilon) = 0$
\end{enumerate}
If $\frac{1}{\epsilon} + \epsilon - 2 = \epsilon$ then $\epsilon = 0.5$, and for $\epsilon \in (0,0.5)$ we have that $\frac{1}{\epsilon} + \epsilon - 2 > \epsilon$. Since $f$ is convex, its maximum on $[0,1]$ is attained at an endpoint, and hence $f(r) \leq \frac{1}{\epsilon} + \epsilon - 2$ for $r \in [0,1]$.
\end{proof}

We can now prove the main regret bound:

\begin{theorem}
    For any data stream $\xi$ and classifier $C$, and any $\eta \in (0,1]$ and $\epsilon \in (0,0.5]$, the expected  regret of Algorithm~\ref{algorithm:OCP_with_queries} with the auxiliary reward is at most
    $$
    \frac{\ln(|M|)}{\eta} + \epsilon \cdot T + \left(\frac{\eta}{\epsilon}\right) \cdot T.
    $$
    If we set $\epsilon = T^{-1/3}$ and $\eta = T^{-2/3} \cdot \sqrt{\ln(|M|)}$ then this simplifies to
    $$
    T^{2/3} \cdot (2 \sqrt{\ln(|M|)} + 1).
    $$
\end{theorem}
\begin{proof}
This proof is similar to a standard regret bound proof for \texttt{EXP3} or \texttt{Hedge}, but with a few minor modifications. First, let $w_{m,t} = \exp \left(\eta \cdot \hat{Z}_{m,t}\right)$ and $W_t = \sum_{m \in M} w_{m,t}$. This means that if the algorithm is \emph{not} making a query at time $t$, then the probability that it will pick a given prediction threshold $m$ is equal to $w_{m,t}/W_t$. Also recall that $\hat{Z}_{m,t}$ is an unbiased estimate of $\sum_{j=1}^{t-1} R(m,j)$ (see Lemma~\ref{lemma:the_estimators_are_unbiased}). In particular, $\hat{Z}_{m,T+1}$ is an unbiased estimate of $\sum_{t=1}^{T} R(m,t)$. The basic proof strategy relies on establishing that, roughly speaking,
\begin{enumerate}
    \item if some threshold $m$ gets a lot of reward, then $\hat{Z}_{m,T+1}$, and hence $w_{m,T+1}$, and hence $W_{T+1}$ will be large in expectation, and
    \item if the learner does not get a lot of reward in expectation, then $W_{T+1}$ will be small in expectation.
\end{enumerate}
Together, this can be used to bound the expected reward of the learning algorithm in terms of the total reward of the best individual prediction threshold.

In this proof, it will be convenient to distinguish between the action probabilities that are suggested by the exponential weights $\{w_{i,t}\}_{i \in M}$ and the action probabilities that the agent is actually using. Therefore, let $P_t$ be the action probabilities of the exponential weights, so that
$$
\mathbb{P}_{A \sim P_t}(A = m) = \frac{w_{m,t}}{W_t}
$$
for $m \in M$, and $\mathbb{P}_{A \sim P_t}(A = \texttt{query}) = 0$. Note that, from the definition of the weights $\{w_{m,t}\}_{m \in M}$, this is equivalent to applying a softmax function with temperature $\eta$ to $\{\hat{Z}_{m,t}\}_{m \in M}$. Similarly, let $Q_t$ be the action probabilities that Algorithm~\ref{algorithm:OCP_with_queries} will use at time $t$, so that $\mathbb{P}_{A \sim Q_t}(A = \texttt{query}) = \epsilon$, and
$$
\mathbb{P}_{A \sim Q_t}(A = m) = (1-\epsilon) \cdot \left(\frac{w_{m,t}}{W_t}\right)
$$
for $m \in M$. In the proof, we will first bound the regret in terms of the action probabilities $P_t$, but pretending that the weights $\{w_{m,t}\}_{m \in M}$ are distributed based on actions drawn from $Q_t$. We will then show that the regret of actions drawn from $Q_t$ is not much larger than this quantity.

We now provide the proof. First note that for any $m \in M$,
\begin{align*}
    \exp \left(\eta \cdot \hat{Z}_{m,T+1}\right) &= w_{m,T+1}\\
    &\leq W_{T+1}
\end{align*}
since $W_{T+1} = \sum_{m \in M} w_{m, T+1}$. Also note that
\begin{align*}
    W_{t+1} &= \sum_{m \in M} w_{m, t+1}\\
    &= \sum_{m \in M} \exp \eta \hat{Z}_{m,t+1}\\
    &= \sum_{m \in M} \exp \eta (\hat{Z}_{m,t} + \hat{X}_{m,t})\\
    &= \sum_{m \in M} \exp \eta \hat{Z}_{m,t} \cdot \exp \eta\hat{X}_{m,t}\\
    &= \sum_{m \in M} w_{m,t} \cdot \exp{\eta \hat{X}_{m,t}}\\
    &= \sum_{m \in M} W_t \cdot \frac{w_{m,t}}{W_t} \cdot \exp{\eta \hat{X}_{m,t}}\\
    &= W_t \cdot \sum_{m \in M} \mathbb{P}_{A_t \sim P_t}(A_t = m) \cdot \exp{\eta \hat{X}_{m,t}}\\
    &=  W_t \cdot \mathbb{E}_{A_t \sim P_t}\left[\exp{\eta \hat{X}_{A_t,t}}\right].
\end{align*}
We can chain this statement to obtain that
\begin{align*}
W_{T+1} &= W_1 \cdot \Pi_{t=1}^T \mathbb{E}_{A_t \sim P_t}\left[\exp{\eta \hat{X}_{A_t,t}}\right]\\
&= |M| \cdot \Pi_{t=1}^T \mathbb{E}_{A_t \sim P_t}\left[\exp{\eta \hat{X}_{A_t,t}}\right]
\end{align*}
noting that $W_1 = \sum_{m \in M} \exp \eta \hat{Z}_{m,1} = \sum_{m \in M} \exp 0 = |M|$. Putting these two things together, we have thus established that
$$
\exp{\eta \hat{Z}_{m,T+1}} \leq |M| \cdot \Pi_{t=1}^T \mathbb{E}_{A_t \sim P_t}\left[\exp{\eta \hat{X}_{A_t,t}}\right]
$$
Taking the natural logarithm of both sides and dividing by $\eta$ gives
$$
\hat{Z}_{m,T+1} \leq \frac{\ln |M|}{\eta} + \left(\frac{1}{\eta}\right)\sum_{t=1}^T \ln \mathbb{E}_{A_t \sim P_t}\left[\exp{\eta \hat{X}_{A_t,t}}\right].
$$
Note that $\hat{Z}_{m,T+1}$ is an unbiased estimate of the total reward of threshold $m$ over all $T$ rounds, and the rightmost term looks quite close to being an estimate for the expected reward of Algorithm~\ref{algorithm:OCP_with_queries}. However, intuitively, the $\exp$ and $\ln$ functions are getting in the way, so we need to find a way to unpack this expression and make it more linear.

Note that $e^x \leq 1+x+x^2$ when $x \leq 1$, and that $\eta \hat{X}_{A_t,t} \leq 1$ with probability 1 (this is why it was important to define the estimators $\hat{X}_{i,t}$ such that they are bounded by 1). Therefore:
\begin{align*}
\sum_{t=1}^T \ln \mathbb{E}_{A_t \sim P_t}\left[\exp{\eta \hat{X}_{A_t,t}}\right]
\leq &\sum_{t=1}^T \ln \mathbb{E}_{A_t \sim P_t}\left[1 + \eta\hat{X}_{A_t,t} + \eta^2\hat{X}_{A_t,t}^2\right]\\
= &\sum_{t=1}^T \ln \left(1 + \mathbb{E}_{A_t \sim P_t}\left[\eta\hat{X}_{A_t,t}\right]+ \mathbb{E}_{A_t \sim P_t}\left[\eta^2\hat{X}_{A_t,t}^2\right]\right).
\end{align*}
Note that $1+x \leq e^x$ for all $x \in \mathbb{R}$.
We can thus derive that
\begin{align*}
\leq &\sum_{t=1}^T \ln \exp\left(\mathbb{E}_{A_t \sim P_t}\left[\eta\hat{X}_{A_t,t}\right] +  \mathbb{E}_{A_t \sim P_t}\left[\eta^2\hat{X}_{A_t,t}^2\right]\right)\\
= &\sum_{t=1}^T \left(\mathbb{E}_{A_t \sim P_t}\left[\eta\hat{X}_{A_t,t}\right] +  \mathbb{E}_{A_t \sim P_t}\left[\eta^2\hat{X}_{A_t,t}^2\right]\right)\\
= &\sum_{t=1}^T \eta \cdot \mathbb{E}_{A_t \sim P_t}\left[\hat{X}_{A_t,t}\right] +  \eta^2 \cdot \sum_{t=1}^T\mathbb{E}_{A_t \sim P_t}\left[\hat{X}_{A_t,t}^2\right].
\end{align*}
Putting this back into the previous inequality, and rearranging, gives us that
\begin{align*}
\hat{Z}_{m,T+1} - \sum_{t=1}^T\mathbb{E}_{A_t \sim P_t}\left[\hat{X}_{A_t,t}\right]
\leq \frac{\ln |M|}{\eta} + \eta\sum_{t=1}^T\mathbb{E}_{A_t \sim P_t}\left[\hat{X}_{A_t,t}^2\right].
\end{align*}
At this point, it is worth reflecting on the fact that we have not yet used any probabilities or randomness to arrive at this statement. More specifically, the statement above holds for any arbitrary collection of numbers $\{\hat{X}_{m,t}\}_{m \in M, t \in \{1 \dots T\}}$, where $\hat{Z}_{m,T+1}$ and $P_t$ are derived from these numbers in the way we specified. In fact, the only thing that we have actually assumed about the numbers $\{\hat{X}_{m,t}\}_{m \in M, t \in \{1 \dots T\}}$ is that they are bounded above by 1. To proceed, we will now treat the numbers $\{\hat{X}_{m,t}\}$ as \emph{random variables} that are generated during a run of Algorithm~\ref{algorithm:OCP_with_queries}, and take the expectation of both sides of the inequality with respect to these random variables. 
For clarity, we will also replace $\mathbb{E}_{A_t \sim P_t}$ with $\sum_{m \in M} P_{m,t}$ to emphasise the fact that the values of $P_{m,t}$ are also random variables dependent on $\{\hat{X}_{m,t}\}$, and that the overall expectation of the expression is with respect to \emph{these} probabilities (not the probabilities $P_t$). Note that $P_{m,t} = w_{m,t}/W_t$.
The left-hand side of the inequality then becomes
\begin{align*}
    \mathbb{E} \left[\hat{Z}_{m,T+1} - \sum_{t=1}^T\sum_{m' \in M} P_{m',t} \cdot \hat{X}_{m',t}\right]
    = \hspace{0.3cm}&\mathbb{E} \Big[\hat{Z}_{m,T+1}\Big] - \sum_{t=1}^T\sum_{m' \in M}\mathbb{E}\left[P_{m',t} \cdot \hat{X}_{m',t}\right]\\
    = \hspace{0.3cm}&\sum_{t=1}^T R(m,t) - \sum_{t=1}^T\sum_{m' \in M}\mathbb{E}\Big[P_{m',t}\Big] \cdot \mathbb{E}\Big[\hat{X}_{m',t}\Big]\\
    = \hspace{0.3cm}&\sum_{t=1}^T R(m,t) - \sum_{t=1}^T\sum_{m' \in M} \mathbb{E}\Big[P_{m',t}\Big] \cdot R(m', t)\\
    = \hspace{0.3cm}&\sum_{t=1}^T R(m,t) - \sum_{t=1}^T\mathbb{E}_{A_t \sim P_t}\left[R(A_t, t)\right]
\end{align*}
We are here using Lemma~\ref{lemma:the_estimators_are_unbiased}. 
Also note that $P_{m',t}$ and $\hat{X}_{m',t}$ are independent, since the distribution of $\hat{X}_{m',t}$ only depends on $\epsilon$, $x_t$, $y_t$, and $C$, none of which are random variables. This means that $\mathbb{E}\left[P_{m',t} \cdot \hat{X}_{m',t}\right] = \mathbb{E}\Big[P_{m',t}\Big] \cdot \mathbb{E}\Big[\hat{X}_{m',t}\Big]$.

Note that this expression is the expected regret with respect to the probability distributions $P_t$, but where the estimators $\hat{X}_{m,t}$ are coming from a run of Algorithm~\ref{algorithm:OCP_with_queries}, which is not sampling its actions from $P_t$. What we are actually interested in is the expected regret with respect to the probability distributions $Q_t$, since this is what Algorithm~\ref{algorithm:OCP_with_queries} actually uses. However, we can extract this regret with a simple rewrite:
\begin{align*}
    \sum_{t=1}^T R(m,t) - \sum_{t=1}^T\mathbb{E}_{A_t \sim P_t}\left[R(A_t, t)\right]
    = \hspace{0.3cm}&\sum_{t=1}^T R(m,t) - \sum_{t=1}^T\mathbb{E}_{A_t \sim Q_t}\left[R(A_t, t)\right]- \epsilon \cdot \sum_{t=1}^T\mathbb{E}_{A_t \sim P_t}\left[R(A_t, t)\right]\\
    \geq \hspace{0.3cm}&\sum_{t=1}^T R(m,t) - \sum_{t=1}^T\mathbb{E}_{A_t \sim Q_t}\left[R(A_t, t)\right] - \epsilon\cdot T
\end{align*}
We are here using the fact that
$$
\mathbb{E}_{A_t \sim Q_t}[R(A_t,t)] = (1-\epsilon) \cdot \mathbb{E}_{A_t \sim P_t}[R(A_t,t)],
$$
which we can of course rewrite as
\begin{align*}
\mathbb{E}_{A_t \sim P_t}[R(A_t,t)]
= \mathbb{E}_{A_t \sim Q_t}[R(A_t,t)] + \epsilon \cdot \mathbb{E}_{A_t \sim P_t}[R(A_t,t)].
\end{align*}
For the last step we are also using the fact that $\epsilon \cdot \sum_{t=1}^T\mathbb{E}_{A_t \sim P_t}[R(A_t,t)] \leq \epsilon \cdot T$, which is a straightforward consequence of the fact that $R(m,t) \in [0,1]$.

Let us now turn to the right-hand side of the inequality. This side of the inequality is equal to
\begin{align*}
     \mathbb{E}\left[\frac{\ln |M|}{\eta} + \eta\sum_{t=1}^T \sum_{m' \in M} P_{m',t} \cdot \hat{X}_{m',t}^2\right]
     = \frac{\ln |M|}{\eta} + \eta\sum_{t=1}^T \mathbb{E}\left[ \sum_{m' \in M} P_{m',t} \cdot \hat{X}_{m',t}^2\right]
\end{align*}
where the expectation is with respect to the parameters $\{\hat{X}_{m',t}\}_{m' \in M,t \in \{1 \dots T\}}$.
Let $\mathcal{A}^\star = M \cup \{\texttt{query}\}$. We can expand the expression inside the expectation:
\begin{align*}
    \mathbb{E}\left[ \sum_{m' \in M} P_{m',t} \cdot \hat{X}_{m',t}^2\right]
    = &\hspace{0.3cm}\sum_{j \in \mathcal{A}^\star} \mathbb{P}_{A_t \sim Q_t}(A_t = j)\sum_{m' \in M} P_{m',t} \cdot \mathbb{E}\left[\hat{X}_{m',j}^2 \hspace{0.1cm}\Big|\hspace{0.1cm} A_t = j\right]\\
    = &\hspace{0.3cm}\mathbb{P}_{A_t \sim Q_t}(A_t \neq \texttt{query}) \sum_{m' \in M} P_{m',t} \cdot 1^2\\
    &\hspace{0.3cm}+\mathbb{P}_{A_t \sim Q_t}(A_t = \texttt{query}) \sum_{m' \in M} P_{m',t} \left(1 - \frac{1-R(m',t)}{\epsilon}\right)^2\\
    = &\hspace{0.3cm} (1-\epsilon) + \epsilon \cdot \sum_{m' \in M} P_{m',t} \cdot \left(1 - \frac{1-R(m',t)}{\epsilon}\right)^2\\
    \leq &\hspace{0.3cm} 1 - \epsilon +  \sum_{m' \in M} P_{m',t} \left(\frac{1}{\epsilon} + \epsilon - 2\right)\\
    = &\hspace{0.3cm} 1 - \epsilon +  \left(\frac{1}{\epsilon} + \epsilon - 2\right) \sum_{m' \in M} P_{m',t}  = \frac{1}{\epsilon} - 1 \leq \frac{1}{\epsilon}
\end{align*}
where the first inequality comes from Lemma~\ref{lemma:simple_calculus_lemma}. Putting all of this together, we thus have that
\begin{align*}
    \sum_{t=1}^T R(m,t) - \sum_{t=1}^T\mathbb{E}_{A_t \sim Q_t}\left[R(A_t, t)\right] - \epsilon \cdot T
    \leq \hspace{0.3cm}&\frac{\ln |M|}{\eta} + \eta \sum_{t=1}^T \frac{1}{\epsilon}\\
    = \hspace{0.3cm}&\frac{\ln |M|}{\eta} + \left(\frac{\eta}{\epsilon}\right)\cdot T
\end{align*}
Moving $\epsilon \cdot T$ to the right-hand side gives us that
\begin{align*}
    \sum_{t=1}^T R(m,t) - \sum_{t=1}^T\mathbb{E}_{A_t \sim Q_t}\left[R(A_t, t)\right]
    \leq \frac{\ln |M|}{\eta} + \epsilon \cdot T + \left(\frac{\eta}{\epsilon}\right)\cdot T
\end{align*}
Letting $m$ be the prediction threshold that maximises $\sum_{t=1}^T R(m,t)$ completes the proof.
\end{proof}

We next move on to Theorem~\ref{thm:regret_coverage_connection}, which is fairly straightforward to prove:

\begin{theorem}
    Given a data stream $\xi$, classifier $C$, and sequence of actions $\mu = \langle a_1 \dots a_T \rangle \in (M \cup \{\texttt{query}\})^T$, if the regret of $\mu$ with the auxiliary reward of Definition~\ref{def:auxiliary_reward} is at most $N$, then the coverage of $\mu$ is at least $\beta - N/T$.
\end{theorem}
\begin{proof}
By assumption, we have that
$$
\max_{a^\star \in A} \sum_{t=1}^T R(a^\star, t) - \sum_{t=1}^T R(a_t, t) \leq N.
$$
Note that $R(1, t) = \beta$ for all $t$, which means that $\beta \cdot T \leq \max_{a^\star \in A} \sum_{t=1}^T R(a^\star, t)$. Combining this with the inequality above gives us that
$$
\beta \cdot T - \sum_{t=1}^T R(a_t, t) \leq N.
$$
Rearranging this expression, and dividing both sides by $T$, gives us that
$$
\beta - N/T \leq \frac{\sum_{t=1}^T R(a_t, t)}{T}.
$$
Note that $R(a_t, t) = 0$ if $a_t$ lacks coverage on round $t$, and that the reward per round can never exceed $1$. This means that $\sum_{t=1}^T R(a_t, t)/T \leq p_{\text{cover}}$. Substituting this into the inequality above completes the proof.
\end{proof}

Using this, we also obtain the following:

\begin{corollary}
    On any data stream $\xi$ and for any classifier $C$, the coverage rate $p_{\text{cover}}$ of Algorithm~\ref{algorithm:OCP_with_queries} satisfies
        $$
        \mathbb{E}\left[p_{\text{cover}}\right] \geq \beta - \left(\frac{2 \sqrt{\ln(|M|)} + 1}{\sqrt[3]{T}}\right)
        $$    
    provided that $\eta = T^{-2/3} \cdot \sqrt{\ln(|M|)}$ and $\epsilon = T^{-1/3} \leq 0.5$.
\end{corollary}
\begin{proof}
Immediate from Theorems~\ref{thm:simple_learner_regret} and \ref{thm:regret_coverage_connection}. Specifically, if $\epsilon = T^{-1/3} \leq 0.5$ and $\eta = T^{-2/3} \cdot \sqrt{\ln(|M|)}$, then the expected regret of Algorithm~\ref{algorithm:OCP_with_queries} is no larger than $T^{2/3} \cdot (2 \sqrt{\ln(|M|)} + 1)$. Theorem~\ref{thm:regret_coverage_connection} then implies that the expected coverage rate of Algorithm~\ref{algorithm:OCP_with_queries} must be at least 
$$
\beta - \frac{ T^{2/3} \cdot (2 \sqrt{\ln(|M|)} + 1)}{T},
$$
which simplifies to 
$$
\beta - \left(\frac{2 \sqrt{\ln(|M|)} + 1}{\sqrt[3]{T}}\right).
$$
\end{proof}

We also prove our bound on the size of the prediction sets, based on the value of $\max_{y \in \mathcal{Y}} p_C(y \mid x_t)$:

\begin{proposition}
    For any classifier $C$, any $x_t \in \mathcal{X}$, and any threshold $m_t \in [0,1]$, let $S_t$ be constructed as in Equation~\ref{equation:prediction_set}.
    Let $\max_{y \in \mathcal{Y}} p_C(y \mid x_t) = p_\text{max}$. 
    If $m_t < p_\text{max}$ then $|S_t| \leq 1 + \frac{1-p_\text{max}}{p_\text{max} - m_t}$, and if $m_t \geq p_\text{max}$ then $|S_t| = |\mathcal{Y}|$.
\end{proposition}
\begin{proof}
    Note that $p_\text{max} + (|S_t|-1) \cdot (p_\text{max} - m_t) \leq 1$, since at least one $y \in S_t$ has probability mass $p_\text{max}$, since all other $y \in S_t$ have a probability mass of at least $p_\text{max} - m_t$ under $p_C$, and since these numbers sum to a value no greater than one. By rearranging this expression, we obtain that
    $$
    |S_t| \leq 1 + \frac{1-p_\text{max}}{p_\text{max} - m_t}
    $$
    if $m_t < p_\text{max}$. Moreover, if $m_t \geq p_\text{max}$ then we of course have that $S_t = \mathcal{Y}$, since the condition in Equation~\ref{equation:prediction_set} will be satisfied by all $y \in \mathcal{Y}$.
\end{proof}


We will next move on to the high-probability bound. To prove this, we first need one more lemma. It is a martingale version of Bernstein's inequality (a special case of Freedman's inequality, \cite{Freedman1975OnTP}).
 
\begin{lemma}\label{lemma:custom_freedman_lemma}
Let $(\mathcal{F}_t)_{t=0}^T$ be a filtration and let $d_1,\dots,d_T$ be random variables such that, for every $t$, $d_t$ is $\mathcal{F}_t$-measurable and
\[
\mathbb{E}[d_t\mid \mathcal{F}_{t-1}]=0,\qquad d_t\le c,\qquad \mathbb{E}[d_t^2\mid\mathcal{F}_{t-1}]\le\sigma^2
\]
almost surely, for constants $c>0$ and $\sigma^2>0$. Then for every $L>0$,
\[
\Pr\Big(\sum_{t=1}^Td_t\ \ge\ 2\sqrt{T\sigma^2L}+2cL\Big)\ \le\ e^{-L}.
\]
In particular, this applies to independent random variables $d_1,\dots,d_T$ with $\mathbb{E}[d_t]=0$, $d_t\le c$ and $\text{Var}(d_t)\le\sigma^2$, with $\mathcal{F}_t$ the $\sigma$-algebra generated by $d_1,\dots,d_t$.
\end{lemma}
\begin{proof}
Let $S_s=\sum_{t\le s}d_t$ and fix $\lambda\in(0,1/c]$. Since $\lambda d_t\le 1$, we can use $e^x\le 1+x+x^2$ for $x\le 1$ to obtain
\[
\mathbb{E}\big[e^{\lambda d_t}\mid \mathcal{F}_{t-1}\big]\ \le\ 1+\lambda\,\mathbb{E}[d_t\mid\mathcal{F}_{t-1}]+\lambda^2\,\mathbb{E}[d_t^2\mid\mathcal{F}_{t-1}]\ \le\ 1+\lambda^2\sigma^2\ \le\ e^{\lambda^2\sigma^2}.
\]
Since $S_{t-1}$ is $ \mathcal{F}_{t-1}$-measurable, $\mathbb{E}[e^{\lambda S_t}]=\mathbb{E}\big[e^{\lambda S_{t-1}}\,\mathbb{E}[e^{\lambda d_t}\mid\mathcal{F}_{t-1}]\big]\le e^{\lambda^2\sigma^2}\mathbb{E}[e^{\lambda S_{t-1}}]$, and iterating gives $\mathbb{E}[e^{\lambda S_T}]\le e^{\lambda^2\sigma^2T}$. By Markov's inequality, for every $s\ge0$,
\[
\Pr(S_T\ge s)\ \le\ e^{-\lambda s}\,\mathbb{E}\big[e^{\lambda S_T}\big]\ \le\ \exp\big(\lambda^2T\sigma^2-\lambda s\big).
\]
Now let $s=2\sqrt{T\sigma^2L}+2cL$. If $s\le 2T\sigma^2/c$, take $\lambda=s/(2T\sigma^2)\le 1/c$; the exponent is $-s^2/(4T\sigma^2)\le-L$ because $s\ge 2\sqrt{T\sigma^2L}$. Otherwise take $\lambda=1/c$; the exponent is $T\sigma^2/c^2-s/c<s/(2c)-s/c=-s/(2c)\le-L$ because $T\sigma^2<cs/2$ and $s\ge 2cL$. In both cases $\Pr(S_T\ge s)\le e^{-L}$. For independent $d_t$ the conditional expectations are ordinary expectations and the same proof applies.
\end{proof}

We can now prove the high-probability bound:

\begin{theorem}
    For any $\epsilon \in (0,0.5]$, $\eta \in (0,1]$, and $\delta \in (0,1)$, we have that the regret of Algorithm~\ref{algorithm:OCP_with_queries} is at most
    $$
    2\left(\frac{n}{\epsilon}\right) + \epsilon T + \eta T + \left(\frac{\eta}{\epsilon}\right)T + \left(\frac{\eta}{\epsilon^2}\right)n + \frac{\ln(|M|)}{\eta}
    $$
    with probability at least $1-\delta$, where $n = 2\sqrt{\epsilon T\ln(3/\delta)} + 2\ln(3/\delta)$. If we let $\epsilon = T^{-1/3}$ and $\eta = T^{-2/3} \cdot \sqrt{\ln(|M|)}$, then we have that regret is at most
    \[
    T^{2/3}\Big(2\sqrt{\ln|M|}+4\sqrt{\ln(3/\delta)}+1\Big)+T^{1/3}\Big(\sqrt{\ln|M|}+4\ln(3/\delta)+2\sqrt{\ln(3/\delta)\ln|M|}\Big)+2\ln(3/\delta)\sqrt{\ln|M|}
    \]
    with probability at least $1-\delta$. This is $O\big(T^{2/3}(\sqrt{\ln|M|}+\sqrt{\ln(1/\delta)})\big)$.
\end{theorem}
\begin{proof}
To prove this, we first define the following quantities:
\begin{enumerate}
    \item Let $R_{T,m}$ be the total reward of prediction threshold $m$, i.e., $\sum_{t=1}^T R(m,t)$.
    \item Let $R_{T,\pi}$ be the (random) total reward of the learner, i.e., $\sum_{t=1}^T R(A_t,t)$.
    \item Let $\hat{R}_{T,m}$ be the total estimated reward of threshold $m$, i.e., $\hat{Z}_{m,T+1}$.
    \item Let $\tilde{R}_{T,\pi}$ be the \enquote{pseudo-expected regret} of the learner, i.e.,
    $$
    \tilde{R}_{T,\pi} = \sum_{t=1}^T \sum_{m \in M} P_{m,t} \cdot \hat{X}_{m,t},
    $$
    where $P_{m,t} = \exp \eta\hat{Z}_{m,t}/\sum_{i \in M} \exp \eta\hat{Z}_{i,t}$.
    \item Let $\text{REG}_m$ be the random regret that the learner suffers with respect to prediction threshold $m$ (so, i.e., $\text{REG}_m = R_{T,m} - R_{T,\pi}$).
\end{enumerate}
We then have that
\begin{align*}
    \text{REG}_m &= R_{T,m} - R_{T,\pi}\\
    &= R_{T,m} - R_{T,\pi} + (\hat{R}_{T,m} - \hat{R}_{T,m}) + (\tilde{R}_{T,\pi} - \tilde{R}_{T,\pi})\\
    &= (\hat{R}_{T,m} - \tilde{R}_{T,\pi}) + (R_{T,m} - \hat{R}_{T,m}) + (\tilde{R}_{T,\pi} - R_{T,\pi})\\
    &\leq \frac{\ln|M|}{\eta} + \eta \sum_{t=1}^T \sum_{m \in M} P_{m,t} \cdot \hat{X}^2_{m,t} + (R_{T,m} - \hat{R}_{T,m}) + (\tilde{R}_{T,\pi} - R_{T,\pi}).
\end{align*}
The inequality is proven in the same way as in the proof of Theorem~\ref{thm:simple_learner_regret}.
Note that this holds with probability 1. To complete the proof, we will create high-probability bounds for the last three terms in this expression (note that $\ln|M|/\eta$ is not a random variable, and thus does not need to be bounded).

We start with the first term, which contains the square of the reward estimators: 
\begin{align*}
    \sum_{m \in M} P_{m,t} \cdot \hat{X}^2_{m,t}
    &= \sum_{m \in M} P_{m,t} \cdot \Bigg(\mathds{1}\{a_t \neq \texttt{query}\} \cdot 1^2 + \mathds{1}\{a_t = \texttt{query}\}\left(1 - \left(\frac{1-R(m,t)}{\epsilon}\right)\right)^2 \Bigg)\\
    &\leq \mathds{1}\{a_t \neq \texttt{query}\} + \mathds{1}\{a_t = \texttt{query}\}\left(\frac{1}{\epsilon^2} + 1 - \frac{2}{\epsilon}\right)\\
    &= 1 + \mathds{1}\{a_t = \texttt{query}\}\left(\frac{1}{\epsilon^2} - \frac{2}{\epsilon}\right)\\
    &\leq 1 + \mathds{1}\{a_t = \texttt{query}\}\left(\frac{1}{\epsilon^2}\right)
\end{align*}
where we are using Lemma~\ref{lemma:simple_calculus_lemma} for the first inequality. Let $Q$ be the total number of queries that the agent makes. This implies that
\begin{align*}
    \eta \sum_{t=1}^T \sum_{m \in M} P_{m,t} \cdot \hat{X}^2_{m,t} &\leq \eta\left(T + \frac{Q}{\epsilon^2}\right)\\
    &= \eta \cdot T + \left(\frac{\eta}{\epsilon^2}\right) \cdot Q
\end{align*}
Note that the expected value of $Q$ is $\epsilon T$, and that the probability of making a query on a given round is i.i.d. In other words, $Q - \epsilon T = \sum_{t=1}^T(\mathds{1}\{a_t = \texttt{query}\} - \epsilon)$ is a sum of independent, zero-mean terms. Moreover, we have that $\mathds{1}\{a_t = \texttt{query}\} - \epsilon < 1$, and $\mathrm{Var}(\mathds{1}\{a_t = \texttt{query}\} - \epsilon) = \epsilon \cdot (1-\epsilon)^2 + (1-\epsilon) \cdot (-\epsilon)^2 = \epsilon \cdot (1-\epsilon) < \epsilon$. We can therefore apply Lemma~\ref{lemma:custom_freedman_lemma} with $c = 1$ and $\sigma^2 = \epsilon$ and conclude that for any $L > 0$, if we define $n_L := 2\sqrt{\epsilon T L} + 2L$, then we have that
\begin{align*}
    \mathbb{P}(Q - \epsilon T \geq n_L) \leq \exp(-L) \hspace{0.3cm}\implies\hspace{0.3cm} \mathbb{P}(Q \leq \epsilon T + n_L) \geq 1-\exp(-L). 
\end{align*}
Combining this with the above, this gives us that
\begin{align*}
    \eta \sum_{t=1}^T \sum_{m \in M} P_{m,t} \cdot \hat{X}^2_{m,t} &\leq \eta \cdot T + \left(\frac{\eta}{\epsilon^2}\right) \cdot (\epsilon T + n_L)\\
    &= \eta \cdot T + \left(\frac{\eta}{\epsilon}\right)T + \left(\frac{\eta}{\epsilon^2}\right) n_L
\end{align*}
with probability at least $1-\exp(-L)$. This gives us a bound on the magnitude of the first term.

Let us now move on to the next term, which measures the mismatch between the actual reward and the estimated reward of prediction threshold $m$:
\begin{align*}
    R_{T,m} - \hat{R}_{T,m}
    &= R_{T,m} - \sum_{t=1}^T \Bigg( \mathds{1}\{a_t \neq \texttt{query}\} \cdot 1 +\mathds{1}\{a_t = \texttt{query}\} \left(1 - \frac{1-R(m,t)}{\epsilon}\right)\Bigg)\\
    &= R_{T,m} - T + \left(\frac{1}{\epsilon}\right) \sum_{t=1}^T \mathds{1}\{a_t = \texttt{query}\}(1-R(m,t)).
\end{align*}
Let $v_t = \mathds{1}\{a_t = \texttt{query}\} \cdot (1-R(m,t))$ and let $V_T = \sum_{t=1}^T v_t$, and note that $\mathbb{E}[v_t] = \epsilon \cdot (1-R(m,t))$. Consider the centred variables $v_t - \mathbb{E}[v_t]$; these variables are independent, have zero mean, are bounded above by $1$, and have variance 
\begin{align*}
    \text{Var}(v_t - \mathbb{E}[v_t]) &= \epsilon \cdot (1-R(m,t) - \epsilon \cdot (1-R(m,t)))^2 + (1-\epsilon)\cdot(- \epsilon \cdot (1-R(m,t)))^2\\
    &= (1-R(m,t))^2 \cdot (1-\epsilon) \cdot \epsilon\\
    &\leq \epsilon
\end{align*}
We can therefore apply Lemma~\ref{lemma:custom_freedman_lemma} with $c=1$ and $\sigma^2 = \epsilon$, and conclude that for any $L > 0$,  if we again define $n_L := 2\sqrt{\epsilon T L} + 2L$, then we have that
\begin{align*}
    \mathbb{P}(V_t - \mathbb{E}\left[V_t\right] \geq n_L) \leq \exp(-L)
    \hspace{0.3cm}\implies\hspace{0.3cm} \mathbb{P}(V_t \leq \mathbb{E}\left[V_t\right] + n_L) \geq 1-\exp(-L). 
\end{align*}
Moreover, $\mathbb{E}\left[V_t\right] = \sum_{t=1}^T \epsilon \cdot (1-R(m,t)) = \epsilon T - \epsilon R_{T,m}$. Thus, with probability at least $1-\exp(-L)$, we have
\begin{align*}
    R_{T,m} - \hat{R}_{T,m} &\leq R_{T,m} - T + \left(\frac{1}{\epsilon}\right)(\epsilon T - \epsilon R_{T,m} + n_L)\\
    &= R_{T,m} - T + T - R_{T,m} + \frac{n_L}{\epsilon}\\
    &= \frac{n_L}{\epsilon}.
\end{align*}
This gives us a bound on the magnitude of the second term.

Let us now move on to the last term, which describes the discrepancy between the agent's \enquote{pseudo-expected reward} and its actual reward. 
Let $\rho_t:=\sum_{m\in M}P_{m,t}\cdot R(m,t)$, and note that we can express $\tilde{R}_{T,\pi} - R_{T,\pi}$ as $\sum_{t=1}^Te_t$, where
$$
e_t:=\mathds{1}\{a_t\neq\texttt{query}\}\cdot\big(1-R(a_t,t)\big)+\mathds{1}\{a_t=\texttt{query}\}\cdot\Big(1-\frac{1-\rho_t}{\epsilon}\Big).
$$
Although the data stream is fixed, the variables $e_t$ are not independent, because $\rho_t$ depends on the earlier query events through $P_t$. They are, however, martingale differences up to a drift. Conditionally on $\mathcal{F}_{t-1}$, the number $\rho_t$ is fixed; with probability $\epsilon$ the round is a query round, and otherwise $a_t$ is drawn from $P_t$, so that
$$
\mathbb{E}[e_t\mid\mathcal{F}_{t-1}]= (1-\epsilon)\,(1-\rho_t) + \epsilon\Big(1-\frac{1-\rho_t}{\epsilon}\Big)=\epsilon \cdot \rho_t.
$$
Note that $e_t\le1$; on a non-query round $e_t=1-R(a_t,t) \leq 1$, and on a query round $e_t=1-(1-\rho_t)/\epsilon\le1$. Moreover, using Lemma~\ref{lemma:simple_calculus_lemma}, we can calculate that
$$
\mathbb{E}[e_t^2\mid\mathcal{F}_{t-1}]\le\epsilon\Big(1-\frac{1-\rho_t}{\epsilon}\Big)^2+(1-\epsilon)\cdot1\le\Big(\frac1\epsilon+\epsilon-2\Big)+(1-\epsilon)=\frac1\epsilon-1\le\frac1\epsilon.
$$
Hence the centred variables $e_t-\epsilon\rho_t$ are martingale differences with respect to $(\mathcal{F}_t)$, bounded above by $1$, with conditional second moment $\text{Var}(e_t\mid\mathcal{F}_{t-1})\le\mathbb{E}[e_t^2\mid\mathcal{F}_{t-1}]\le1/\epsilon$. We can therefore apply Lemma~\ref{lemma:custom_freedman_lemma} with $c=1$ and $\sigma^2=1/\epsilon$ and conclude that for any $L$, with probability at least $1-\exp(-L)$, we have that
$$
\sum_{t=1}^T\big(e_t-\epsilon\rho_t\big) \le 2\sqrt{\frac{TL}{\epsilon}}+2L.
$$
If we define $n_L := 2\sqrt{\epsilon T L} + 2L$, then we can simplify this further:
$$
2\sqrt{\frac{TL}{\epsilon}}+2L\le2\sqrt{\frac{TL}{\epsilon}}+\frac{2L}{\epsilon}=\frac{n_L}{\epsilon}.
$$
Therefore, with probability at least $1-\exp(-L)$, we have that,
$$
\tilde{R}_{T,\pi}-R_{T,\pi}=\sum_{t=1}^T\epsilon\rho_t+\sum_{t=1}^T\big(e_t-\epsilon\rho_t\big)\le\epsilon T+ \frac{n_L}{\epsilon}.
$$
This gives us a bound on the magnitude of the last term.

Putting everything together, and using a simple union bound for the joint probability of the three bounds, we have that for any $L > 0$, if we define $n_L := 2\sqrt{\epsilon T L} + 2L$, then the regret is at most
$$
2\left(\frac{n_L}{\epsilon}\right) + \epsilon T + \eta T + \left(\frac{\eta}{\epsilon}\right) \cdot T + \left(\frac{\eta}{\epsilon^2}\right) \cdot n_L + \frac{\ln|M|}{\eta},
$$
with probability at least $1 - 3 \exp (-L)$. To finish the proof, let $L = \ln(3/\delta)$, so that the lower probability bound becomes $1-\delta$, and let $n = n_L = 2\sqrt{\epsilon T\ln(3/\delta)} + 2\ln(3/\delta).$
\end{proof}

Finally, we also note that Theorem~\ref{thm:high_probability_bound} induces a high-probability bound on the coverage rate of \texttt{OCPQ}:

\begin{corollary}
    If $\epsilon = T^{-1/3} \leq 0.5$ and $\eta = T^{-2/3} \cdot \sqrt{\ln(|M|)}$, then for any $\delta \in (0, 1)$ we have that with probability at least $1-\delta$, the coverage rate $p_\text{cover}$ of Algorithm~\ref{algorithm:OCP_with_queries} satisfies
    \begin{align*}
    p_\text{cover} \geq \hspace{0.2cm} \beta - \frac{2\sqrt{\ln|M|}+4\sqrt{\ln(3/\delta)}+1}{T^{1/3}} - \frac{\sqrt{\ln|M|}+4\ln(3/\delta)+2\sqrt{\ln(3/\delta)\ln|M|}}{T^{2/3}} - \frac{2\ln(3/\delta)\sqrt{\ln|M|}}{T}.
    \end{align*}
\end{corollary}
\begin{proof}
    Immediate from Theorems~\ref{thm:high_probability_bound} and \ref{thm:regret_coverage_connection}.
\end{proof}

\subsection{OCPQ With Adaptive $\beta$}\label{app:proof_appendix_for_adaptive_OCPQ}

In this subsection, we will provide our theoretical results for \texttt{OCPQ} with an adaptive $\beta$. We first need two lemmas:

\begin{lemma}\label{lemma:clipping_lemma}
    For any $x \in \mathbb{R}$, we have that $\text{clip}_{[0,1]}(x) \geq x - \max(x-1, 0)$.
\end{lemma}
\begin{proof}
    To see this, first note that if $x \in [0,1]$, then both sides of the inequality are equal to $x$. If $x > 1$ then the left-hand side is equal to $1$, and the right-hand side is equal to $x - (x-1) = 1$. Finally, if $x < 0$, then the left-hand side is equal to $0$, and the right-hand side is negative.
\end{proof}

\begin{lemma}\label{lemma:coverage_estimate_unbiased}
    For any $t \leq T$, we have that
    $$
    \mathbb{E}[\theta_t] = \mathbb{E}\left[\frac{1}{t} \cdot \sum_{i=1}^{t} \mathds{1}\{a_i \neq \texttt{query}\} \cdot \mathds{1}\{y_i \in S_i\}\right].
    $$
\end{lemma}
\begin{proof}
    Recall that $\theta_t = \left(\frac{1}{t}\right) \cdot \sum_{i=1}^t u_i$, where $u_i$ is
    $$
    u_i :=  \mathds{1}\{a_i = \texttt{query}\} \cdot \left(\frac{1}{\epsilon}\right) \cdot  \Pr(y_i \in S_i \mid h_{i-1}),
    $$
    i.e., the probability that $y_i \in S_i$, conditional on the history up until time $i$, multiplied by $1/\epsilon$ on query rounds and $0$ on non-query rounds. Note that the probability that $a_i = \texttt{query}$ is independent of $h_{i-1}$. Therefore, 
    $$
    \mathbb{E}[u_i \mid h_{i-1}] = (1-\epsilon) \cdot 0 + \epsilon \cdot \left(\frac{1}{\epsilon}\right) \cdot  \Pr(y_i \in S_i \mid h_{i-1}) = \Pr(y_i \in S_i \mid h_{i-1}).
    $$
    By the tower rule, we thus have that 
    $$
    \mathbb{E}[u_i] = \Pr(y_i \in S_i) = \mathbb{E}[\mathds{1}\{a_i \neq \texttt{query}\} \cdot \mathds{1}\{y_i \in S_i\}].
    $$
    Applying the linearity of expectation completes the proof.
\end{proof}

Using these lemmas, we can now prove the coverage bound. We here provide a version that is somewhat more general than the version stated in the main text of the paper, in that it applies to \texttt{OCPQ} with adjustment periods of arbitrary length. The version stated in the main text is noted as a special case:

\begin{theorem}
    Suppose that we run the adaptive version of \texttt{OCPQ} with target coverage $\sigma \in [0,1]$ over $k \geq 2$ time periods whose lengths are $n_1 \dots n_k$, that $\beta_1 = \sigma$ and $\beta_i = \text{clip}_{[0,1]}\left(\frac{(n_i + m_i) \cdot \sigma - m_i \cdot \theta_{i}^\star}{n_i}\right)$ for $i > 1$, that $\eta_i = n_i^{-2/3} \cdot \sqrt{\ln|M|}$ and $\epsilon_i = n_i^{-1/3}$, and that the length of each period is long enough that $\eta_i \leq 1$ and $\epsilon_i \leq 0.5$. We then have that
    \begin{align*}
        \mathbb{E}[p_{cover}] \geq \sigma - \frac{n_k^{2/3}}{T} \cdot \left(2 \sqrt{\ln(|M|)} + 1\right) - \frac{n_k}{T} \cdot \mathbb{E}\left[\max\left(\left(\frac{T\sigma - m_k \cdot \theta_{k}^\star}{n_k}\right) - 1, 0\right)\right].
    \end{align*}
    where $m_i = \sum_{j=1}^{i-1} n_i$ is the number of rounds before the start of the $i$'th period, and $\theta^\star_i$ is the value of $\theta_t$ at the start of the $i$'th period (i.e., $\theta^\star_i := \theta_{m_i}$). If we assume that $n_i = n_1 \cdot 2^{i-2}$ for $i \geq 2$, then this simplifies to
    \begin{align*}
        \mathbb{E}[p_{cover}] \geq \sigma - \left(\frac{1}{2}\right)^{2/3}\cdot\left(\frac{2\sqrt{\ln(|M|)} + 1}{\sqrt[3]{T}}\right) - \frac{1}{2} \cdot \mathbb{E}\left[\max\left(2\sigma -\theta_{k}^\star - 1, 0\right)\right].
    \end{align*}
\end{theorem}
\begin{proof}
    To keep the proof compact, let us first define a few new quantities:
    \begin{enumerate}
        \item Let $p_i$ be the coverage rate after the end of period $i$, i.e., $p_i = \frac{1}{m_i + n_i} \cdot \sum_{t=1}^{m_i + n_i} \mathds{1}\{a_t \neq \texttt{query}\} \cdot \mathds{1}\{y_t \in S_t\}$.
        \item Let $c_i$ be the coverage rate within period $i$, i.e., $c_i = \frac{1}{n_i} \cdot \sum_{t=m_i+1}^{m_i+n_i} \mathds{1}\{a_t \neq \texttt{query}\} \cdot \mathds{1}\{y_t \in S_t\}$
        \item Let $\delta_i = \left(\frac{2 \sqrt{\ln(|M|)} + 1}{\sqrt[3]{n_i}}\right)$ be the \enquote{coverage deficit} of \texttt{OCPQ} with time horizon $n_i$, as per Corollary~\ref{corollary:coverage_rate_bound}. 
    \end{enumerate}
    Note that $p_i = (m_i \cdot p_{i-1} + n_i \cdot c_i)/(m_i+n_i)$ for all $i > 1$; this lets us break the coverage calculation into cases.
    
    We can now make the following derivation:
    \begin{align*}
        \mathbb{E}[p_k] &= \mathbb{E}\left[\frac{m_k \cdot p_{k-1} + n_k \cdot c_k}{T}\right]\\
        &= \frac{m_k}{T} \cdot \mathbb{E}[p_{k-1}] + \frac{n_k}{T} \cdot \mathbb{E}[c_k]\\
        &\geq  \frac{m_k}{T} \cdot \mathbb{E}[p_{k-1}] + \frac{n_k}{T} \cdot \mathbb{E}[\beta_k - \delta_k]\\
        &=  \frac{m_k}{T} \cdot \mathbb{E}[p_{k-1}] + \frac{n_k}{T} \cdot \mathbb{E}\left[\text{clip}_{[0,1]}\left(\frac{T\sigma - m_k \cdot \theta_{k}^\star}{n_k}\right)\right] - \frac{n_k}{T} \cdot \delta_k\\
        &\geq  \frac{m_k}{T} \cdot \mathbb{E}[p_{k-1}] + \frac{n_k}{T} \cdot \mathbb{E}\left[\left(\frac{T\sigma - m_k \cdot \theta_{k}^\star}{n_k}\right) - \max\left(\left(\frac{T\sigma - m_k \cdot \theta_{k}^\star}{n_k}\right) - 1, 0\right)\right] - \frac{n_k}{T} \cdot \delta_k\\
        &=  \frac{m_k}{T} \cdot \mathbb{E}[p_{k-1}] + \sigma - \frac{m_k}{T} \cdot \mathbb{E}[\theta^\star_{k}] - \frac{n_k}{T} \cdot \mathbb{E}\left[\max\left(\left(\frac{T\sigma - m_k \cdot \theta_{k}^\star}{n_k}\right) - 1, 0\right)\right] - \frac{n_k}{T} \cdot \delta_k\\
        &=  \frac{m_k}{T} \cdot (\mathbb{E}[p_{k-1}] - \mathbb{E}[\theta^\star_{k}]) + \sigma - \frac{n_k}{T} \cdot \delta_k -  \frac{n_k}{T} \cdot \mathbb{E}\left[\max\left(\left(\frac{T\sigma - m_k \cdot \theta_{k}^\star}{n_k}\right) - 1, 0\right)\right]\\
        &= \sigma - \frac{n_k}{T} \cdot \delta_k -  \frac{n_k}{T} \cdot \mathbb{E}\left[\max\left(\left(\frac{T\sigma - m_k \cdot \theta_{k}^\star}{n_k}\right) - 1, 0\right)\right]\\
        &= \sigma - \frac{n_k}{T} \cdot \left(\frac{2 \sqrt{\ln(|M|)} + 1}{\sqrt[3]{n_k}}\right) - \frac{n_k}{T} \cdot \mathbb{E}\left[\max\left(\left(\frac{T\sigma - m_k \cdot \theta_{k}^\star}{n_k}\right) - 1, 0\right)\right]\\
        &= \sigma - \frac{n_k^{2/3}}{T} \cdot \left(2 \sqrt{\ln(|M|)} + 1\right) - \frac{n_k}{T} \cdot \mathbb{E}\left[\max\left(\left(\frac{T\sigma - m_k \cdot \theta_{k}^\star}{n_k}\right) - 1, 0\right)\right].
    \end{align*}
    The first inequality follows from Corollary~\ref{corollary:coverage_rate_bound} (conditioning on the history at the start of period $k$), and the second inequality follows from Lemma~\ref{lemma:clipping_lemma}. We also use Lemma~\ref{lemma:coverage_estimate_unbiased} to conclude that $\mathbb{E}[p_{k-1}] - \mathbb{E}[\theta^\star_{k}] = 0$, which means that $\frac{m_k}{T} \cdot (\mathbb{E}[p_{k-1}] - \mathbb{E}[\theta^\star_{k}])$ cancels out.
    
    Finally, if we assume that $n_i = n_1 \cdot 2^{i-1}$ for $i > 1$, then we have that $n_k = T/2$ and $m_k = n_k$. If we plug this into the expression above, then it simplifies to
    \begin{align*}
        \mathbb{E}[p_{cover}] \geq \sigma - \left(\frac{1}{2}\right)^{2/3}\cdot\left(\frac{2\sqrt{\ln(|M|)} + 1}{\sqrt[3]{T}}\right) - \frac{1}{2} \cdot \mathbb{E}\left[\max\left(2\sigma -\theta_{k}^\star - 1, 0\right)\right].
    \end{align*}
    This completes the proof.
\end{proof}

We will next provide a proof that bounds the magnitude of the error term $\mathbb{E}\left[\max\left(2\sigma -\theta_{k}^\star - 1, 0\right)\right]$, which stems from the fact that $\beta$ must be clipped to lie in $[0,1]$. Note that while the proof of Theorem~\ref{thm:adaptive_ocpq_coverage_bound} applies to periods with arbitrary length, our proofs bounding the clipping term assume that the periods double in length, as described in the main text of the paper. This is primarily simply a matter of convenience -- these proofs could be generalised to other period lengths, at the expense of making the proofs messier and harder to read. 

Before we give our proof, we will first provide a few useful lemmas:

\begin{lemma}\label{lemma:shortfall_decomposition}
For every $i \in 2 \dots k$, we have that
$$
\sigma - p_i \leq \frac{1}{2} \cdot \left(\max\left(2\sigma -\theta_i^\star - 1, 0\right) + \frac{\text{reg}_i}{n_i} + (\theta_i^\star - p_{i-1})\right)
$$
where $n_i$ is the length of period $i$, $p_i$ is the coverage at the end of period $i$, $\theta^\star_i$ is the value of $\theta_t$ at the start of period $i$, and $\text{reg}_i$ is the regret that is incurred within period $i$ (relative to the arm that is optimal within period $i$, not the arm that is optimal over the full time horizon).
\end{lemma}
\begin{proof}
    We start by observing that
    \begin{align*}
        c_i &\geq \beta_i - \frac{\text{reg}_i}{n_i}\\
        &= \text{clip}_{[0,1]}(2\sigma - \theta_i^\star) - \frac{\text{reg}_i}{n_i}\\
        &\geq 2\sigma - \theta_i^\star - \max\left(2\sigma -\theta_{i}^\star - 1, 0\right) - \frac{\text{reg}_i}{n_i}.
    \end{align*}
    The first inequality follows from Theorem~\ref{thm:regret_coverage_connection}, and the second follows from Lemma~\ref{lemma:clipping_lemma}. Using this inequality, we can derive that
    \begin{align*}
        p_i &= \frac{1}{2}\cdot (p_{i-1} + c_i)\\
        &\geq \frac{1}{2}\cdot \left(p_{i-1} + 2\sigma - \theta_i^\star - \max\left(2\sigma -\theta_{i}^\star - 1, 0\right) - \frac{\text{reg}_i}{n_i}\right)\\
        &= \sigma + \frac{1}{2}\cdot \left(p_{i-1} - \theta_i^\star - \max\left(2\sigma -\theta_{i}^\star - 1, 0\right) - \frac{\text{reg}_i}{n_i}\right).
    \end{align*}
    By rearranging, we get that
    $$
    \sigma - p_i \leq \frac{1}{2} \cdot \left(\max\left(2\sigma -\theta_i^\star - 1, 0\right) + \frac{\text{reg}_i}{n_i} + (\theta_i^\star - p_{i-1})\right).
    $$
\end{proof}
This lemma shows that if the coverage rate $p_i$ after the $i$'th period is lower than the target coverage rate $\sigma$, then this error must come from one of three sources: the clipping term that is applied to $\beta_i$, the regret per time step that is incurred within period $i$, or the error in the estimation $\theta_i^\star$ of the coverage $p_{i-1}$ at the start of period $i$.

Our proof will proceed by first assuming that these error terms are small, showing that this implies that $\mathbb{E}\left[\max\left(2\sigma -\theta_{k}^\star - 1, 0\right)\right]$ will be small, and then showing that the error terms in fact will be small with high probability. We must first identify a suitable event that we can condition on. For indices $\ell < j$, define $G_{\ell,j}$ as the event that
\begin{equation}
    G_{\ell,j} = \bigcap_{i=\ell+1}^{j-1} \left\{\frac{\text{reg}_i}{n_i} \leq \frac{1-\sigma}{3}\right\} \cap \bigcap_{i=\ell}^{j-1} \left\{|\theta_{i+1}^\star - p_{i}| \leq \frac{1-\sigma}{3}\right\},
\end{equation}
where $\text{reg}_i$ is the regret that is incurred within the $i$'th time period (relative to the action that is optimal within the $i$'th time period, not the action that is optimal over the full time horizon). As before, $\theta^\star_i$ is the value of $\theta_t$ at the \emph{start} of the $i$'th time period.
In other words, if $G_{\ell,j}$ occurs, then the regret per time step is not too high for any period between $\ell+1$ and $j-1$, and the estimation error is not too high at the \emph{start} of any time period between $\ell+1$ and $j$. Conditional on this event, we can show that the clipping term will shrink over time:

\begin{lemma}\label{lemma:contraction_if_good_event}
    Let $\kappa_i = \max\left(2\sigma -\theta_{i}^\star - 1, 0\right)$.
    If $G_{\ell,j}$ occurs, then for every $i$ with $\ell < i < j$, we have that
    $$
    \kappa_{i+1} \leq \max\left(\frac{1}{2}\kappa_{i} - \frac{1}{3}(1-\sigma), 0\right)
    $$
\end{lemma}
\begin{proof}
    We begin by expanding $\sigma - \theta^\star_{i+1}$, and applying Lemma~\ref{lemma:shortfall_decomposition}:
    \begin{align*}
        \sigma - \theta^\star_{i+1} &= (\sigma - p_i) + (p_i - \theta^\star_{i+1})\\
        &\leq \frac{1}{2} \cdot \left(\max\left(2\sigma -\theta_i^\star - 1, 0\right) + \frac{\text{reg}_i}{n_i} + (\theta_i^\star - p_{i-1})\right) + (p_i - \theta^\star_{i+1})\\
        &= \frac{1}{2} \cdot \left(\kappa_i + \frac{\text{reg}_i}{n_i} + (\theta_i^\star - p_{i-1})\right) + (p_i - \theta^\star_{i+1}).
    \end{align*}
    If $G_{\ell,j}$ occurs then $\text{reg}_i/n_i$, $(\theta_i^\star - p_{i-1})$, and $(p_i - \theta^\star_{i+1})$ are all bounded by $(1-\sigma)/3$. Substituting this into the expression above gives us that
    \begin{align*}
        \sigma - \theta^\star_{i+1} &\leq \frac{\kappa_i}{2} + \frac{1-\sigma}{6} + \frac{1-\sigma}{6} + \frac{1-\sigma}{3}\\
        &= \frac{1}{2}\kappa_i + \frac{2}{3}(1-\sigma)
    \end{align*}
    Subtracting $1-\sigma$ from both sides gives us that
    \begin{align*}
        2\sigma - \theta_{i+1}^\star - 1 \leq \frac{1}{2}\kappa_i - \frac{1}{3}(1-\sigma)
    \end{align*}
    Applying $\max(\cdot, 0)$ to both sides of the inequality completes the proof.
\end{proof}

We next show how this shrinkage implies that the clipping term will hit $0$ after a certain number of periods:

\begin{lemma}\label{lemma:absorption_lemma}
    Let $\tau \in (0,1]$ and let $x_0, x_1, \dots$ be non-negative reals with $x_0 \leq 1$ and $x_{i+1} \leq \max\left(\frac{1}{2} x_{i} - \tau, 0\right)$ for $i\geq0$. Then $x_i = 0$ for every $i \geq 1 + \left\lceil \log_2\left(\frac{1}{2\tau}\right) \right\rceil$.
\end{lemma}
\begin{proof}
    By induction, we have that $x_i \leq 2^{-i}$ for all $i$. Therefore, if $m = \left\lceil \log_2\left(\frac{1}{2\tau}\right) \right\rceil$ then we have that $x_m \leq 2^{-m} \leq 2\tau$, and so $x_{m+1} \leq \max\left(\frac{1}{2} 2\tau - \tau, 0\right) = 0$. Finally, 0 is absorbing: if $x_i = 0$ then $x_{i+1} = 0$.
\end{proof}

We must next bound the probability that event $G_{\ell,j}$ will occur, which means that we must both bound the probability of having a high regret within the relevant time periods, and the probability of getting a large estimation error at the start of those time periods. We can bound the former using Theorem~\ref{thm:high_probability_bound}. We will bound the latter using Freedman's inequality: 

\textbf{Freedman's Inequality} \citep{Freedman1975OnTP} :
Let $(X_t)_{t \le N}$ be a martingale difference sequence with respect to a filtration $(\mathcal{F}_t)$, with $X_t \le b$ almost surely and $\sum_{t\le N}\mathbb{E}[X_t^2 \mid \mathcal{F}_{t-1}] \le v$ almost surely. Then for every $s > 0$,
$$
\Pr\left(\textstyle\sum_{t \leq N} X_t \geq s\right) \leq \exp\left(-\frac{s^2}{2(v + bs/3)}\right).
$$

\begin{lemma}\label{lemma:estimation_error}
    Let $p_t$ and $\theta_t$ be the coverage and the coverage estimate at the start of round $t$ (\underline{note:} this is round $t$, not period $t$). Then for every $t \geq 0$ and every $a \in (0,1]$, we have that
    $$
    \Pr(|p_t - \theta_t| \geq a) \leq 2\exp\left(-0.2a^2 t^{2/3}\right).
    $$
\end{lemma}
\begin{proof}
    Let $X_t = \mathds{1}\{y_t \in S_t\} - u_t$, so $p_t - \theta_t = \tfrac{1}{t}\sum_{i \leq t} X_i$
    and $(X_t)$ is a martingale difference sequence by Lemma~\ref{lemma:coverage_estimate_unbiased}. We will apply Freedman's inequality using $s = a \cdot t$.
    
    We must first bound the conditional variance $\sum_{t\le N}\mathbb{E}[X_t^2 \mid \mathcal{F}_{t-1}]$. Note that $u_i = 0$ whenever $a_i \neq \texttt{query}$, whereas $\mathds{1}\{y_i \in S_i\} = 0$ whenever $a_i = \texttt{query}$. This means that $X_i^2 = \mathds{1}\{y_i \in S_i\} + u_i^2$. With a query rate of $\epsilon$, we get that
    $$
    \mathbb{E}[X_i^2 \mid h_{i-1}]
    = \Pr(y_i \in S_i \mid h_{i-1}) + \epsilon\cdot\left(\frac{\Pr(y_i \in S_i \mid h_{i-1})}{\epsilon}\right)^2 \leq  1 + \frac{1}{\epsilon}.
    $$
    Since $\epsilon \geq t^{-1/3}$ for every round before round $t$, we have that the sum of the $X_i^2$-terms cannot exceed $t \cdot (1 + t^{1/3}) = t + t^{4/3}$. We can therefore let $v = t + t^{4/3}$.
    
    For the upper tail, note that
    $X_t \le \mathds{1}\{y_t \in S_t\} \le 1$, so we can let $b = 1$. Since $a \leq 1$, we have that
    $$
    2(v + bs/3) = 2(t + t^{4/3} + at/3) \leq  \left(4 + \frac{2}{3}\right) \cdot t^{4/3} \leq 5 \cdot t^{4/3}.
    $$
    
    For the lower tail, we apply the inequality to $(-X_t)$. Now
    $-X_t \le u_t \le 1/\epsilon \leq t^{1/3}$, so we let $b = t^{1/3}$.
    Then 
    $$
    2(v + bs/3) = 2(t + t^{4/3} + t^{1/3}at/3) \leq  \left(4 + \frac{2}{3}\right) \cdot t^{4/3} \leq 5 \cdot t^{4/3}.
    $$
    We thus have that $5 \cdot t^{4/3}$ is an upper bound on the denominator in both cases. A union bound over both tails thus give us that
    $$
    \Pr\left(|\textstyle\sum_{t \leq N} X_t| \geq at\right) \leq \exp\left(-\frac{(at)^2}{5 \cdot t^{4/3}}\right) \leq 2\exp\left(-0.2a^2 t^{2/3}\right).
    $$
\end{proof}

Putting this together, we can now bound the clipping term. This proof takes as input a $\delta < 1$ and a period index $\ell$, such that for each period after period $\ell$, we have that with probability at least $1-\delta$, the regret per time step will not exceed $(1-\sigma)/3$. If the period lengths grow monotonically, then we can find suitable values for $\ell$ and $\delta$ using Theorem~\ref{thm:high_probability_bound} (e.g., pick an arbitrary value of $\delta$ and let $\ell$ be the smallest index such that the upper bound in Theorem~\ref{thm:high_probability_bound} is less than $(1-\sigma)/3$ using $T = n_\ell$ and $\delta$).

\begin{theorem}\label{thm:clip_bound}
    Let $\sigma > 0.5$, and assume that we have $\ell \in \mathbb{N}$, $\delta \in (0,1)$ such that for any period $i \geq \ell$, we have that $\Pr\left(\frac{\text{reg}_i}{n_i} \geq \frac{1-\sigma}{3} \right) \leq \delta$. Then for every $j \geq \ell + 2 + \left\lceil \log_2\left(\frac{3}{2(1-\sigma)}\right) \right\rceil$, we have that
    $$
    \Pr\left(\max\left(2\sigma -\theta_{j}^\star - 1, 0\right) = 0\right) \geq 1 - (j-\ell-1)\cdot \delta - (j-\ell) \cdot \left(2 \exp\left(-\frac{(1-\sigma)^2 \cdot (n_\ell)^{2/3}}{45}\right)\right).
    $$
\end{theorem}
\begin{proof}
    We first show that if event $G_{\ell,j}$ occurs, then $\max\left(2\sigma -\theta_{j}^\star - 1, 0\right) = 0$. For any $i \in 0 \dots j-\ell-1$, let $x_i = \max\left(2\sigma -\theta_{\ell+1+i}^\star - 1, 0\right)$. Lemma~\ref{lemma:contraction_if_good_event} tells us that, conditional on $G_{\ell,j}$, we have that 
    $$
    x_{i+1} \leq \max\left(\frac{1}{2} x_i - \frac{1}{3}(1-\sigma), 0\right)
    $$
    for each $i \in 0 \dots j-\ell-2$. We can therefore apply Lemma~\ref{lemma:absorption_lemma}, and conclude that $x_i = 0$ for every $i \geq 1 + \left\lceil \log_2\left(\frac{3}{2(1-\sigma)}\right) \right\rceil$, provided that $i \leq j-\ell-1$. Since $j \geq \ell + 2 + \left\lceil \log_2\left(\frac{3}{2(1-\sigma)}\right) \right\rceil$, this means that $x_{j-\ell-1} = 0$, and so $\max\left(2\sigma -\theta_{\ell+1+j-\ell-1}^\star - 1, 0\right) = \max\left(2\sigma -\theta_{j}^\star - 1, 0\right) = 0$.

    We next bound the probability that event $G_{\ell,j}$ occurs. Recall that
    $$
    G_{\ell,j} = \bigcap_{i=\ell+1}^{j-1} \left\{\frac{\text{reg}_i}{n_i} \leq \frac{1-\sigma}{3}\right\} \cap \bigcap_{i=\ell}^{j-1} \left\{|\theta_{i+1}^\star - p_{i}| \leq \frac{1-\sigma}{3}\right\},
    $$
    so $G_{\ell,j}$ requires that the regret per time step is bounded by $(1-\sigma)/3$ for each period strictly between $\ell$ and $j$, and that the estimation error $|\theta_{i+1}^\star - p_{i}|$ is bounded by $(1-\sigma)/3$ at the start of each period after $\ell$ (but not including $\ell$) and up to (and including) $j$. 
    
    We have assumed that for any period $i \geq \ell$, we have that $\Pr\left(\frac{\text{reg}_i}{n_i} \geq \frac{1-\sigma}{3} \right) \leq \delta$. A simple union bound therefore gives us that the probability that $\frac{\text{reg}_i}{n_i} \geq \frac{1-\sigma}{3}$ for \emph{any} period between $\ell$ and $j$ can be at most $(j-\ell-1)\cdot \delta$, since there are $j-\ell-1$ such periods.
    
    Next, Lemma~\ref{lemma:estimation_error} says that for every $t \geq 0$ and every $a \in (0,1]$, we have that
    $$
    \Pr(|p_t - \theta_t| \geq a) \leq 2\exp\left(-0.2a^2 t^{2/3}\right).
    $$
    If we let $a = (1-\sigma)/3$ then this simplifies to $2\exp\left(-\frac{(1-\sigma)^2 \cdot t^{2/3}}{45}\right)$. Since we only need the value of $\theta_t$ at the start of the periods between $\ell$ and $j$, and since $t \geq n_\ell$ at the start of each of those periods, it follows that 
    $$
    \Pr\left(|p_i - \theta^\star_{i+1}| \geq \frac{1-\sigma}{3}\right) \leq 2\exp\left(-\frac{(1-\sigma)^2 \cdot (n_\ell)^{2/3}}{45}\right)
    $$
    for each $i \in \ell, \dots, j-1$. We can now again apply a simple union bound, and conclude that the probability that $\bigcap_{i=\ell}^{j-1} \left\{|\theta_{i+1}^\star - p_{i}| \leq \frac{1-\sigma}{3}\right\}$ fails to occur can be at most $(j-\ell) \cdot 2\exp\left(-\frac{(1-\sigma)^2 \cdot (n_\ell)^{2/3}}{45}\right)$. Taken together, this means that
    $$
    \Pr(G_{\ell,j}) \geq 1 - (j-\ell-1)\cdot \delta - (j-\ell) \cdot \left(2 \exp\left(-\frac{(1-\sigma)^2 \cdot (n_\ell)^{2/3}}{45}\right)\right).
    $$
    This completes the proof.
\end{proof}

Using Theorem~\ref{thm:clip_bound}, we can also straightforwardly derive the following:

\begin{corollary}\label{corollary:clip_bound}
    Let $\sigma > 0.5$, and assume that we have $\ell \in \mathbb{N}$, $\delta \in (0,1)$ such that for any period $i \geq \ell$, we have that $\Pr\left(\frac{\text{reg}_i}{n_i} \geq \frac{1-\sigma}{3} \right) \leq \delta$. Then for every $j \geq \ell + 2 + \left\lceil \log_2\left(\frac{3}{2(1-\sigma)}\right) \right\rceil$, we have that
    $$
    \mathbb{E}\left[\max\left(2\sigma -\theta_{j}^\star - 1, 0\right)\right] \leq (2\sigma - 1) \cdot \left((j-\ell-1)\cdot \delta + (j-\ell) \cdot \left(2 \exp\left(-\frac{(1-\sigma)^2 \cdot (n_\ell)^{2/3}}{45}\right)\right)\right).
    $$
\end{corollary}
\begin{proof}
    Immediate from Theorem~\ref{thm:clip_bound}, noting that $\max\left(2\sigma -\theta_{j}^\star - 1, 0\right) \in [0, 2\sigma-1]$.
\end{proof}

This shows that the expectation of the clipping term eventually will be small (and, indeed, will be $0$ with high probability), provided that the regret shrinks sufficiently quickly and that the time horizon is sufficiently long.

\section{Adaptive Query Rate}\label{app:adaptive-query}
An extension of Algorithm 1 consists of replacing the fixed query rate $\epsilon$ with an adaptive query rate $q_t$, which can be different at any time step and can be derived using information up to time $t$, such as the past interaction history, the current input $x_t$, and the current weights. While this extension has practical benefits, such as the ability to increase the query probability of inputs deemed more difficult, we show below that, in our adversarial setting, the regret- and coverage-optimal strategy is to allocate the query probability uniformly, i.e., $q_t = \epsilon$ for all $t$. Nevertheless, below we provide versions of Lemma 1, Theorem 1 and Corollary 1 (and corresponding proofs) for the adaptive query rate that generalise the fixed-rate query case. We finally show that the latter cannot be outperformed by an adaptive strategy given adversarial sequences and the same query budget.

Let $\mathcal{H}_t$ denote the information available to the learner immediately before the query decision at round $t$. We allow the learner to choose a query probability $q_t \in (0,0.5]$ as any $\mathcal{H}_t$-measurable quantity. The query coin at round $t$ is then drawn with probability $q_t$. Since the data stream $\xi$ is fixed in advance by the oblivious adversary, $R(m,t)$ is fixed conditional on $\mathcal{H}_t$.

For each $m \in M$ and $t \leq T$, we define a random variable $\hat{X}_{m,t}$ that is equal to $1$ if $a_t \neq \texttt{query}$, and otherwise,
$$
\hat{X}_{m,t} = 1 - \frac{1-R(m,t)}{q_t}.
$$
Note that this means that $\hat{Z}_{m,t+1}=\hat{Z}_{m,t}+\hat{X}_{m,t}$. Moreover, $\hat{X}_{m,t}$ acts as an estimator for $R(m,t)$, as we show below.
\begin{lemma}\label{lemma:the_estimators_are_unbiased_adaptive}
For every $m\in M$ and $t\leq T$,
$$
\mathbb{E}[\hat{X}_{m,t}\mid\mathcal{H}_t]=R(m,t).
$$
Consequently, $\mathbb{E}[\hat{X}_{m,t}]=R(m,t)$ and $\mathbb{E}[\hat{Z}_{m,t}]=\sum_{j=1}^{t-1}R(m,j)$, where the expectation is over the randomness in the learning algorithm.
\end{lemma}
\begin{proof}
Conditional on $\mathcal{H}_t$, the value $q_t$ is fixed and the algorithm queries with probability $q_t$. We therefore have
\begin{align*}
\mathbb{E}[\hat{X}_{m,t}\mid\mathcal{H}_t]
=&(1-q_t)\cdot 1+q_t\cdot\left(1-\frac{1-R(m,t)}{q_t}\right)\\
=&1-q_t+q_t-(1-R(m,t))\\
=&R(m,t).
\end{align*}
The unconditional identity follows from the tower property: $$\mathbb{E}[\hat{X}_{m,t}] = \mathbb{E}\left[\mathbb{E}[\hat{X}_{m,t} \mid\mathcal{H}_t]\right] = \mathbb{E}[R(m,t)]=R(m,t),$$
where the last equality holds because the sequence is fixed and so $R(m,t)$ is not random. 
Using $\hat{Z}_{m,t}=\sum_{j=1}^{t-1}\hat{X}_{m,j}$ and linearity of expectation then gives $\mathbb{E}[\hat{Z}_{m,t}]=\sum_{j=1}^{t-1}R(m,j)$.
\end{proof}

We can now prove the main regret bound:
\begin{theorem}
    For any data stream $\xi$ and classifier $C$, any $\eta\in(0,1]$, and any predictable sequence of query probabilities $q_t\in(0,0.5]$, the expected regret of Algorithm~\ref{algorithm:OCP_with_queries} with the auxiliary reward is at most
    $$
    \frac{\ln(|M|)}{\eta}
    +\mathbb{E}\left[\sum_{t=1}^Tq_t\right]
    +\eta\,\mathbb{E}\left[\sum_{t=1}^T\frac{1}{q_t}\right].
    $$
    In particular, if $q_t=\epsilon$ for every $t$, this reduces to
    $$
    \frac{\ln(|M|)}{\eta}+\epsilon\cdot T+\left(\frac{\eta}{\epsilon}\right)\cdot T.
    $$
    If we further set $\epsilon=T^{-1/3}$ and $\eta=T^{-2/3}\cdot\sqrt{\ln(|M|)}$, then this simplifies to
    $$
    T^{2/3}\cdot(2\sqrt{\ln(|M|)}+1).
    $$
\end{theorem}
\begin{proof}
    This proof is very similar to the proof of Theorem 1. The first part of the proof is the same, and so we skip it.

    Let $w_{m,t} = \exp \left(\eta \cdot \hat{Z}_{m,t}\right)$ and $W_t = \sum_{m \in M} w_{m,t}$. This means that if the algorithm is \emph{not} making a query at time $t$, then the probability that it will pick a given prediction threshold $m$ is equal to $w_{m,t}/W_t$. Also recall that $\hat{Z}_{m,t}$ is an unbiased estimate of $\sum_{j=1}^{t-1}R(m,j)$ (see Lemma~\ref{lemma:the_estimators_are_unbiased}). In particular, $\hat{Z}_{m,T+1}$ is an unbiased estimate of $\sum_{t=1}^{T}R(m,t)$. 
    

In this proof, it will be convenient to distinguish between the action probabilities that are suggested by the exponential weights $\{w_{i,t}\}_{i \in M}$ and the action probabilities that the agent is actually using. Therefore, let $P_t$ be the action probabilities of the exponential weights, so that
$$
\mathbb{P}_{A \sim P_t}(A = m) = \frac{w_{m,t}}{W_t}
$$
for $m \in M$, and $\mathbb{P}_{A \sim P_t}(A = \texttt{query}) = 0$. Note that, from the definition of the weights $\{w_{m,t}\}_{m \in M}$, this is equivalent to applying a softmax function with temperature $\eta$ to $\{\hat{Z}_{m,t}\}_{m \in M}$. 
Similarly, let $Q_t$ be the action probabilities that Algorithm~\ref{algorithm:OCP_with_queries} will use at time $t$, conditional on $\mathcal{H}_t$, so that
$$
\mathbb{P}_{A\sim Q_t}(A=\texttt{query}\mid\mathcal{H}_t)=q_t,
$$
and
$$
\mathbb{P}_{A\sim Q_t}(A=m\mid\mathcal{H}_t)=(1-q_t)\cdot\left(\frac{w_{m,t}}{W_t}\right)
$$
for $m\in M$. Both $P_t$ and $q_t$ are $\mathcal{H}_t$-measurable. 

From the proof of Theorem 1, we have that 
\begin{align*}
&\hat{Z}_{m,T+1} - \sum_{t=1}^T\mathbb{E}_{A_t \sim P_t}\left[\hat{X}_{A_t,t}\right]\\
\leq &\frac{\ln |M|}{\eta} + \eta\sum_{t=1}^T\mathbb{E}_{A_t \sim P_t}\left[\hat{X}_{A_t,t}^2\right].
\end{align*}
Note that the statement above holds for any arbitrary collection of numbers $\{\hat{X}_{m,t}\}_{m \in M, t \in \{1 \dots T\}}$, where $\hat{Z}_{m,T+1}$ and $P_t$ are derived from these numbers in the way we specified. 
We will now treat the numbers $\{\hat{X}_{m,t}\}$ as \emph{random variables} that are generated during a run of Algorithm~\ref{algorithm:OCP_with_queries} (using adaptive query rates $q_t$ instead of a fixed rate $\epsilon$), and take the expectation of both sides of the inequality with respect to these random variables. 

For clarity, we will also replace $\mathbb{E}_{A_t\sim P_t}$ with $\sum_{m\in M}P_{m,t}$ to emphasise the fact that the values of $P_{m,t}$ are themselves random variables determined by the past. Note that $P_{m,t}=w_{m,t}/W_t$ and is $\mathcal{H}_t$-measurable. The left-hand side of the inequality then becomes
\begin{align*}
&\mathbb{E}\left[\hat{Z}_{m,T+1}-\sum_{t=1}^T\sum_{m'\in M}P_{m',t}\hat{X}_{m',t}\right]\\
=\hspace{0.3cm}&\sum_{t=1}^TR(m,t)-\sum_{t=1}^T\sum_{m'\in M}\mathbb{E}[P_{m',t}\hat{X}_{m',t}]\\
=\hspace{0.3cm}&\sum_{t=1}^TR(m,t)-\sum_{t=1}^T\sum_{m'\in M}\mathbb{E}\left[\mathbb{E}[P_{m',t}\hat{X}_{m',t}\mid\mathcal{H}_t]\right]\\
=\hspace{0.3cm}&\sum_{t=1}^TR(m,t)-\sum_{t=1}^T\sum_{m'\in M}\mathbb{E}\left[P_{m',t}\,\mathbb{E}[\hat{X}_{m',t}\mid\mathcal{H}_t]\right]\\
=\hspace{0.3cm}&\sum_{t=1}^TR(m,t)-\sum_{t=1}^T\mathbb{E}\left[\sum_{m'\in M}P_{m',t}R(m',t)\right],
\end{align*}
where the second-to-last equality holds because $P_{m',t}$ is $\mathcal{H}_t$-measurable, i.e., deterministic given $\mathcal{H}_t$ and the last equality because of Lemma~\ref{lemma:the_estimators_are_unbiased}. 
In contrast to the fixed-$\epsilon$ proof, we do not use independence between $P_{m,t}$ and $\hat{X}_{m,t}$: under adaptive querying, these variables need not be independent, whereas conditional unbiasedness is sufficient.

Let
$$
\bar R_t=\sum_{m\in M}P_{m,t}R(m,t).
$$
Conditional on $\mathcal{H}_t$, Algorithm~\ref{algorithm:OCP_with_queries} queries with probability $q_t$, receiving reward $0$, and otherwise samples a threshold according to $P_t$. Hence
$$
\mathbb{E}[R(a_t,t)\mid\mathcal{H}_t]=(1-q_t)\bar R_t,
$$
which can be rewritten as
$$
\bar R_t=\mathbb{E}[R(a_t,t)\mid\mathcal{H}_t]+q_t\bar R_t.
$$
Since $0\leq\bar R_t\leq1$, taking expectations and summing over $t$ gives
$$
\sum_{t=1}^T\mathbb{E}[\bar R_t]
\leq \mathbb{E}\left[\sum_{t=1}^TR(a_t,t)\right]
+\mathbb{E}\left[\sum_{t=1}^Tq_t\right].
$$
Therefore, the expected left-hand side of the deterministic exponential-weights inequality is at least
\begin{align*}
\sum_{t=1}^TR(m,t)
-\mathbb{E}\left[\sum_{t=1}^TR(a_t,t)\right]
-\mathbb{E}\left[\sum_{t=1}^Tq_t\right].
\end{align*}

Let us now turn to the right-hand side of the inequality. Since $P_{m,t}$ and $q_t$ are $\mathcal{H}_t$-measurable, we can bound its second-moment term conditionally:
\begin{align*}
&\mathbb{E}\left[\sum_{m\in M}P_{m,t}\hat{X}_{m,t}^2\,\middle|\,\mathcal{H}_t\right]\\
=&(1-q_t)\sum_{m\in M}P_{m,t}\cdot1^2
+q_t\sum_{m\in M}P_{m,t}\left(1-\frac{1-R(m,t)}{q_t}\right)^2\\
=&1-q_t+q_t\sum_{m\in M}P_{m,t}\left(1-\frac{1-R(m,t)}{q_t}\right)^2\\
\leq&1-q_t+\sum_{m\in M}P_{m,t}\left(\frac{1}{q_t}+q_t-2\right)\\
=&1-q_t+\left(\frac{1}{q_t}+q_t-2\right)\sum_{m\in M}P_{m,t} = \frac{1}{q_t}-1 \leq \frac{1}{q_t}.
\end{align*}
The first inequality applies Lemma~\ref{lemma:simple_calculus_lemma} conditionally with $\epsilon=q_t$. Taking expectations, the right-hand side of the deterministic exponential-weights inequality is therefore at most
$$
\frac{\ln|M|}{\eta}+\eta\,\mathbb{E}\left[\sum_{t=1}^T\frac{1}{q_t}\right].
$$

Putting all of this together, we thus have that, for every $m\in M$,
\begin{align*}
&\sum_{t=1}^TR(m,t)-\mathbb{E}\left[\sum_{t=1}^TR(a_t,t)\right]\\
\leq\hspace{0.3cm}&\frac{\ln|M|}{\eta}
+\mathbb{E}\left[\sum_{t=1}^Tq_t\right]
+\eta\,\mathbb{E}\left[\sum_{t=1}^T\frac{1}{q_t}\right].
\end{align*}
Letting $m$ be the prediction threshold that maximises $\sum_{t=1}^TR(m,t)$ completes the proof.
\end{proof}

We can now use the above result with Theorem 2 to state the coverage bounds for adaptive querying.

\begin{corollary}
    On any data stream $\xi$ and for any classifier $C$, the coverage rate $p_{\mathrm{cover}}$ of Algorithm~\ref{algorithm:OCP_with_queries} with predictable query probabilities $q_t\in(0,0.5]$ satisfies
    $$
    \mathbb{E}[p_{\mathrm{cover}}]
    \geq \beta
    -\frac{\ln|M|}{\eta T}
    -\frac{1}{T}\mathbb{E}\left[\sum_{t=1}^Tq_t\right]
    -\frac{\eta}{T}\mathbb{E}\left[\sum_{t=1}^T\frac{1}{q_t}\right].
    $$
    In particular, for the fixed choice $q_t=\epsilon=T^{-1/3}\leq0.5$ and $\eta=T^{-2/3}\sqrt{\ln|M|}$, this reduces to
    $$
    \mathbb{E}[p_{\mathrm{cover}}]\geq\beta-\left(\frac{2\sqrt{\ln(|M|)}+1}{\sqrt[3]{T}}\right).
    $$
\end{corollary}
\begin{proof}
Theorem~\ref{thm:regret_coverage_connection} holds pathwise, so if $\operatorname{reg}(\mu)$ denotes the realised regret,
$$
p_{\mathrm{cover}}\geq\beta-\frac{\operatorname{reg}(\mu)}{T}.
$$
Taking expectations and applying Theorem~\ref{thm:simple_learner_regret} gives
\begin{align*}
\mathbb{E}[p_{\mathrm{cover}}]
&\geq\beta-\frac{\mathbb{E}[\operatorname{reg}(\mu)]}{T}\\
&\geq\beta
-\frac{\ln|M|}{\eta T}
-\frac{1}{T}\mathbb{E}\left[\sum_{t=1}^Tq_t\right]
-\frac{\eta}{T}\mathbb{E}\left[\sum_{t=1}^T\frac{1}{q_t}\right].
\end{align*}
Substituting $q_t=\epsilon=T^{-1/3}$ and $\eta=T^{-2/3}\sqrt{\ln|M|}$ recovers the stated fixed-rate bound.
\end{proof}

Below we state the optimality result for the fixed-rate query strategy.
\begin{corollary}
    Consider a query budget $B$ and a query strategy $(q_t)_{t=1,\ldots,T}$ that utilises the full budget, i.e., $\sum_{i=1}^T q_i = B$. Then, for any budget $B$,
    \begin{enumerate}
        \item  the fixed-rate strategy $q_t=\epsilon=B/T$ minimises the upper bound on the worst-case expected regret, and
        \item consequently, it maximizes the corresponding lower bound on expected coverage.
    \end{enumerate}
    \begin{proof}
    Note that the regret bound contains the terms $\mathbb{E}[\sum_{t=1}^T q_t]$ and $\eta\mathbb{E}[\sum_{t=1}^T 1/q_t]$ and the coverage bound contains the terms $-\mathbb{E}[\sum_{t=1}^T q_t]$ and $-\eta\mathbb{E}[\sum_{t=1}^T 1/q_t]$. By the Cauchy-Schwarz inequality, we have that 
$$
\sum_{t=1}^T \frac{1}{q_t} \geq \frac{T^2}{B},
$$
with equality when the query probabilities are constant, because in that case
$\sum_{t=1}^T \frac{1}{q_t}= \sum_{t=1}^T T/B$. 
    \end{proof}
\end{corollary}

\section{High-Probability Bound Comparisons}\label{app:high_probability_bound_comparison}

As noted previously, \citet{label_efficient_prediction} introduce two algorithms for a class of partial monitoring games, of which OCPWF is a special case. They prove (in their Theorem 2) that for any $T \in \mathbb{N}$, any $m \in \mathbb{N}$, and any $\Delta \in (0,1)$, if their \emph{label efficient exponentially weighted average forecaster} is run with the right parameters (depending on $T$, $m$, and $\Delta$) then with probability at least $1-\Delta$, it will make at most $m$ queries and incur a regret of at most
$$
2 \cdot T \cdot \sqrt{\frac{\ln(K)}{m}} + 6 \cdot T \cdot \sqrt{\frac{\ln(4\cdot T/\Delta)}{m}}
$$
where $K$ is the number of actions (meaning that $K = |M|$ for OCPWF). We compare this bound against our Theorem~\ref{thm:high_probability_bound} for each value of $\delta \in \{10^{-i} : i \in 1 \dots 10\}$ and each value of $T \in \{10^i : i \in 1 \dots 10\}$, using $|M|=10$. Using these parameters, we find that Theorem~\ref{thm:high_probability_bound} gives a tighter bound on the regret for $T \geq 1,000$, regardless of the value of $\delta$, whereas \citet{label_efficient_prediction} only gives a tighter bound when $T \leq 100$ (and even then, our bound is still tighter for some values of $\delta$). 
It is also worth noting that both bounds exceed $T$ for $T \leq 10^3$, so both bounds are trivial in this range.
Part of this data is shown in Table~\ref{table:high_probability_bound_comparison}.

\begin{sidewaystable}
\begin{tabular}{rl|cccccccc|}
\cline{3-10}
\multicolumn{1}{l}{}                             &            & \multicolumn{8}{c|}{Horizon $T$}                                                                                                                                                                                         \\ \cline{3-10} 
\multicolumn{1}{l}{}                             &            & \multicolumn{1}{c|}{$10^2$} & \multicolumn{1}{c|}{$10^3$} & \multicolumn{1}{c|}{$10^4$} & \multicolumn{1}{c|}{$10^5$} & \multicolumn{1}{c|}{$10^6$} & \multicolumn{1}{c|}{$10^7$} & \multicolumn{1}{c|}{$10^8$} & \multicolumn{1}{c|}{$10^9$} \\ \hline
\multicolumn{1}{|r|}{\multirow{10}{*}{$\delta$}} & $10^{-2}$  & \multicolumn{1}{c|}{$\textcolor{blue}{5\cdot 10^{2}}$, $\textcolor{red}{5\cdot 10^{2}}$}       & \multicolumn{1}{c|}{$\textcolor{blue}{2\cdot 10^{3}}$, $\textcolor{red}{2\cdot 10^{3}}$}       & \multicolumn{1}{c|}{$\textcolor{blue}{7\cdot 10^{3}}$, $\textcolor{red}{1\cdot 10^{4}}$}       & \multicolumn{1}{c|}{$\textcolor{blue}{3\cdot 10^{4}}$, $\textcolor{red}{6\cdot 10^{4}}$}       & \multicolumn{1}{c|}{$\textcolor{blue}{1\cdot 10^{5}}$, $\textcolor{red}{3\cdot 10^{5}}$}       & \multicolumn{1}{c|}{$\textcolor{blue}{6\cdot 10^{5}}$, $\textcolor{red}{1\cdot 10^{6}}$}       & \multicolumn{1}{c|}{$\textcolor{blue}{3\cdot 10^{6}}$, $\textcolor{red}{7\cdot 10^{6}}$}       & \multicolumn{1}{c|}{$\textcolor{blue}{1\cdot 10^{7}}$, $\textcolor{red}{3\cdot 10^{7}}$}       \\ \cline{2-10} 
\multicolumn{1}{|r|}{}                           & $10^{-3}$  & \multicolumn{1}{c|}{$\textcolor{red}{6\cdot 10^{2}}$, $\textcolor{blue}{5\cdot 10^{2}}$}       & \multicolumn{1}{c|}{$\textcolor{blue}{2\cdot 10^{3}}$, $\textcolor{red}{3\cdot 10^{3}}$}       & \multicolumn{1}{c|}{$\textcolor{blue}{8\cdot 10^{3}}$, $\textcolor{red}{1\cdot 10^{4}}$}       & \multicolumn{1}{c|}{$\textcolor{blue}{4\cdot 10^{4}}$, $\textcolor{red}{6\cdot 10^{4}}$}       & \multicolumn{1}{c|}{$\textcolor{blue}{2\cdot 10^{5}}$, $\textcolor{red}{3\cdot 10^{5}}$}       & \multicolumn{1}{c|}{$\textcolor{blue}{7\cdot 10^{5}}$, $\textcolor{red}{2\cdot 10^{6}}$}       & \multicolumn{1}{c|}{$\textcolor{blue}{3\cdot 10^{6}}$, $\textcolor{red}{7\cdot 10^{6}}$}       & \multicolumn{1}{c|}{$\textcolor{blue}{2\cdot 10^{7}}$, $\textcolor{red}{4\cdot 10^{7}}$}       \\ \cline{2-10} 
\multicolumn{1}{|r|}{}                           & $10^{-4}$  & \multicolumn{1}{c|}{$\textcolor{red}{6\cdot 10^{2}}$, $\textcolor{blue}{6\cdot 10^{2}}$}       & \multicolumn{1}{c|}{$\textcolor{blue}{2\cdot 10^{3}}$, $\textcolor{red}{3\cdot 10^{3}}$}       & \multicolumn{1}{c|}{$\textcolor{blue}{9\cdot 10^{3}}$, $\textcolor{red}{1\cdot 10^{4}}$}       & \multicolumn{1}{c|}{$\textcolor{blue}{4\cdot 10^{4}}$, $\textcolor{red}{7\cdot 10^{4}}$}       & \multicolumn{1}{c|}{$\textcolor{blue}{2\cdot 10^{5}}$, $\textcolor{red}{3\cdot 10^{5}}$}       & \multicolumn{1}{c|}{$\textcolor{blue}{8\cdot 10^{5}}$, $\textcolor{red}{2\cdot 10^{6}}$}       & \multicolumn{1}{c|}{$\textcolor{blue}{4\cdot 10^{6}}$, $\textcolor{red}{8\cdot 10^{6}}$}       & \multicolumn{1}{c|}{$\textcolor{blue}{2\cdot 10^{7}}$, $\textcolor{red}{4\cdot 10^{7}}$}       \\ \cline{2-10} 
\multicolumn{1}{|r|}{}                           & $10^{-5}$  & \multicolumn{1}{c|}{$\textcolor{red}{7\cdot 10^{2}}$, $\textcolor{blue}{6\cdot 10^{2}}$}       & \multicolumn{1}{c|}{$\textcolor{blue}{2\cdot 10^{3}}$, $\textcolor{red}{3\cdot 10^{3}}$}       & \multicolumn{1}{c|}{$\textcolor{blue}{1\cdot 10^{4}}$, $\textcolor{red}{1\cdot 10^{4}}$}       & \multicolumn{1}{c|}{$\textcolor{blue}{4\cdot 10^{4}}$, $\textcolor{red}{7\cdot 10^{4}}$}       & \multicolumn{1}{c|}{$\textcolor{blue}{2\cdot 10^{5}}$, $\textcolor{red}{3\cdot 10^{5}}$}       & \multicolumn{1}{c|}{$\textcolor{blue}{9\cdot 10^{5}}$, $\textcolor{red}{2\cdot 10^{6}}$}       & \multicolumn{1}{c|}{$\textcolor{blue}{4\cdot 10^{6}}$, $\textcolor{red}{8\cdot 10^{6}}$}       & \multicolumn{1}{c|}{$\textcolor{blue}{2\cdot 10^{7}}$, $\textcolor{red}{4\cdot 10^{7}}$}       \\ \cline{2-10} 
\multicolumn{1}{|r|}{}                           & $10^{-6}$  & \multicolumn{1}{c|}{$\textcolor{red}{8\cdot 10^{2}}$, $\textcolor{blue}{6\cdot 10^{2}}$}       & \multicolumn{1}{c|}{$\textcolor{blue}{3\cdot 10^{3}}$, $\textcolor{red}{3\cdot 10^{3}}$}       & \multicolumn{1}{c|}{$\textcolor{blue}{1\cdot 10^{4}}$, $\textcolor{red}{2\cdot 10^{4}}$}       & \multicolumn{1}{c|}{$\textcolor{blue}{5\cdot 10^{4}}$, $\textcolor{red}{7\cdot 10^{4}}$}       & \multicolumn{1}{c|}{$\textcolor{blue}{2\cdot 10^{5}}$, $\textcolor{red}{4\cdot 10^{5}}$}       & \multicolumn{1}{c|}{$\textcolor{blue}{9\cdot 10^{5}}$, $\textcolor{red}{2\cdot 10^{6}}$}       & \multicolumn{1}{c|}{$\textcolor{blue}{4\cdot 10^{6}}$, $\textcolor{red}{8\cdot 10^{6}}$}       & \multicolumn{1}{c|}{$\textcolor{blue}{2\cdot 10^{7}}$, $\textcolor{red}{4\cdot 10^{7}}$}       \\ \cline{2-10} 
\multicolumn{1}{|r|}{}                           & $10^{-7}$  & \multicolumn{1}{c|}{$\textcolor{red}{9\cdot 10^{2}}$, $\textcolor{blue}{7\cdot 10^{2}}$}       & \multicolumn{1}{c|}{$\textcolor{blue}{3\cdot 10^{3}}$, $\textcolor{red}{3\cdot 10^{3}}$}       & \multicolumn{1}{c|}{$\textcolor{blue}{1\cdot 10^{4}}$, $\textcolor{red}{2\cdot 10^{4}}$}       & \multicolumn{1}{c|}{$\textcolor{blue}{5\cdot 10^{4}}$, $\textcolor{red}{8\cdot 10^{4}}$}       & \multicolumn{1}{c|}{$\textcolor{blue}{2\cdot 10^{5}}$, $\textcolor{red}{4\cdot 10^{5}}$}       & \multicolumn{1}{c|}{$\textcolor{blue}{1\cdot 10^{6}}$, $\textcolor{red}{2\cdot 10^{6}}$}       & \multicolumn{1}{c|}{$\textcolor{blue}{4\cdot 10^{6}}$, $\textcolor{red}{8\cdot 10^{6}}$}       & \multicolumn{1}{c|}{$\textcolor{blue}{2\cdot 10^{7}}$, $\textcolor{red}{4\cdot 10^{7}}$}       \\ \cline{2-10} 
\multicolumn{1}{|r|}{}                           & $10^{-8}$  & \multicolumn{1}{c|}{$\textcolor{red}{1\cdot 10^{3}}$, $\textcolor{blue}{7\cdot 10^{2}}$}       & \multicolumn{1}{c|}{$\textcolor{blue}{3\cdot 10^{3}}$, $\textcolor{red}{3\cdot 10^{3}}$}       & \multicolumn{1}{c|}{$\textcolor{blue}{1\cdot 10^{4}}$, $\textcolor{red}{2\cdot 10^{4}}$}       & \multicolumn{1}{c|}{$\textcolor{blue}{5\cdot 10^{4}}$, $\textcolor{red}{8\cdot 10^{4}}$}       & \multicolumn{1}{c|}{$\textcolor{blue}{2\cdot 10^{5}}$, $\textcolor{red}{4\cdot 10^{5}}$}       & \multicolumn{1}{c|}{$\textcolor{blue}{1\cdot 10^{6}}$, $\textcolor{red}{2\cdot 10^{6}}$}       & \multicolumn{1}{c|}{$\textcolor{blue}{5\cdot 10^{6}}$, $\textcolor{red}{9\cdot 10^{6}}$}       & \multicolumn{1}{c|}{$\textcolor{blue}{2\cdot 10^{7}}$, $\textcolor{red}{4\cdot 10^{7}}$}       \\ \cline{2-10} 
\multicolumn{1}{|r|}{}                           & $10^{-9}$  & \multicolumn{1}{c|}{$\textcolor{red}{1\cdot 10^{3}}$, $\textcolor{blue}{7\cdot 10^{2}}$}       & \multicolumn{1}{c|}{$\textcolor{blue}{3\cdot 10^{3}}$, $\textcolor{red}{4\cdot 10^{3}}$}       & \multicolumn{1}{c|}{$\textcolor{blue}{1\cdot 10^{4}}$, $\textcolor{red}{2\cdot 10^{4}}$}       & \multicolumn{1}{c|}{$\textcolor{blue}{5\cdot 10^{4}}$, $\textcolor{red}{8\cdot 10^{4}}$}       & \multicolumn{1}{c|}{$\textcolor{blue}{2\cdot 10^{5}}$, $\textcolor{red}{4\cdot 10^{5}}$}       & \multicolumn{1}{c|}{$\textcolor{blue}{1\cdot 10^{6}}$, $\textcolor{red}{2\cdot 10^{6}}$}       & \multicolumn{1}{c|}{$\textcolor{blue}{5\cdot 10^{6}}$, $\textcolor{red}{9\cdot 10^{6}}$}       & \multicolumn{1}{c|}{$\textcolor{blue}{2\cdot 10^{7}}$, $\textcolor{red}{4\cdot 10^{7}}$}       \\ \cline{2-10} 
\multicolumn{1}{|r|}{}                           & $10^{-10}$ & \multicolumn{1}{c|}{$\textcolor{red}{1\cdot 10^{3}}$, $\textcolor{blue}{8\cdot 10^{2}}$}       & \multicolumn{1}{c|}{$\textcolor{blue}{4\cdot 10^{3}}$, $\textcolor{red}{4\cdot 10^{3}}$}       & \multicolumn{1}{c|}{$\textcolor{blue}{1\cdot 10^{4}}$, $\textcolor{red}{2\cdot 10^{4}}$}       & \multicolumn{1}{c|}{$\textcolor{blue}{6\cdot 10^{4}}$, $\textcolor{red}{8\cdot 10^{4}}$}       & \multicolumn{1}{c|}{$\textcolor{blue}{2\cdot 10^{5}}$, $\textcolor{red}{4\cdot 10^{5}}$}       & \multicolumn{1}{c|}{$\textcolor{blue}{1\cdot 10^{6}}$, $\textcolor{red}{2\cdot 10^{6}}$}       & \multicolumn{1}{c|}{$\textcolor{blue}{5\cdot 10^{6}}$, $\textcolor{red}{9\cdot 10^{6}}$}       & \multicolumn{1}{c|}{$\textcolor{blue}{2\cdot 10^{7}}$, $\textcolor{red}{4\cdot 10^{7}}$}       \\ \cline{2-10} 
\multicolumn{1}{|r|}{}                           & $10^{-11}$ & \multicolumn{1}{c|}{$\textcolor{red}{1\cdot 10^{3}}$, $\textcolor{blue}{8\cdot 10^{2}}$}       & \multicolumn{1}{c|}{$\textcolor{blue}{4\cdot 10^{3}}$, $\textcolor{red}{4\cdot 10^{3}}$}       & \multicolumn{1}{c|}{$\textcolor{blue}{1\cdot 10^{4}}$, $\textcolor{red}{2\cdot 10^{4}}$}       & \multicolumn{1}{c|}{$\textcolor{blue}{6\cdot 10^{4}}$, $\textcolor{red}{9\cdot 10^{4}}$}       & \multicolumn{1}{c|}{$\textcolor{blue}{3\cdot 10^{5}}$, $\textcolor{red}{4\cdot 10^{5}}$}       & \multicolumn{1}{c|}{$\textcolor{blue}{1\cdot 10^{6}}$, $\textcolor{red}{2\cdot 10^{6}}$}       & \multicolumn{1}{c|}{$\textcolor{blue}{5\cdot 10^{6}}$, $\textcolor{red}{9\cdot 10^{6}}$}       & \multicolumn{1}{c|}{$\textcolor{blue}{2\cdot 10^{7}}$, $\textcolor{red}{4\cdot 10^{7}}$}       \\ \hline
\end{tabular}
\caption{We compare the high-probability regret bound in Theorem~\ref{thm:high_probability_bound} to the high-probability regret bound by \citet{label_efficient_prediction} (their Theorem~2) using different values of $\delta$ and $T$, assuming that $|M| = 10$. Each cell in the table lists our bound first and the bound from \citet{label_efficient_prediction} second, with the lower (and thus better) bound being coloured \textcolor{blue}{\textbf{blue}}, and the larger bound being coloured \textcolor{red}{\textbf{red}} (before rounding, so two non-equal values may appear the same in the table).}\label{table:high_probability_bound_comparison}
\end{sidewaystable}

\end{document}